\documentclass{article}

\usepackage{arxiv}
\usepackage{amssymb,amsmath,amsthm}
\usepackage{mathtools}
\usepackage{upgreek}
\newcommand{\defeq}{\vcentcolon=}
\usepackage[numbers, compress]{natbib}
\usepackage{graphicx}
\usepackage{booktabs}

\usepackage{bm}
\usepackage[utf8]{inputenc} 
\usepackage[T1]{fontenc}    
\usepackage{xr}

\usepackage{bm}
\usepackage{tabularx}
\usepackage{makecell}
\usepackage[utf8]{inputenc} 
\usepackage[T1]{fontenc}    
\usepackage{hyperref}       
\usepackage{url}            
\usepackage{booktabs}       
\usepackage{amsfonts}       
\usepackage{nicefrac}       
\usepackage{microtype}      
\usepackage{xcolor}         
\usepackage{multirow}
\usepackage{upgreek}
\newtheorem{theorem}{Theorem}
\newtheorem{definition}{Definition}
\newtheorem{corollary}{Corollary}

\usepackage{cleveref}
\usepackage{tablefootnote}

\title{Informing Geometric Deep Learning with Electronic Interactions to Accelerate Quantum Chemistry}

\author{%
  Zhuoran Qiao \\
  California Institute of Technology\\
  \texttt{zqiao@caltech.edu} \\
  \And
  Anders S. Christensen \\
  Entos, Inc.\\ 
  \texttt{anders@entos.ai} \\
  \And
  Matthew Welborn \\
  Entos, Inc. \\ 
  \texttt{matt@entos.ai} \\
  \And
  Frederick R. Manby \\
  Entos, Inc. \\ 
  \texttt{fred@entos.ai} \\
  \And
  Anima Anandkumar \\
  California Institute of Technology\\
  NVIDIA\\ 
  \texttt{anima@caltech.edu} \\
  \And
  Thomas F. Miller III \\
  California Institute of Technology\\
  Entos, Inc. \\ 
  \texttt{tfm@caltech.edu} \\
}

\begin{document}

\maketitle

\begin{abstract}
Predicting electronic energies, densities, and related chemical properties can facilitate the discovery 
of novel catalysts, medicines, and battery materials. 
By developing a physics-inspired equivariant neural network, we introduce a method to learn molecular representations based on the electronic interactions among atomic orbitals.
Our method, OrbNet-Equi, leverages efficient tight-binding simulations and learned mappings to recover high fidelity quantum chemical properties.
OrbNet-Equi models a wide spectrum of target properties with an accuracy consistently better than standard machine learning methods and a speed several orders of magnitude greater than density functional theory. 
Despite only using training samples collected from readily available small-molecule libraries, OrbNet-Equi outperforms traditional methods on comprehensive downstream benchmarks that encompass diverse main-group chemical processes. Our method also describes interactions in challenging charge-transfer complexes and open-shell systems. 
We anticipate that the strategy presented here will help to expand opportunities for studies in chemistry and materials science, where the acquisition of experimental or reference training data is costly.
\end{abstract}

Discovering new molecules and materials is central to tackling contemporary challenges in energy storage and drug discovery~\cite{ling2022review,vamathevan2019applications}. As the experimentally uninvestigated chemical space for these applications is immense, large-scale computational design and screening for new molecule candidates has the potential to vastly reduce the burden of laborious experiments and to accelerate discovery~\cite{baik2002computing,shoichet2004virtual,cordova2020data}. A crucial task is to model the quantum chemical properties of molecules by solving the many-body Schrödinger equation, which is commonly addressed by \textit{ab initio} electronic structure methods~\cite{SzaboNesbet,kohn1965self} such as density functional theory (DFT) (Figure~\ref{fig:features}a). While very successful, \textit{ab initio} methods are laden with punitive computational requirements that makes it difficult to achieve a throughput on a scale of the unexplored chemical space.

In contrast, machine learning (ML) approaches are highly flexible as function approximators, and thus are promising for modelling molecular properties at a drastically reduced computational cost. A large class of ML-based molecular property predictors  includes  methods that use atomic-coordinate-based input features which closely resemble molecular mechanics (MM) descriptors~\cite{mackerell2004empirical,Behler2007,Behler2016,christensen2020fchl,zhang2018deep,schutt2017schnet,unke2019physnet,DeepMoleNet,DimeNet,spherenet,schutt2021equivariant}; these methods will be referred to as Atomistic ML methods in the current work (Figure~\ref{fig:features}b). Atomistic ML methods have been employed to solve challenging problems in molecular sciences such as RNA structure prediction~\cite{townshend2021geometric} and anomalous phase transitions~\cite{cheng2020evidence}. However, there remains a key discrepancy between Atomistic ML and \textit{ab initio} approaches regarding the modelling of quantum chemical properties, as Atomistic ML approaches typically neglect the electronic degrees of freedom which are central for the description of important phenomena such as electronic excitations, charge transfer, and long-range interactions.
Moreover, recent work shows that Atomistic ML can struggle with transferability
on downstream tasks where the molecules may chemically deviate from the training samples~\cite{Hutch,rosenberger2021modeling} as is expected to be common for under-explored chemical spaces. 

Recent efforts to embody quantum mechanics (QM) into molecular representations based on electronic structure theory have made breakthroughs in improving both the chemical and electronic transferability of ML-based molecular modelling~\cite{mobml2,deephf,Li2018,orbnet1,li2021kohn,nagai2020completing,kirkpatrick2021pushing}.
Leveraging a physical feature space extracted from QM simulations, such QM-informed ML methods
have attained data efficiency that significantly surpass  Atomistic ML methods, especially when extrapolated to systems with length scales or chemical compositions unseen during training. 
Nevertheless, QM-informed ML methods still fall short in terms of the flexibility of modelling diverse molecular properties unlike their atomistic counterparts, as  they are typically implemented for a limited set of learning targets such as the electronic energy or the exchange-correlation potential. A key bottleneck hampering the broader applicability of QM-informed approaches is the presence of unique many-body symmetries necessitated by an explicit treatment on electron-electron interactions. Heuristic schemes have been used to enforce invariance~\cite{mobml1,neuralxc,deephf,orbnet1,low2022inclusion,karandashev2021orbital} at a potential loss of information in their input features or expressivity in their ML models. Two objectives remain elusive for QM-informed machine learning: (a) incorporate the underlying physical symmetries with maximal data efficiency and model flexibility, and (b) accurately infer downstream molecular properties for large chemical spaces, at a computational resource requirement on par with existing empirical and Atomisic ML methods.

Herein, we introduce an end-to-end ML method for QM-informed molecular representations, OrbNet-Equi, in fulfillment of these two objectives. 
OrbNet-Equi featurizes a mean-field electronic structure via the atomic orbital basis, and learns molecular representations through a machine learning model that is equivariant with respect to isometric basis transformations (Figure~\ref{fig:features}c-e). By the virtue of equivariance, OrbNet-Equi respects essential physical constraints of symmetry conservation so that the target quantum chemistry properties are learned independent of a reference frame. Underpinning OrbNet-Equi is a neural network designed with insights from recent advances in geometric deep learning~\cite{bronstein2021geometric,cohen2016group,weiler20183d,kondor2018clebsch,cormorant,thomas2018tensor,fuchs2020se}, but with key architectural innovations 
to achieve equivariance based on the tensor-space algebraic structures entailed in atomic-orbital-based molecular representations.  

We demonstrate the data efficiency of OrbNet-Equi on learning molecular properties using input features obtained from tight-binding QM simulations which are efficient and scalable to systems with thousands of atoms~\cite{bannwarth2021extended}. We find that OrbNet-Equi consistently achieves lower prediction errors than existing Atomistic ML methods and our previous QM-informed ML method~\cite{orbnet1} on diverse target properties such as electronic energies, dipole moments, electron densities, and frontier orbital energies. Specifically, our study on learning frontier orbital energies illustrates an effective strategy to improve the prediction of electronic properties by incorporating molecular-orbital-space information.

To showcase its transferability to complex real-world chemical spaces, we trained an OrbNet-Equi model on single-point energies of 236k molecules curated from readily available small-molecule libraries. The resulting model, OrbNet-Equi/SDC21, achieves a performance competitive to state-of-the-art composite DFT methods when tested on a wide variety of main-group quantum chemistry benchmarks, 
while being up to thousand-fold faster at runtime.
As a particular case study, we found that OrbNet-Equi/SDC21 substantially improved the prediction accuracy of ionization potentials relative to semi-empirical QM methods, even though no radical species was included for training. Thus, our method has the potential to accelerate simulations for challenging problems in organic synthesis~\cite{chen2020embedded}, battery design~\cite{janet2020accurate}, and molecular biology~\cite{dommer2021covidisairborne}.
Detailed data analysis pinpoints viable future directions to systematically improve its chemical space coverage, opening a plausible pathway towards a generic hybrid physics-ML paradigm for the acceleration of molecular modelling and discovery.

\section*{Results}

\label{sec:exp}

\subsection*{The OrbNet-Equi methodology}

OrbNet-Equi featurizes a molecular system through mean-field QM simulations.
Semi-empirical tight-binding models~\cite{bannwarth2021extended} are used through this study since they
can be solved rapidly for both small-molecules and extended systems, which enables deploying OrbNet-Equi to large chemical spaces.
In particular, we employ the recently reported GFN-xTB~\cite{gfn1} QM model
in which the mean-field electronic structure $\Psi_{0}$ is obtained through self-consistently solving a tight-binding model system (Figure~\ref{fig:features}c). Built upon $\Psi_{0}$, the inputs to the neural network comprises a stack of matrices $\mathbf{T}[{\Psi_{0}}]$ defined as 
single-electron operators $\hat{\mathcal{O}}[{\Psi_{0}}]$ represented in the atomic orbitals (Figure~\ref{fig:features}d),
\begin{equation}
    \big(\mathbf{T}[{\Psi_{0}}] \big)^{n, l, m; n', l', m'}_{AB} = \langle \Phi_A^{n, l, m} \lvert \hat{\mathcal{O}}[{\Psi_{0}}] \rvert  \Phi_B^{n', l', m'} \rangle 
\end{equation}
where $A$ and $B$ are both atom indices; $(n, l, m)$ and $(n', l', m')$ indicate a basis function in the set of atomic orbitals $\{ \Phi\}$ centered at each atom. Motivated by mean-field electronic energy expressions, the input atomic orbital features are selected as $\mathbf{T} = (\mathbf{F, P, H, S})$ using the Fock $\mathbf{F}$, density $\mathbf{P}$, core-Hamiltonian  $\mathbf{H}$, and overlap  $\mathbf{S}$ matrices of the tight-binding QM model 
(see Methods~\ref{si_qchem}), unless otherwise specified.

OrbNet-Equi 
learns a map $\mathcal{F}$ to approximate the target molecular property $\mathbf{y}$ of high-fidelity electronic structure simulations or experimental measurements,
\begin{equation}
\label{eq:learning}
   \min_{\mathcal{F}} \, \mathcal{L} \Big(\mathbf{y} , \mathcal{F}\big(\mathbf{T}[{\Psi_{0}}] \big) \Big) ,
\end{equation}
where 
$\mathcal{L}$ denotes a cost functional between the reference and predicted targets over training data.
The learning problem described by \eqref{eq:learning} requires a careful treatment on 
isometric coordinate transformations imposed on the molecular system,
because the coefficients of $\mathbf{T}[{\Psi_{0}}]$ are defined up to a given viewpoint (Figure \ref{fig:features}e). 
Precisely, the atomic orbitals $\{ \Phi_A^{n,l,m} \}$ 
undergo a unitary linear recombination subject to 3D rotations: 
$\mathcal{R} \cdot \lvert \Phi_A^{n,l,m} \rangle = \sum_{m'} \mathcal{D}^{l}_{m,m'}(\mathcal{R}) \lvert \Phi_A^{n,l,m'} \rangle$, 
where $\mathcal{D}^{l}(\mathcal{R})$ denotes the Wigner-D matrix of degree $l$ for a rotation operation $\mathcal{R}$.
As a consequence of the basis changing induced by $\mathcal{R}$, $\mathbf{T}[{\Psi_{0}}]$ is transformed block-wise:
\begin{equation}
\label{eq:basis_rot}
    \big( \mathcal{R} \cdot \mathbf{T}[{\Psi_{0}}] \big)^{l; l'}_{AB}
    = \mathcal{D}^{l}(\mathcal{R}) \, \big( \mathbf{T}[{\Psi_{0}}] \big)^{l; l'}_{AB} \, \mathcal{D}^{l'}(\mathcal{R})^{ \dagger}
\end{equation}
where the dagger symbol denotes an Hermitian conjugate. To account for the roto-translation symmetries, 
the neural network $\mathcal{F}$ must be made equivariant with respect to all such isometric basis rotations, that is,
\begin{equation}
\label{eq:rotsym}
    \mathcal{R} \cdot \mathcal{F}\big( \mathbf{T}[{\Psi_{0}}] \big) \equiv \mathcal{F}\big(\mathcal{R} \cdot \mathbf{T}[{\Psi_{0}}] \big) 
\end{equation}
which is fulfilled through our delicate design of the neural network in OrbNet-Equi (Figure \ref{fig:architect}). The neural network 
iteratively updates a set of representations $\mathbf{h}^t$ defined at each atom through its neural network modules, and reads out predictions using a pooling layer located at the end of the network.
During its forward pass, diagonal blocks of the inputs $\mathbf{T}[{\Psi_{0}}]$ are first transformed into components that are isomorphic to orbital-angular-momentum eigenstates, which are then cast to the initial representations $\mathbf{h}^{t=0}$. 
Each subsequent module 
exploits off-diagonal blocks of $\mathbf{T}[{\Psi_{0}}]$ to propagate non-local information among atomic orbitals and refine the representations $\mathbf{h}^{t}$, which resembles a process of applying time-evolution operators on quantum states. 
We provide a technical introduction to the neural network architecture in Methods~\ref{sec:nn}.
We incorporate 
other constraints on the learning task such as size-consistency solely through programming the pooling layer (Methods~\ref{si_pooling}), 
therefore achieving task-agnostic modelling for diverse chemical properties.
Additional details and theoretical results are provided in Appendix~\ref{si_implementation}-\ref{si_theory}.


\subsection*{Performance on benchmark datasets} 
\label{sec:qm9}

We begin with benchmarking OrbNet-Equi on the QM9 dataset~\cite{qm9} which has been widely adopted for assessing ML-based molecular property prediction methods. QM9 contains 134k small organic molecules 
at optimized geometries, with target properties computed by DFT. 
Following previous works~\cite{schutt2017schnet,unke2019physnet,cormorant,klicpera2020fast,schutt2021equivariant,spherenet}, 
we take 110000 random samples as the training set and 10831 samples as the test set.
We present results for both the ``direct-learning'' training strategy which corresponds to training the model directly on the target property, and, whenever applicable, the ``delta-learning'' strategy~\cite{ramakrishnan2015big} which corresponds to training on the residual between output of the tight-binding QM model and the target level of theory.

We first trained OrbNet-Equi on two representative targets, the total electronic energy $U_0$ and the molecular dipole moment vector $\vec{\mu}$ (Figure \ref{fig:qm9}a-b), for which a plethora of task-specific ML models has previously been developed \cite{Huang2020,Faber2018,bartok2017machine,orbnet1,christensen2019operators,veit2020predicting}. The total energy $U_0$ is predicted through a sum over atom-wise energy contributions and the dipole moment $\vec{\mu}$ is predicted through a combination of atomic partial charges and dipoles (Appendix~\ref{si_pooling}).  
For $U_0$ (Figure~\ref{fig:qm9}a), the direct-learning results of OrbNet-Equi match the state-of-the-art kernel-based ML method FCHL18/GPR~\cite{Faber2018} in terms of the test mean absolute error (MAE), while being scalable to large data regimes (Figure \ref{fig:qm9}a, training data size > 20,000) where no competitive result has been reported before.
With delta-learning, OrbNet-Equi outperforms our previous QM-informed ML approach OrbNet~\cite{orbnet1} by $\sim45\%$ in the test MAE. Because OrbNet also uses the GFN-xTB QM model for featurization and the delta-learning strategy for training, this improvement underscores the  
strength of our neural network design which seamlessly integrates the underlying physical symmetries. Moreover, for dipole moments $\vec{\mu}$ (Figure \ref{fig:qm9}b),
OrbNet-Equi exhibits 
steep learning curve slopes 
regardless of the training strategy, highlighting its capability of learning rotational-covariant quantities at no sacrifice of data efficiency. 

We then targeted on the learning task of frontier molecular orbital (FMO) properties, in particular energies of the highest occupied molecular orbital (HOMO), the lowest unoccupied molecular orbital (LUMO) and the HOMO-LUMO gaps which are important in the prediction of chemical reactivity and optical properties~\cite{fukui,silva2007experimental}. Because the FMOs are inherently defined in the electron energy space and are often spatially localized, it is expected to be challenging to predict FMO properties based on molecular representations in which a notion of electronic energy levels is absent.
OrbNet-Equi overcame this obstacle by breaking the orbital filling degeneracy of its input features to encode plausible electron excitations near the FMO energy levels, that is, adding energy-weighted density matrices of `hole-excitation' $\mathbf{D}^{\beta}_\mathrm{h}$ and that of `particle-excitation' $\mathbf{D}^{\beta}_\mathrm{p}$:

\begin{align}
    (D^{\beta}_\mathrm{h})_{\mu \nu} &= \sum_{i} C^{*}_{\mu i} C_{\nu i} \cdot \exp\big(-\beta (\epsilon_\mathrm{HOMO}-\epsilon_i)\big) \cdot n_i \\
    (D^{\beta}_\mathrm{p})_{\mu \nu} &= \sum_{i} C^{*}_{\mu i} C_{\nu i} \cdot \exp\big(\beta (\epsilon_\mathrm{LUMO}-\epsilon_i)\big) \cdot (1-n_i)
\end{align}
where $\epsilon_i$ and $n_{i}$ are the orbital energy and occupation number of the $i$-th molecular orbital from tight-binding QM, and $C_{\mu i}, C_{\nu i}$ denotes the molecular orbital coefficients with $\mu$ and $\nu$ indexing the atomic orbital basis. Here the effective temperature parameters $\beta$ are chosen as $\beta = [4, 16, 64, 256]$ (atomic units), and a global-attention based pooling is used to ensure size-intensive predictions (Appendix~\ref{si_mo_feats}).  Figure~\ref{fig:qm9}c shows that the inclusion of energy-weighted density matrices ($\mathbf{D}^{\beta}_\mathrm{h}, \mathbf{D}^{\beta}_\mathrm{p}$) indeed greatly enhanced model generalization on FMO energies, as evident from the drastic test MAE reduction against the model with default ground-state features ($\mathbf{F, P, S, H}$) as well as the best result from Atomistic ML methods.
Remarkably, for models using default ground-state features (Figure~\ref{fig:qm9}c, red lines) we noticed a rank reversal behavior between direct-learning and delta-learning models as more training samples become available, mirroring similar observations from a recent Atomistic ML study~\cite{delfta}. 
The absence of this crossover when ($\mathbf{D}^{\beta}_\mathrm{h}, \mathbf{D}^{\beta}_\mathrm{p}$) are provided (Figure~\ref{fig:qm9}c, blue) suggests that the origin of such a learning slow-down is the incompleteness of spatially-degenerate descriptors, and the gap between delta-learning and direct-learning curves can be restored by breaking the energy-space degeneracy.
This analysis reaffirms the role of identifying the dominant physical degrees of freedom 
in the context of ML-based prediction of quantum chemical properties, 
and is expected to benefit the modelling of relevant electrochemical and optical properties such as redox potentials.

Furthermore, OrbNet-Equi is benchmarked on 12 targets of QM9 using the 110k full training set (Appendix~Table \ref{table:qm9}), for which we programmed its pooling layer to reflect the symmetry constraint of each target property (Appendix~\ref{si_pooling}). We observed top-ranked performance on all targets with average test MAE around two-fold lower than atomistic deep learning methods. In addition, we tested OrbNet-Equi on fitting molecular potential energy surfaces by training on multiple configurations of a molecule (Appendix~\ref{sec:md17}). Results (Appendix~Table~\ref{table:rmd17}-\ref{table:md17}) showed that OrbNet-Equi obtained energy and force prediction errors that match state-of-the-art machine learning potential methods~\cite{rmd17, nequip} on the MD17 dataset~\cite{md17,rmd17}, suggesting that our method also efficiently  
generalizes over the conformation degrees of freedom apart from being transferable across the chemical space. These extensive benchmarking studies confirm that our strategy is consistently applicable to a wide range of molecular properties.

\subsection*{Accurate modelling for electron densities} 
\label{sec:density}

We next focus on the task of predicting the electron density 
$\rho(\vec{r}): \mathbb{R}^{3}\rightarrow \mathbb{R}$ 
which plays an essential role in both the formulation 
of DFT and in the interpretation of molecular interactions. It is also more challenging than predicting the energetic properties from a machine learning perspective, due to the need of 
preserving its real-space continuity and rotational covariance. OrbNet-Equi learns to output a set of expansion coefficients $\hat{d}_{A}^{n l m}$ 
to represent the predicted electron density $\hat{\rho}(\vec{r})$ through a density fitting basis set $\{\chi\}$
(Methods~\ref{si_qcore}, Appendix~\ref{sec:dens_pooling}),
\begin{equation}
    \label{eq:des_expand}
    \hat{\rho}(\vec{r}) = \sum_A^{N_\mathrm{atom}} \sum_{l}^{l_\mathrm{max}(z_A)} \sum_{m=-l}^{l} \sum_{n}^{n_\mathrm{max}(z_A, l)} \hat{d}_{A}^{n l m} \ \chi_A^{n l m}(\vec{r}) 
\end{equation}
where $l_\mathrm{max}(z_A)$ is the maximum angular momentum in the density fitting basis set for atom type $z_A$, and $n_\mathrm{max}(z_A, l)$ denotes the cardinality of basis functions with angular momentum $l$.
We train OrbNet-Equi to learn DFT electron densities on the QM9 dataset of small organic molecules and the BfDB-SSI~\cite{BFGDBBurns2017} dataset of amino-acid side-chain dimers (Figure~\ref{fig:density}) using the direct-learning strategy. OrbNet-Equi results are substantially better than Atomistic ML baselines in terms of the average $L^1$ density error 
$\varepsilon_{\rho} = \frac{\int \lvert \rho(\vec{r})-\hat{\rho}(\vec{r})\rvert d \vec{r}}{\int \lvert \rho(\vec{r}) \rvert d \vec{r}}$ (Methods~\ref{si_summary}); specifically, OrbNet-Equi achieves an average $\varepsilon_{\rho}$ of 0.191$\pm$0.003\% on BfDB-SSI using 2000 training samples compared to 0.29\% of SA-GPR~\cite{fabrizio2019electron}, and an average $\varepsilon_{\rho}$ of 0.206$\pm$0.001\% on QM9 using 123835 training samples as compared to 0.28\%-0.36\% of DeepDFT~\cite{jorgensen2020deepdft}. 
Figure~\ref{fig:density}a confirms that OrbNet-Equi predicts densities at consistently low errors across the real-space and maintains a robust asymptotic decay behavior within low-density ($\rho(\vec{r})<10^{-4}\ a_0^{-3}$) regions that are far from the molecular system. 

To understand whether the model generalizes to cases where charge transfer is significant, as in donor-acceptor systems, we introduce a simple baseline predictor termed monomer density superposition (MDS). The MDS electron density of a dimeric system is taken as the sum of independently-computed DFT electron densities of the two monomers. OrbNet-Equi yields accurate predictions in the presence of charge redistribution induced by non-covalent effects, as identified by dimeric examples from the BfDB-SSI test set for which the MDS density (Figure~\ref{fig:density}b, x-axis) largely deviates from the DFT reference density of the dimer due to inter-molecular interactions. One representative example is a strongly interacting Glutamic acid - Lysine system (Figure~\ref{fig:density}, c-d) whose salt-bridge formation is known to be essential for the helical stabilization in protein folding~\cite{marqusee1987helix}, for which OrbNet-Equi predicts $\rho(\vec{r})$ with $\varepsilon_{\rho}= 0.211\pm0.001\%$  significantly lower than that of monomer density superposition ($\varepsilon_{\rho}=1.47\pm0.02\%$).
The accurate modelling of $\rho(\vec{r})$ offers an opportunity for constructing transferable DFT models for extended systems by learning on both energetics and densities, while at a small fraction of expense relative to solving the Kohn-Sham equations from scratch.


\subsection*{Transferability on downstream tasks}
\label{sec:downstream}



Beyond data efficiency on established datasets in train-test split settings, a crucial but highly challenging aspect is whether the model accurately infers downstream properties after being trained on data that are feasible to obtain. 
To comprehensively evaluate whether OrbNet-Equi 
can be transferred to unseen chemical spaces without any additional supervision, we have trained an OrbNet-Equi model on a dataset curated from readily available small-molecule databases (Methods~\ref{si_dataset}). The training dataset contains 236k samples with chemical space coverage for drug-like molecules and biological motifs containing chemical elements C, O, N, F, S, Cl, Br, I, P, Si, B, Na, K, Li, Ca and Mg, and thermalized geometries. The resulting OrbNet-Equi/SDC21 potential energy model is solely trained on DFT single-point energies using the delta-learning strategy. Without any fine-tuning, we directly apply OrbNet-Equi/SDC21 to downstream benchmarks that are recognized for assessing the accuracy of physics-based molecular modelling methods.

The task of ranking conformer energies of drug-like molecules is benchmarked via the Hutchison dataset of conformers of $\sim$700 molecules~\cite{Hutch}  (Figure~\ref{fig:6}a; Table~\ref{table:downstream}, row 1-2).
On this task, OrbNet-Equi/SDC21 achieves a median $R^2$ score of 0.87±0.02 and $R^2$ distributions closely matching the reference DFT theory on both neutral and charged systems. On the other hand, we notice that the median $R^2$ of OrbNet-Equi/SDC21 with respect to the reference DFT theory ($\omega$B97X-D3/def2-TZVP) is 0.96±0.01, suggesting that the current performance on this task is saturated by the accuracy of DFT and can be systematically improved by applying fine-tuning techniques on higher-fidelity labels~\cite{smith2019approaching,zheng2021artificial}. Timing results on the Hutchison dataset (Table~\ref{table:infer}) confirms that the neural network inference time of OrbNet-Equi/SDC21 is on par with the GFN-xTB QM featurizer, resulting in an overall computational speed that is $10^{2-3}$ fold faster relative to existing cost-efficient composite DFT methods~\cite{Hutch,brandenburg2018b97,grimme2021r2scan}.
To understand the model's ability to describe dihedral energetics which are crucial for virtual screening tasks,
%
we benchmark OrbNet-Equi on the prediction of intra-molecular torsion energy profiles using the TorsionNet500~\cite{rai2020torsionnet} dataset, the most diverse benchmark set available for this problem (Table~\ref{table:downstream}, row 3). Although no explicit torsion angle sampling was performed during training data generation, OrbNet-Equi/SDC21 exhibits a barrier MAE of 0.173$\pm$0.003 kcal/mol much lower than the 1 kcal/mol threshold commonly considered for chemical accuracy. On the other hand, we notice a MAE of 0.7 kcal/mol 
for the TorsionNet model~\cite{rai2020torsionnet} which was trained on $\sim$1 million torsion energy samples. As shown in Figure~\ref{fig:6}b, OrbNet-Equi/SDC21 robustly captures the torsion sectors of potential energy surface on an example challenging for both semi-empirical QM~\cite{gfn1,gfn2} and cost-efficient composite DFT~\cite{brandenburg2018b97} methods, precisely resolving both the sub-optimal energy minima location at $\sim30^{\circ}$ dihedral angle as well as the barrier energy between two local minimas within a 1 kcal/mol chemical accuracy. 
Next, the ability to characterize non-covalent interactions is assessed on the S66x10 dataset~\cite{smith2016revised} of inter-molecular dissociation curves (Table~\ref{table:s66x10}), on which OrbNet-Equi achieves an equilibrium-distance binding energy MAE of $0.35\pm0.09$ kcal/mol with respect to the reference DFT theory compared against 
$1.55\pm0.17$ kcal/mol of the GFN-xTB baseline. As shown from a Uracil-Uracil base-pair example (Figure~\ref{fig:6}c) for which high-fidelity wavefunction-based reference calculations have been reported, the binding energy curve along the inter-molecular axis predicted by OrbNet-Equi/SDC21 agrees well with both DFT and the high-level CCSD(T) results. To further understand the accuracy and smoothness of the energy surfaces and the applicability on dynamics tasks, we perform geometry optimizations on the ROT34 dataset of 12 small organic molecules and the MCONF dataset of 52 conformers of melatonin~\cite{rot34,mconf} (Figure~\ref{fig:6}d; Table~\ref{table:downstream}, row 4-5). Remarkably, OrbNet-Equi/SDC21 consistently exhibits the lowest average RMSD among all physics-based and ML-based approaches (Table~\ref{table:downstream}) including the popular cost-efficient DFT method B97-3c~\cite{brandenburg2018b97}. Further details regarding the numerical experiments and error metrics are provided in Methods~\ref{si_summary}.

Remarkably, on the G21IP dataset~\cite{curtiss1991gaussian} of adiabatic ionization potentials, we find that the OrbNet-Equi/SDC21 model achieves prediction errors substantially lower than semi-empirical QM methods (Figure~\ref{fig:6}e, Table~\ref{table:gmtkn55}) even though samples of open-shell signatures are expected to be rare from the training set (Methods~\ref{si_dataset}). Such an improvement cannot be solely attributed to structure-based corrections, since there is no or negligible geometrical changes between the neutral and ionized species for both the single-atom systems and several poly-atomic systems (e.g., IP\_66, a Phosphanide anion) in the G21IP dataset. This reveals that our method has the potential to be transferred to unseen electronic states in a zero-shot manner, which represents an early evidence that a hybrid physics-ML strategy may unravel novel chemical processes such as unknown electron-catalyzed reactions~\cite{studer2014electron}.

To comprehensively study the transferability of OrbNet-Equi on complex under-explored main-group chemical spaces, we evaluate OrbNet-Equi/SDC21 on the challenging, community-recognized benchmark collection of the General Thermochemistry, Kinetics, and Non-covalent Interactions (GMTKN55)\cite{gmtkn55} datasets  (Figure~\ref{fig:7}).
Prediction error statistics on the GMTKN55 benchmark are reported with three filtration schemes. First, we evaluate the WTMAD error metrics (Methods~\ref{si_summary}) on reactions that only consist of neutral and closed-shell molecules with chemical elements CHONFSCl (Figure~\ref{fig:7}a), as is supported by an 
Atomistic-ML-based potential method, ANI-2x~\cite{ani2x}, which is trained on large-scale DFT data. 
OrbNet-Equi/SDC21 predictions are found to be highly accurate on this subset, as seen from the WTMAD with respect to CCSD(T) being on par with the DFT methods on all five reaction classes and significantly outperforming ANI-2x and the GFN family of semi-empirical QM methods~\cite{gfn1,gfn2}. It is worth noting that OrbNet-Equi/SDC21 uses much fewer number of training samples than the ANI-2x training set, 
which signifies the effectiveness of combining physics-based and ML-based modelling.

The second filtration scheme includes reactions that consist of closed-shell - but can be charged - molecules with chemical elements that have appeared in the SDC21 training dataset (Figure~\ref{fig:7}b). Although all chemical elements and electronic configurations in this subset are contained in the training dataset, we note that unseen types of physical interactions or bonding types are included, such as in alkali metal clusters from the ALK8 subset~\cite{gmtkn55} and short strong hydrogen bonds in the AHB21 subset~\cite{lao2015accurate}. Therefore, assessments of OrbNet-Equi with this filtration strategy reflect its performance on cases where examples of atom-level physics are provided but the the chemical compositions are largely unknown. Despite this fact, the median WTMADs of OrbNet-Equi/SDC21 are still competitive to DFT methods on the tasks of small-system properties, large-system properties and intra-molecular interactions. On reaction barriers and inter-molecular non-covalent interactions (NCIs), OrbNet-Equi/SDC21 results fall behind DFT
but still show improvements against the GFN-xTB baseline and match the accuracy of GFN2-xTB which is developed with physic-based schemes to improve the descriptions on NCI properties against its predecessor GFN-xTB. 

The last scheme includes all reactions in the GMTKN55 benchmarks containing chemical elements and spin states never seen during training (Figure~\ref{fig:7}c), which represents on the most stringent test and reflects the performance of OrbNet-Equi/SDC21 when being indiscriminately deployed as a quantum chemistry method. When evaluated on the collection of all GMTKN55 tasks (Figure~\ref{fig:7}, `Total' panel), OrbNet-Equi/SDC21 maintains the lowest median WTMAD among methods considered here that can be executed at the computational cost of semi-empirical QM calculations. Moreover, we note that failure modes on a few highly extrapolative subsets can be identified to diagnose cases that are challenging for the QM model used for featurization (Table~\ref{table:gmtkn55}). For example, the fact that predictions being inaccurate on the W4-11 subset of atomization energies~\cite{karton2011w4} and the G21EA subset of electron affinities~\cite{curtiss1991gaussian} parallels the absence of an explicit treatment of triplet or higher-spin species within the formulation of GFN family of tight-binding models. On the population level, the distribution of prediction WTMADs across GMTKN55 tasks also differ from that of GFN2-xTB, which implies that further incorporating physics-based approximations into the QM featurizer can complement the ML model, and thus the accuracy boundary of semi-empirical methods can be pushed to a regime where no known physical approximation is feasible.



\section*{Discussion}


We have introduced OrbNet-Equi, a QM-informed geometric deep learning framework for learning molecular or material properties using representations in the atomic orbital basis.
OrbNet-Equi shows excellent data efficiency for learning  related to  both energy-space and real-space properties, expanding the diversity of molecular properties that can be modelled by QM-informed machine learning. 
Despite only using readily available small-molecule libraries as training data, OrbNet-Equi offers an accuracy alternative to DFT methods on comprehensive main-group quantum chemistry benchmarks at a computation speed on par with semi-empirical methods, thus offering a possible replacement for conventional \textit{ab initio} simulations for general-purpose downstream applications.
For example, OrbNet-Equi could immediately facilitate applications such as  screening electrochemical properties of electrolytes for the design of flow batteries~\cite{janet2020accurate}, and performing accurate direct or hybrid QM/MM simulations for reactions in transition-metal catalysis~\cite{chen2020embedded,li2022dispersion}. The method can also improve the 
modelling for complex reactive biochemical processes~\cite{lampret2020roles} using multi-scale strategies that have been demonstrated in our previous study~\cite{dommer2021covidisairborne}, while conventional \textit{ab initio} reference calculations can be prohibitively expensive even on a minimal sub-system. 

The demonstrated transferability of OrbNet-Equi to seemingly dissimilar chemical species identifies a promising future direction of improving the accuracy and chemical space coverage through adding simple model systems of the absent types of physical interactions to the training data, a strategy that is consistent with using synthetic data to improve ML models~\cite{chen2020automated} which has been demonstrated for improving the accuracy of DFT functionals~\cite{kirkpatrick2021pushing}. Additionally, OrbNet-Equi may provide valuable perspectives for the development of physics-based QM models by relieving the burden of parameterizing Hamiltonian parameters against specific target systems, potentially expanding their design space to higher energy-scales without sacrificing model accuracy. 
Because the framework presented here can be readily extended to alternative quantum chemistry models for either molecular or material systems, we expect OrbNet-Equi to broadly benefit studies in chemistry, materials science, and biotechnology. 

\section{Methods}

\subsection{The UNiTE neural network}

\label{sec:unite}
\label{sec:nn}


This section introduces Unitary N-body Tensor Equivariant Network (UNiTE), the neural network model developed for the OrbNet-Equi method to enable learning equivariant maps between the input atomic orbital features $\mathbf{T}[\Psi_0]$ and the property predictions $\hat{\mathbf{y}}[\Psi_0]$. Given the inputs $\mathbf{T}$, UNiTE
first generates initial representations $\mathbf{h}^{t=0}$ through its diagonal reduction module (Methods \ref{sec:wigner}). 
Then UNiTE updates the representations  $\mathbf{h}^{t=0} \mapsto \mathbf{h}^{t=1} \mapsto \cdots \mapsto \mathbf{h}^{t=t_f}$ with $t_1$ stacks of block convolution (Methods \ref{sec:blockwise}), message passing  (Methods \ref{sec:message_passing}), and point-wise interaction  (Methods \ref{sec:interaction}) modules, followed by $t_2$ stacks of point-wise interaction modules.
A pooling layer (Methods \ref{sec:pooling}) outputs predictions $\hat{\mathbf{y}}$ using the final representations $\mathbf{h}^{t=t_f}$  at $t_f=t_1+t_2$ as inputs. 

$\mathbf{h}^{t}$ is a stack of atom-wise representations, i.e., for a molecular system containing $d$ atoms, $\mathbf{h}^{t}\defeq [\mathbf{h}^{t}_{1}, \mathbf{h}^{t}_{2}, \cdots, \mathbf{h}^{t}_{d}]$. The representation for the $A$-th atom, $\mathbf{h}^{t}_{A}$, is a concatenation of neurons which are associated with irreducible representations of group $\mathrm{O(3)}$. Each neuron in $\mathbf{h}^{t}_{A}$ is identified by a channel index $n\in \{1, 2, \cdots , n_{\mathrm{max}}\}$, a "degree" index $l\in \{0, 1, 2, \cdots, l_{\mathrm{max}}\}$, and a "parity" index $p\in \{+1, -1\}$. The neuron $\mathbf{h}^{t}_{A,nlp}$ is a vector of length $2l+1$ and transforms as the $l$-th irreducible representation of group $\mathrm{SO(3)}$; i.e., $\mathbf{h}^{t}_{A,nlp} = \bigoplus_m {h}^{t}_{A,nlpm}$ where $\oplus$ denotes a vector concatenation operation and $m \in \{-l, -l+1, \cdots, l-1, l\}$. We use $N_{lp}$ to denote the number of neurons with degree $l$ and parity $p$ in $\mathbf{h}^t$, and $N\defeq\sum_{l,p}N_{lp}$ to denote the total number of neurons in $\mathbf{h}^t$.

For a molecular/material system with atomic coordinates $\mathbf{x} \in \mathbb{R}^{d\times3}$, the following equivariance properties with respect to isometric Euclidean transformations are fulfilled for any input gauge-invariant and Hermitian operator $\mathcal{O}[\Psi_0]$; for all allowed indices $A, n, l, p, m$:
\begin{itemize}
\item Translation invariance:
\begin{equation}
\mathbf{h}^{t}_{A,nlpm} \mapsto \mathbf{h}^{t}_{A,nlpm} \quad \textrm{for} \quad \mathbf{x} \mapsto \mathbf{x} + \mathbf{x}_0 \label{eq:ts_eqv}
\end{equation}
where $\mathbf{x}_0 \in \mathbb{R}^{3} $ is an arbitrary global shift vector; 

\item Rotation equivariance:
\begin{equation}
    \mathbf{h}^{t}_{A,nlpm} \mapsto \sum_{m'} \mathbf{h}^{t}_{A,nlpm'} \mathcal{D}^{l}_{m,m'}(\alpha, \beta, \gamma) \label{eq:rot_eqv}
\end{equation}
for $\mathbf{x} \mapsto \mathbf{x} \cdot \mathcal{R}(\alpha, \beta, \gamma) $ where $\mathcal{R}(\alpha, \beta, \gamma)$ denotes a rotation matrix corresponding to standard Euler angles $\alpha, \beta, \gamma$;

\item Parity inversion equivariance:
\begin{equation}
    \mathbf{h}^{t}_{A,nlpm} \mapsto (-1)^{l} \cdot p \cdot \mathbf{h}^{t}_{A,nlpm} \quad \textrm{for} \quad \mathbf{x} \mapsto -\mathbf{x}  \label{eq:inv_eqv}
\end{equation}
\end{itemize}

The initial vector representations $\mathbf{h}^{t=0}$ are generated by decomposing diagonal sub-tensors of 
the input $\mathbf{T}$ into a spherical-tensor representation without explicitly solving tensor factorization, 
based on the tensor product property of group $\mathrm{SO(3)}$.
The intuition behind this operation is that the diagonal sub-tensors of $\mathbf{T}$ can be viewed as isolated systems 
interacting with an effective external field, whose rotational symmetries are described by the Wigner-Eckart Theorem~\cite{sakurai1995modern} which links tensor operators to their spherical counterparts and applies here within a natural generalization.
Each update step $\mathbf{h}^{t}\mapsto \mathbf{h}^{t+1}$ is composed of 
(a) block convolution, 
(b) message passing, and
(c) point-wise interaction modules which are all equivariant with respect to index permutations and basis transformations.
%
In an update step $\mathbf{h}^{t}\mapsto \mathbf{h}^{t+1}$, each off-diagonal block of 
$\mathbf{T}$ corresponding to a pair of atoms is contracted with $\mathbf{h}^{t}$. This block-wise contraction operation can be interpreted as performing 
local convolutions using the blocks of $\mathbf{T}$ as convolution kernels, and therefore is called \textit{block convolution} module. The output block-wise representations 
are then passed into a \textit{message passing} module, 
which is analogous to a message-passing operation on edges in graph neural networks~\cite{gilmer2017neural}. 
The message passing outputs are then fed into a \textit{point-wise interaction} module with the previous-step representation $\mathbf{h}^{t}$ 
 to finish the update 
$\mathbf{h}^{t}\mapsto \mathbf{h}^{t+1}$. 
The point-wise interaction modules are constructed as a stack of multi-layer perceptrons (MLPs), Clebsch-Gordan product operations and skip connections. 
Within those modules, a \textit{matching layer} assigns the channel indices of $\mathbf{h}^{t}$ to indices of the atomic orbital basis.

We also introduce a normalization layer termed Equivariant Normalization (EvNorm, see Methods \ref{sec:evnorm}) to improve training and generalization of the neural network.
EvNorm 
normalizes scales of the representations $\mathbf{h}$, while recording the direction-like information to be recovered afterward.  
EvNorm is fused with a point-wise interaction module through first applying EvNorm to the module inputs, then using an MLP to transform the normalized frame-invariant scale information, and finally by multiplying the recorded direction vector to the MLP's output. 
Using EvNorm within the point-wise interaction modules is found to stabilize training
and eliminate the need for tuning weight initializations and learning rates across different tasks. 


The explicit expressions for the neural network modules are provided for quantum operators $\mathcal{O}$ being one-electron operators and therefore the input tensors $\mathbf{T}$ is a stack of matrices (i.e., order-2 tensors). Without loss of generality, we also assume that $\mathbf{T}$ contains only one feature matrix. Additional technical aspects regarding the case of multiple input features, the inclusion of geometric descriptors, and implementation details are discussed in Appendix \ref{si_implementation}. The proofs regarding equivariance and theoretical generalizations to order-$N$ tensors are provided in Appendix~\ref{si_theory}. 

\subsubsection{The Diagonal Reduction module} 

\label{sec:wigner}

We define the shorthand notations $\mu\defeq(n_1, l_1, m_1)$ and $\nu\defeq(n_2, l_2, m_2)$ to index atomic orbitals. The initialization scheme for $\mathbf{h}^{t=0}$ is based on the following proposition: 
%
for each diagonal block of $\mathbf{T}$, $\mathbf{T}_{AA}$, defined for an on-site atom pair $(A, A)$,
\begin{equation}
    T_{AA}^{\mu,\nu} \defeq \langle \Phi_A^{\mu} \lvert \hat{\mathcal{O}} \rvert  \Phi_A^{\nu} \rangle ,
\end{equation}
there exists a set of $\mathbf{T}$-independent coefficients $Q^{\mu,\nu}_{nlpm}$ such that 
the following linear transformation $\psi$ 
\begin{equation}
    \label{eq:nto1}
     \psi(\mathbf{T}_{AA})_{nlpm} \defeq \sum_{\mu,\nu} T_{AA}^{\mu,\nu} \, Q^{\mu,\nu}_{nlpm} 
\end{equation}
is injective and yields $\mathbf{h}_A \defeq \psi(\mathbf{T}_{AA})$ that satisfy equivariance (\eqref{eq:ts_eqv}-\eqref{eq:inv_eqv}). 

The existence of $\mathbf{Q}$ is discussed in Appendix, Corollary~\ref{cor:wigner}. 
For the sake of computational feasibility, a physically-motivated scheme is employed to tabulate $\mathbf{Q}$ and produce order-1 equivariant embeddings $\mathbf{h}_A$, using \textit{on-site 3-index overlap integrals} $\tilde{\mathbf{Q}}$: 
\begin{align}
    \label{eq:3idx}
    \tilde{Q}^{\mu,\nu}_{nlm} &\defeq \tilde{Q}^{n_1, l_1, m_1; n_2, l_2, m_2}_{nlm} \nonumber \\
    &= \int_{\mathbf{r}\in\mathbb{R}^{3}} (\Phi_A^{n_1, l_1, m_1}(\mathbf{r}))^{*} \Phi_A^{n_2, l_2, m_2}(\mathbf{r}) \tilde{\Phi}_A^{n, l, m}(\mathbf{r}) d \mathbf{r} 
\end{align}
where $\Phi_A$ are the atomic orbital basis, and $\tilde{\Phi}_A$ are auxiliary Gaussian-type basis functions defined as (for conciseness, at $\mathbf{x}_A=0$):
\begin{equation}
    \tilde{\Phi}^{n,l,m}(\mathbf{r}) \defeq  c_{n,l} \cdot \exp(-\gamma_{n,l} \cdot r^{2}) \, r^{l} \, Y_{lm}(\frac{\mathbf{r}}{r})
\end{equation}
where $c_{n,l}$ is a normalization constant such that $\int_{\mathbf{r}} \lvert \lvert \tilde{\Phi}_A^{n, l, m}(\mathbf{r}))\lvert \lvert ^{2} d \mathbf{r}=1$ following standard conventions~\cite{dunning1977gaussian}. For numerical experiments considered in this work the scale parameters $\gamma$ are chosen as (in atomic units):
\begin{align*}
    \gamma_{n,l=0} &\defeq 128 \cdot (0.5)^{n-1} \quad \textrm{where} \quad n\in \{1, 2, \cdots, 16 \} \\
    \gamma_{n,l=1} &\defeq 32 \cdot (0.25)^{n-1} \quad \textrm{where} \quad n\in \{1, 2, \cdots, 8 \} \\
    \gamma_{n,l=2} &\defeq 4.0 \cdot (0.25)^{n-1} \quad \textrm{where} \quad n\in \{1, 2, 3, 4 \} 
\end{align*}
$\tilde{\mathbf{Q}}$ adheres to equivariance constraints 
due to its relation to $\mathrm{SO(3)}$ Clebsch-Gordan coefficients $C_{l_1 m_1;l_2 m_2}^{lm} \propto \int_{\mathbf{r}\in\mathbb{S}^{2}} Y_{l_1 m_1}(\mathbf{r}) Y_{l_2 m_2}(\mathbf{r}) (Y_{l m}(\mathbf{r}))^{*} d \mathbf{r}$~\cite{sakurai1995modern}.
Note that the auxiliary basis $\tilde{\Phi}_A$ is independent of the atomic numbers thus the resulting $\mathbf{h}_A$ are of equal length for all chemical elements. 
$\tilde{\mathbf{Q}}$ can be efficiently generated using electronic structure programs, here done with \cite{entos}. 
The resulting $\mathbf{h}_A$ in explicit form are
\begin{align*}
    \mathbf{h}_A &\defeq \bigoplus_{n,l,p,m} h_{A,nlpm} \quad \textrm{where} \\
    h_{A,nl(p=+1)m} &=
    \sum_{\mu, \nu} T_{AA}^{\mu,\nu} \, \tilde{Q}^{\mu,\nu}_{nlm} \\
    h_{A,nl(p=-1)m} &= 0
\end{align*}
$\mathbf{h}_A$ are then projected by learnable linear weight matrices such that the number of channels for each $(l,p)$ matches the model specifications. The outputs are regarded as the initial representations $\mathbf{h}^{t=0}$ to be passed into other modules.


\subsubsection{The Block Convolution module} 

\label{sec:blockwise}

In an update step $\mathbf{h}^{t}\mapsto \mathbf{h}^{t+1}$, sub-blocks of $\mathbf{T}$ are first contracted with a stack of linearly-transformed order-1 representations $\mathbf{h}^{t}$. 
\begin{equation}
    \label{eq:conv}
    \mathbf{m}_{AB,\nu}^{t,i} = \sum_{\mu} \big( \rho_i(\mathbf{h}^{t}_{A}) \big)_{\mu} \,  T_{AB}^{\mu, \nu}
\end{equation}
which can be viewed as a 1D convolution between each block $\mathbf{T}_{AB}$ (as convolution kernels) and the $\rho(\mathbf{h}^{t}_{A})$ (as the signal) in the $i$-th channel where $i \in \{1, 2, \cdots, I \}$ is the convolution channel index. The block convolution produces block-wise representations $\mathbf{m}_{AB}^{t}$ for each block index $(A, B)$. $\rho_i$ is called a \textit{matching layer} at atom $A$ and channel $i$, defined as:
\begin{equation}
    \big( \rho_i(\mathbf{h}^{t}_{A}) \big)_{\mu} = \mathrm{Gather}\big(\mathbf{W}_{l}^{i} \cdot (\mathbf{h}_A^{t})_{l(p=+1)m}, n[\mu, z_A] \big)
\end{equation}
$\mathbf{W}_{l}^{i} \in \mathbb{R}^{M_{l} \times N_{l,+1}}$ are learnable linear weight matrices specific to each degree index $l$, where $M_{l}$ is the maximum principle quantum number for shells of angular momentum $l$ within the atomic orbital basis used for featurization. The $\mathrm{Gather}$ operation maps the feature dimension to valid atomic orbitals by indexing $\mathbf{W}_{l}^{i} \cdot (\mathbf{h}_A^{t})_{l(p=+1)m}$ using $n[\mu, z_A]$, the principle quantum numbers of atomic orbitals $\mu$ for atom type $z_A$.

\subsubsection{The Message Passing module} Block-wise representations $\mathbf{m}_{AB}^{t}$ are then aggregated into each atom index $A$ by summing over the indices $B$, analogous to a `message-passing' between nodes and edges in common realizations of graph neural networks \cite{gilmer2017neural},
\begin{equation}
    \tilde{\mathbf{m}}_{A}^t = \sum_B \bigoplus\limits_{i,j} \ \mathbf{m}_{BA}^{t,i}  \cdot \alpha_{AB}^{t,j} \label{eq:mp} 
\end{equation}
up to a non-essential symmetrization and inclusion of point-cloud geometrical terms (\eqref{eq:qc_accu}). $\alpha_{AB}^{t,j}$ in \eqref{eq:mp} are scalar-valued weights parameterized as \textit{SE(3)-invariant multi-head attentions}:
\begin{equation}
    \label{eq:attn}
    \bm{\alpha}^{t}_{AB} = \mathrm{MLP}\big((\mathbf{z}_{AB}^{t} \cdot \mathbf{W}^{t}_{\upalpha}) \odot \bm{\kappa}(\lvert \lvert \mathbf{T}_{AB}\lvert \lvert ) / \sqrt{N}\big)  
\end{equation}
where $\odot$ denotes an element-wise (Hadamard) product, and
\begin{equation}
    \mathbf{z}_{AB}^{t} = \bigoplus_{n,l,p}\sum_{m=-l}^{l} h^{t}_{A, nlpm} \cdot h^{t}_{B, nlpm} 
\end{equation}
where $\mathrm{MLP}$ denotes a 2-layer MLP, $\mathbf{W}^{t}_{\upalpha}$ are learnable linear functions and $j\in\{1,2,\cdots,J\}$ denotes an attention head (one value in $\bm{\alpha}^{t}_{AB}$). $\bm{\kappa}(\cdot)$ is chosen as Morlet wavelet basis functions:
\begin{align}
    \bm{\kappa}(\lvert \lvert \mathbf{T}_{AB}\lvert \lvert ) &\defeq \mathbf{W}_{\upkappa}\Big(\bigoplus_{k} \sum_{n,l}\sum_{n',l'} \xi_{k}(\log \big(\lvert \lvert \mathbf{T}_{AB}^{n,l;n',l'}\lvert \lvert )\big)\Big) \\
    \xi_{k}(x) &\defeq \exp(-\gamma_k \cdot x^2) \cdot \cos(\pi \gamma_k \cdot x) \label{eq:rbf}
\end{align}
where $\mathbf{W}^{t}_{\upkappa}$ are learnable linear functions and $\gamma_k$ are learnable frequency coefficients initialized as $\gamma_k = 0.3 \cdot (1.08)^{k}$ where $k\in \{0, 1, \cdots, 15\}$. Similar to the scheme proposed in $\mathrm{SE(3)}$-transformers \cite{fuchs2020se}, the attention mechanism \eqref{eq:attn} 
improves the network capacity without increasing memory costs as opposed to explicitly expanding $\mathbf{T}$.

The aggregated message $\tilde{\mathbf{m}}_A^t$ is combined with the representation of current step $\mathbf{h}_{A}^{t}$ through a point-wise interaction module $\phi$ (see Methods~\ref{sec:interaction}) to complete the update $\mathbf{h}_{A}^{t} \mapsto \mathbf{h}_{A}^{t+1}$. 

\label{sec:message_passing}

\subsubsection{Equivariant Normalization (EvNorm)} 
\label{sec:evnorm}
We define $\mathrm{EvNorm}: \mathbf{h} \mapsto ( \bar{\mathbf{h}}, \hat{\mathbf{h}})$ where $\bar{\mathbf{h}}$ and $\hat{\mathbf{h}}$ are given by
\begin{equation}
    \label{eq:EvNorm}
    \bar{\mathbf{h}}_{nlp} \defeq \frac{\lVert \mathbf{h}_{nlp}\rVert -\mu^h_{nlp}}{\sigma^h_{nlp}} \quad \textrm{and} \quad
    \hat{\mathbf{h}}_{nlpm} \defeq \frac{\mathbf{h}_{nlpm}}{\lVert \mathbf{h}_{nlp} \rVert + 1/\beta_{nlp} + \epsilon}
\end{equation}
where $\lVert \cdot \rVert$ denotes taking a neuron-wise regularized $L^2$ norm: 
\begin{equation}
    \lVert \mathbf{h}_{nlp} \rVert \defeq \sqrt{\sum_{m} \mathbf{h}_{nlpm}^2 + \epsilon^{2}} - \epsilon   
\end{equation}
$\mu^h_{klp}$ and $\sigma^h_{klp}$ are mean and variance estimates of the invariant content $\lVert \mathbf{h} \rVert$ that can be obtained from either batch or layer statistics as in normalization schemes developed for scalar neural networks~\cite{ioffe2015batch,ba2016layer}; $\beta_{klp}$ are positive, learnable scalars controlling the fraction of vector scale information from $\mathbf{h}$ to be retained in $\hat{\mathbf{h}}$, and $\epsilon$ is a numerical stability factor.
The EvNorm operation \eqref{eq:EvNorm} decouples $\mathbf{h}$ to the normalized frame-invariant representation $\bar{\mathbf{h}}$ suitable for being transformed by an MLP, and a `pure-direction' $\hat{\mathbf{h}}$ that is later multiplied to the MLP-transformed normalized invariant content to finish updating $\mathbf{h}$. Note that in \eqref{eq:EvNorm}, $\mathbf{h}=\mathbf{0}$ is always a fixed point of the map $\mathbf{h} \mapsto \hat{\mathbf{h}}$ and the vector directions information $\mathbf{h}$ is always preserved.


\subsubsection{The Point-wise Interaction module and representation updates} 
\label{sec:interaction}

A \textit{point-wise interaction module} $\phi$ (\eqref{eq:pi01}-\eqref{eq:pi02}) nonlinearly updates the atom-wise representations through $\mathbf{h}^{t+1} = \phi(\mathbf{h}^{t}, \mathbf{g})$
\begin{align}
    \mathbf{f}^{t}_{lpm} &= \big( \mathrm{MLP}_{1}(\bar{\mathbf{h}}^{t}) \big)_{lp} \odot (\hat{\mathbf{h}}^{t}_{lpm} \cdot \mathbf{W}_{l,p}^{\mathrm{in},t}) \quad \textrm{where} \quad (\bar{\mathbf{h}}^{t}, \hat{\mathbf{h}}^{t} ) = \mathrm{EvNorm}(\mathbf{h}^t) \label{eq:pi01} \\
    \mathbf{q}_{l p m} &= \mathbf{g}_{lpm} + \sum_{l_1, l_2}\sum_{m_1, m_2} \sum_{p_1, p_2} (\mathbf{f}^t_{l_1 p_1 m_1} \odot \mathbf{g}_{l_2 p_2 m_2}) \, C_{l_1 m_1; l_2 m_2}^{lm} \, \delta^{(-1)^{l_1+l_2+l}}_{p_1\cdot p_2\cdot p} \label{eq:cgp} \\
    \mathbf{h}^{t+1}_{lpm} &= \mathbf{h}^{t}_{lpm} + \big(\mathrm{MLP}_{2}(\bar{\mathbf{q}})\big)_{lp} \odot ( \hat{\mathbf{q}}_{lpm} \cdot \mathbf{W}_{l,p}^{\mathrm{out},t} ) \quad \textrm{where} \quad ( \bar{\mathbf{q}}, \hat{\mathbf{q}} ) = \mathrm{EvNorm}(\mathbf{q}) \label{eq:pi02}
\end{align}
which consist of coupling another $\mathrm{O}(3)$-equivariant representation $\mathbf{g}$ with $\mathbf{h}^{t}$ and performing normalizations. In \eqref{eq:pi01}-\eqref{eq:pi02}, $C_{l_1 m_1; l_2 m_2}^{lm}$ are Clebsch-Gordan coefficients of group $\mathrm{SO(3)}$, $\delta_i^j$ is a Kronecker delta function, and $\mathrm{MLP}_{1}$ and $\mathrm{MLP}_{2}$ denote multi-layer perceptrons acting on the feature ($nlp$) dimension. $\mathbf{W}_{l,p}^{\mathrm{in},t}\in \mathbb{R}^{N_{l,p} \times N_{l,p}}$ and $\mathbf{W}_{l,p}^{\mathrm{out},t}\in \mathbb{R}^{N_{l,p} \times N_{l,p}}$ correspond to learnable linear weight matrices specific to the update step $t$ and each $(l, p)$.

For $t<t_1$, the updates are performed by combining $\mathbf{h}^t$ with the aggregated messages $\tilde{\mathbf{m}}^t$ from step $t$:
\begin{equation}
    \mathbf{h}_{A}^{t+1} = \phi\big( \mathbf{h}_{A}^{t}, \rho^{\dagger}(\tilde{\mathbf{m}}_{A}^t) \big)
    \label{eq:accu}
\end{equation}
where $\rho^{\dagger}$ is called a reverse matching layer, defined as:
\begin{align}
    \big( \rho^{\dagger}( \tilde{\mathbf{m}}_A^{t}) \big)_{l(p=+1)m} &=  \mathbf{W}_{l}^{\dagger} \cdot  \sum_{\mu} \mathrm{Scatter}\big( \tilde{\mathbf{m}}_{A,\mu}^{t}, n[\mu, z_A] \big) \\
    \big( \rho^{\dagger}( \tilde{\mathbf{m}}_A^{t}) \big)_{l(p=-1)m} &= \mathbf{0}
\end{align}
the $\mathrm{Scatter}$ operation maps the atomic-orbital dimension in $\tilde{\mathbf{m}}^{t}$ to a feature dimension with fixed length $M_{l}$ using $n[\mu, z_A]$ as the indices, and flattens the outputs into shape $(N_\mathrm{atoms}, M_{l} I J)$. $\mathbf{W}^{\dagger}_{l} \in \mathbb{R}^{N_{l,+1} \times M_{l} I J}$ are learnable linear weight matrices to project the outputs into the shape of $\mathbf{h}^{t}$. 

For $t_1 \leq t < t_2$, the updates are based on local information:
\begin{equation}
    \mathbf{h}_{A}^{t+1} = \phi\big( \mathbf{h}_{A}^{t}, \mathbf{h}_{A}^{t} \big).
\end{equation}


\subsubsection{Pooling layers and training} 
\label{sec:pooling}
A programmed pooling layer reads out the target prediction $\hat{\mathbf{y}}$ after the representations $\mathbf{h}^{t}$ are updated to the last step $\mathbf{h}^{t_f}$. 
Pooling operations employed for obtaining main numerical results are detailed in Appendix~\ref{si_pooling}; hyperparameter, training and loss function details are provided in Appendix~\ref{si_training}. As a concrete example, the dipole moment vector is predicted as $\vec{\mu} = \sum_A(\vec{x}_A \cdot q_A + \vec{\mu}_A)$ where $\vec{x}_A$ is the 3D coordinate of atom $A$, and atomic charges $q_A$ and atomic dipoles $\vec{\mu}_A$ are predicted respectively using scalar ($l=0$) and Cartesian-coordinate vector ($l=1$) components of $\mathbf{h}_A^{t_f}$.

\subsection{QM-informed featurization details and gradient calculations}
\label{si_qchem}
\label{si_xtb}

The QM-informed representation employed in this work is motivated by a pair of our previous works~\cite{orbnet1,qiao2020multi}, but in this study the features are directly evaluated in the atomic orbital basis without the need of heuristic post-processing algorithms to enforce rotational invariance. 

In particular, this work (as well as \cite{orbnet1} and \cite{qiao2020multi}) constructs features based on the GFN-xTB semi-empirical QM method~\cite{gfn1}.
As a member of the class of \textit{mean field} 
quantum chemical methods, GFN-xTB centers around the self-consistent solution of the  Roothaan-Hall equations,
\begin{equation}
    \mathbf{F} \mathbf{C} = \mathbf{S} \mathbf{C} \boldsymbol{\epsilon}.
    \label{eq:scf}
\end{equation}

All boldface symbols are matrices represented in the atomic orbital basis. 
For the particular case of GFN-xTB, 
the atomic orbital basis is similar to STO-6G and comprises a set of hydrogen-like orbitals. 
$\mathbf{C}$ is the molecular orbital coefficients which defines $\Psi_{0}$, and
$\boldsymbol{\epsilon}$ is a diagonal eigenvalue matrix of the molecular orbital energies. 
$\mathbf{S}$ is the overlap matrix and is given by

\begin{equation}
    S_{\mu\nu} = \langle \Phi^{\mu} \lvert  \Phi^{\nu} \rangle
\end{equation}
where $\mu$ and $\nu$ index the atomic orbital basis $\{ \Phi \}$. 
$\mathbf{F}$ is the \textit{Fock matrix} and is given by
\begin{equation}
    \mathbf{F} = \mathbf{H} + \mathbf{G}\left[\mathbf{P}\right].
    \label{eq:fock}
\end{equation}
$\mathbf{H}$ is the one-electron integrals including electron-nuclear attraction and electron kinetic energy. 
$\mathbf{G}$ is the two-electron integrals comprising the electron-electron repulsion.
Approximation of $\mathbf{G}$ is the key task for self-consistent field methods, and GFN-xTB provides an accurate and efficient tight-binding approximation for $\mathbf{G}$.
Finally, $\mathbf{P}$ is the (one-electron-reduced) density matrix, and is given by

\begin{equation}
    P_{\mu\nu} = \sum_{i=1}^{n_\mathrm{elec}/2} C_{\mu i}^* C_{\nu i}.
\end{equation}
$n_\mathrm{elec}$ is the number of electrons, and a closed-shell singlet ground state is assumed for simplicity.
Equations \ref{eq:scf} and \ref{eq:fock} are solved for $\mathbf{P}$.
The electronic energy $E$ is related to the Fock matrix by
\begin{equation}
    \mathbf{F} = \frac{\delta E}{\delta \mathbf{P}}.
\end{equation}
The particular form of the GFN-xTB electronic energy can be found in \cite{gfn1}.

UNiTE is trained to predict the quantum chemistry properties of interest 
based on the inputs $\mathbf{T} = ( \mathbf{F}$, $\mathbf{P}$, $\mathbf{S}$,  $\mathbf{H} )$ with possible extensions (e.g., the energy-weighted density matrices). For the example of learning the DFT electronic energy with the "delta-learning" training strategy:
\begin{equation}
    E_\mathrm{DFT} \approx E_\mathrm{TB} + \mathcal{F}(\mathbf{T}).
    \label{eq:learned_energy}
\end{equation}
Note that $\mathbf{F}$, $\mathbf{P}$, $\mathbf{S}$, and $\mathbf{H}$ all implicitly depend on the atomic coordinates $\mathbf{x}$ and charge/spin state specifications.

In addition to predicting $E$ it is also common to compute its gradient with respect to atomic nuclear coordinates $\mathbf{x}$ to predict the forces used for geometry optimization and molecular dynamics simulations. 
We directly differentiate the energy \eqref{eq:learned_energy} to obtain energy-conserving forces. 
The partial derivatives of the UNiTE energy with respect to $\mathbf{F}$, $\mathbf{P}$, $\mathbf{S}$, and $\mathbf{H}$ is determined through automatic differentiation. The resulting forces are computed through an adjoint approach developed in Appendix D of our previous work~\cite{qiao2020multi}, with the simplification that the SAAO transformation matrix $\mathbf{X}$ is replaced by the identity.

\subsection{Dataset and computational details}

\subsubsection*{Training datasets}
\label{si_dataset}


The molecule datasets used in Section~\ref{sec:qm9}-\ref{sec:density} are all previously published. 
Following Section 2.1 of~\cite{fabrizio2019electron}, the 2291 BFDb-SSI samples for training and testing are selected as the sidechain–sidechain dimers in the original BFDb-SSI dataset that contain $\leq$ 25 atoms and no sulfur element to allow for comparisons among methods.

The Selected Drug-like and biofragment Conformers (SDC21) dataset used for training the OrbNet-Equi/SDC21 model described in Section~\ref{sec:downstream} is collected from several publicly-accessible sources.
First 11,827 neutral SMILES strings were extracted from the ChEMBL database~\cite{chembl27}. 
For each SMILES string, up to four conformers were generated by Entos Breeze, and optimized at the GFN-xTB level.
Non-equilibrium geometries of the conformers were generated using either normal mode sampling~\cite{SmithANI2017} at 300K or \textit{ab initio} molecular dynamics for 200fs at 500K in a ratio of 50\%/50\%, resulting in a total of 178,836 structures.
An additional number 2,549 SMILES string were extracted from ChEMBL, and random protonation states for these were selected using Dimorphite-DL~\cite{DimorphiteDL}, as well as another 2,211 SMILES strings which were augmented by adding randomly selected salts from the list of common salts in the ChEMBL Structure Pipeline~\cite{ChemblPipelineBento2020}.
For these two collections of modified ChEMBL SMILES strings, non-equilibrium geometries were created using the same protocol described earlier, resulting in 21,141 and 27,005 additional structures for the two sets, respectively.
To compensate for the bias towards large drug-like molecules, $\sim$45,000 SMILES strings were enumerated using common bonding patterns, from which a 9,830 conformers were generated from a randomly sampled subset.
Lastly, 
molecules in the BFDb-SSI and JSCH-2005 datasets were added to the training data set~\cite{BFGDBBurns2017,JSCH}.
In total, the data set consists of 237,298 geometries spanning the elements C, O, N, F, S, Cl, Br, I, P, Si, B, Na, K, Li, Ca, and Mg.
For each geometry DFT single point energies were calculated on the dataset at the $\omega$B97X-D3/def2-TZVP level of theory in Entos Qcore version 0.8.17.\cite{wb97xd,Weigend2005,entos} 
Lastly, we additionally filtered the geometries for which DFT calculation failed to converge or broken bonds between the equilibirum and non-equilibrium geometries are detected, resulting in 235,834 geometries used for training the OrbNet-Equi/SDC21 model.

\subsubsection*{Electronic structure computational details}
\label{si_qcore}

The dipole moment labels $\vec{\mu}$ for QM9 dataset used in Section~\ref{sec:qm9} were calculated at the B3LYP level of DFT theory with def2-TZVP AO basis set to match the level of theory used for published QM9 labels, using Entos Qcore version 1.1.0~\cite{Lee1988,def2tzvp,entos}. The electron density labels $\rho(\vec{r})$ for QM9 and BFDb-SSI were computed at the $\omega$B97X-D3/def2-TZVP level of DFT theory using def2-TZVP-JKFIT~\cite{weigend2008hartree} for Coulomb and Exchange fitting, also as the electron charge density expansion basis $\{\chi\}$. The density expansion coefficients $\mathbf{d}$ are calculated as
\begin{equation}
    d_{\gamma} = \sum_{\xi} \sum_{\mu,\nu} \big( (\mathbf{S}^{\mathrm{\rho}})^{-1} \big)_{\gamma \xi} S_{\mu\nu;\xi} P_{\mu\nu}
\end{equation}
where $\mu,\nu$ are AO basis indices, $\xi,\gamma$ are density fitting basis indices. Note that $\gamma$ stands for the combined index $(A, n, l, m)$ in \eqref{eq:des_expand}. $\mathbf{P}$ is the DFT AO density matrix, $\mathbf{S}^{\mathrm{\rho}}$ is the density fitting basis overlap matrix, and $S_{\mu\nu;\xi}$ are 3-index overlap integrals between the AO basis and the density fitting basis $\{\chi\}$. 

\subsubsection*{Benchmarking details and summary statistics}
\label{si_summary}




For the mean $L^1$ electronic density error over the test sets reported in Section~\ref{sec:density}, we use 291 dimers as the test set for the BFDb-SSI dataset, and 10000 molecules as the test set for the QM9 dataset, following literature~\cite{fabrizio2019electron,jorgensen2020deepdft}. $\varepsilon_{\rho}$ for each molecule in the test sets is computed using a 3D cubic grid of voxel spacing $(0.2, 0.2, 0.2)$ Bohr for BFDb-SSI test set and voxel spacing $(1.0, 1.0, 1.0)$ Bohr for the QM9 test set, both with cutoff at $\rho(\vec{r}) = 10^{-5} \ a_0^{-3}$.  We note that two baseline methods used slightly different normalization conventions when computing the dataset-averaged $L^1$ density errors $\varepsilon_{\rho}$, (a) computing $\varepsilon_{\rho}$ for each molecule and normalizing over the number of molecules in the test set~\cite{jorgensen2020deepdft} or (b) normalizing over the total number of electrons in the test set~\cite{fabrizio2019electron}. We found the average $\varepsilon_{\rho}$ computed using normalization (b) is higher than (a) by around 5\% for our results. We follow their individual definitions for average $\varepsilon_{\rho}$ for the quantitative comparisons described in the main text, that is, using scheme (a) for QM9 but scheme (b) for BfDB-SSI.

For downstream task statistics reported in Figure~\ref{fig:6} and Table~\ref{table:downstream}, 
%
%
%
the results on the Hutchison dataset in  Figure~\ref{fig:6}a are calculated as the $R^{2}$ correlation coefficients comparing the conformer energies of multiple conformers from a given model to the energies from DLPNO-CCSD(T).
The median $R^{2}$ in Table~\ref{table:downstream} with respect to both DLPNO-CCSD(T) and $\omega$B97X-D3/def2-TZVP are calculated over the $R^{2}$-values for every molecule, and error bars are estimated by bootstrapping the pool of molecules. The error bars for TorsionNet500 and s66x10 are computed as 95\% confidence intervals. 
%
%
Geometry optimization experiments are performed through relaxing the reference geometries until convergence. Geometry optimization accuracies in Figure~\ref{fig:6}d and Table~\ref{table:downstream} are reported as the symmetry-corrected root mean square deviation (RMSD) of the minimized geometry versus the reference level of theory ($\omega$B97X-D3/def2-TZVP) calculated over molecules in the benchmark set. Additional computational details for this task are provided in Appendix~\ref{si_downstream}.

%
For the GMTKN55 benchmark dataset collection, the reported CCSD(T)/CBS results are used as reference values. 
The WTAD scores for producing Figure~\ref{fig:7} is defined similar to the updated weighted mean absolute deviation (WTMAD-2) in~\cite{gmtkn55}, but computed for each reaction in GMTKN55:
\begin{equation}
    \mathrm{WTAD}_{i,j} = \frac{56.84}{ \frac{1}{N_i} \sum_j |\Delta E|_{i,j}} \cdot |\Delta E|_{i,j}
\end{equation}
for $j$-th reaction in the $i$-th task subset. Note that the subset-wise WTMAD-2 metric in Appendix Table~\ref{table:gmtkn55} is given by
\begin{equation}
    \textrm{WTMAD-2}_i = \frac{1}{N_i} \sum_{j} \mathrm{WTAD}_{i,j}.
\end{equation}
and the overall WTMAD-2 is reproduced by
\begin{equation}
    \textrm{WTMAD-2} = \frac{1}{\sum_{i}^{55} N_i} \sum_{i,j} \mathrm{WTAD}_{i,j}.
\end{equation}
%
%

\section*{Acknowledgements}
Z.Q. acknowledges graduate research funding from Caltech and partial support from the Amazon–Caltech AI4Science fellowship. T.F.M. and A.A. acknowledge partial support from the Caltech DeLogi fund, and A.A. acknowledges support from a Caltech Bren professorship. Z.Q. acknowledges Bo Li, Vignesh Bhethanabotla, Dani Kiyasseh, Hongkai Zheng, Sahin Lale, and Rafal Kocielnik for proofreading and helpful comments on the manuscript.

\clearpage

\begin{figure}[h!]
    \centering
    \includegraphics[width=0.7\linewidth]{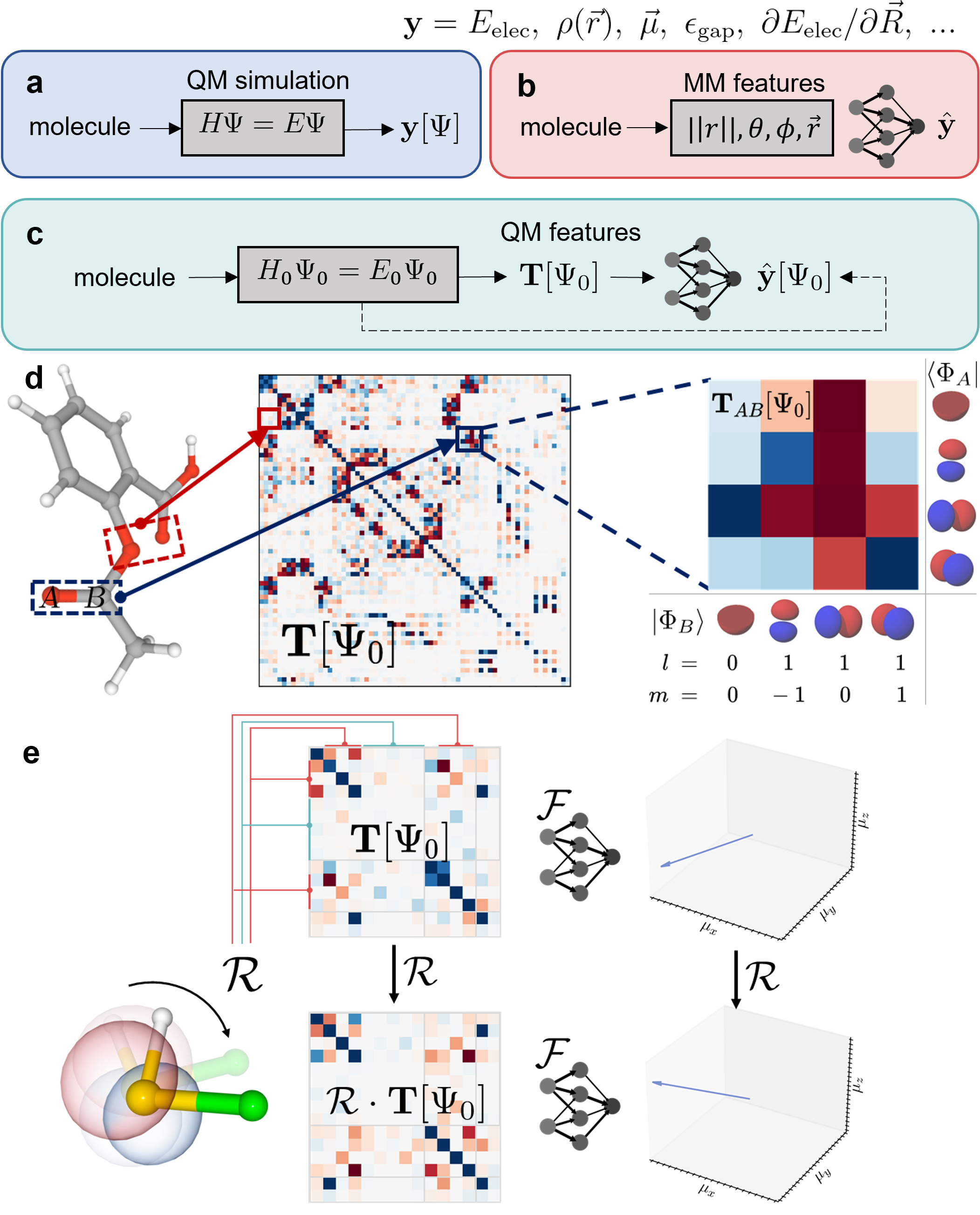}
    \caption{QM-informed machine learning for modelling molecular properties. (a) Conventional \textit{ab initio} quantum chemistry methods predict molecular properties based on electronic structure theory through computing molecular wavefunctions and interaction terms, with general applicability but at high computational cost. (b) Atomistic machine learning approaches use geometric descriptors such as interatomic distances, angles, and directions to bypass the procedure of solving the electronic structure problem, but often requires vast amounts of data to generalize toward new chemical species. (c) In our approach, features are extracted from a highly coarse-grained QM simulation to capture essential physical interactions. An equivariant neural network efficiently learns the mapping, yielding improved  transferability at an evaluation speed that is competitive to Atomistic ML methods.  (d) Characteristics of the atomic orbital features considered in OrbNet-Equi. Every pair of atoms $(A, B)$ is mapped to a block in the feature matrix, with the row dimension of the block matching the atomic orbitals of the source atom $A$ and the column dimension matching the atomic orbitals of the destination atom $B$. (e) OrbNet-Equi is equivariant with respect to isometric basis transformations on the atomic orbitals (Equations~\ref{eq:basis_rot}-\ref{eq:rotsym}), yielding consistent predictions (illustrated as the dipole moment vector of a HSF molecule) at different viewpoints.}
    \label{fig:features}
\end{figure}

\begin{figure}[h!]
    \centering
    \includegraphics[width=\textwidth]{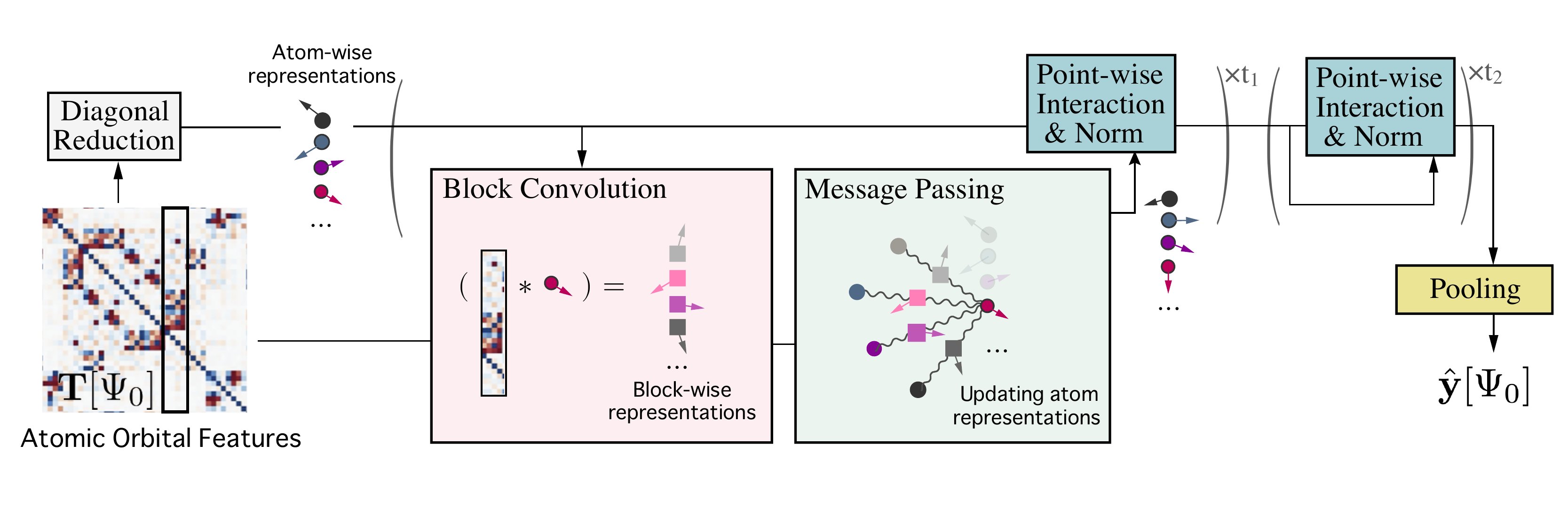}
    \caption{Schematic illustration of the OrbNet-Equi method. The input atomic orbital features $\mathbf{T}[{\Psi_{0}}]$ are obtained from a low-fidelity QM simulation. A neural network termed UNiTE first initializes atom-wise representations through the diagonal reduction module, and then updates the representations through stacks of block convolution, message passing, and point-wise interaction modules. A programmed pooling layer reads out high-fidelity  property predictions $\hat{\mathbf{y}}$ based on the final representations. Neural network architecture details are provided in Methods~\ref{sec:unite}.}
    \label{fig:architect}
\end{figure}

\begin{figure}[h!]
  \centering
  \includegraphics[width=0.8\linewidth]{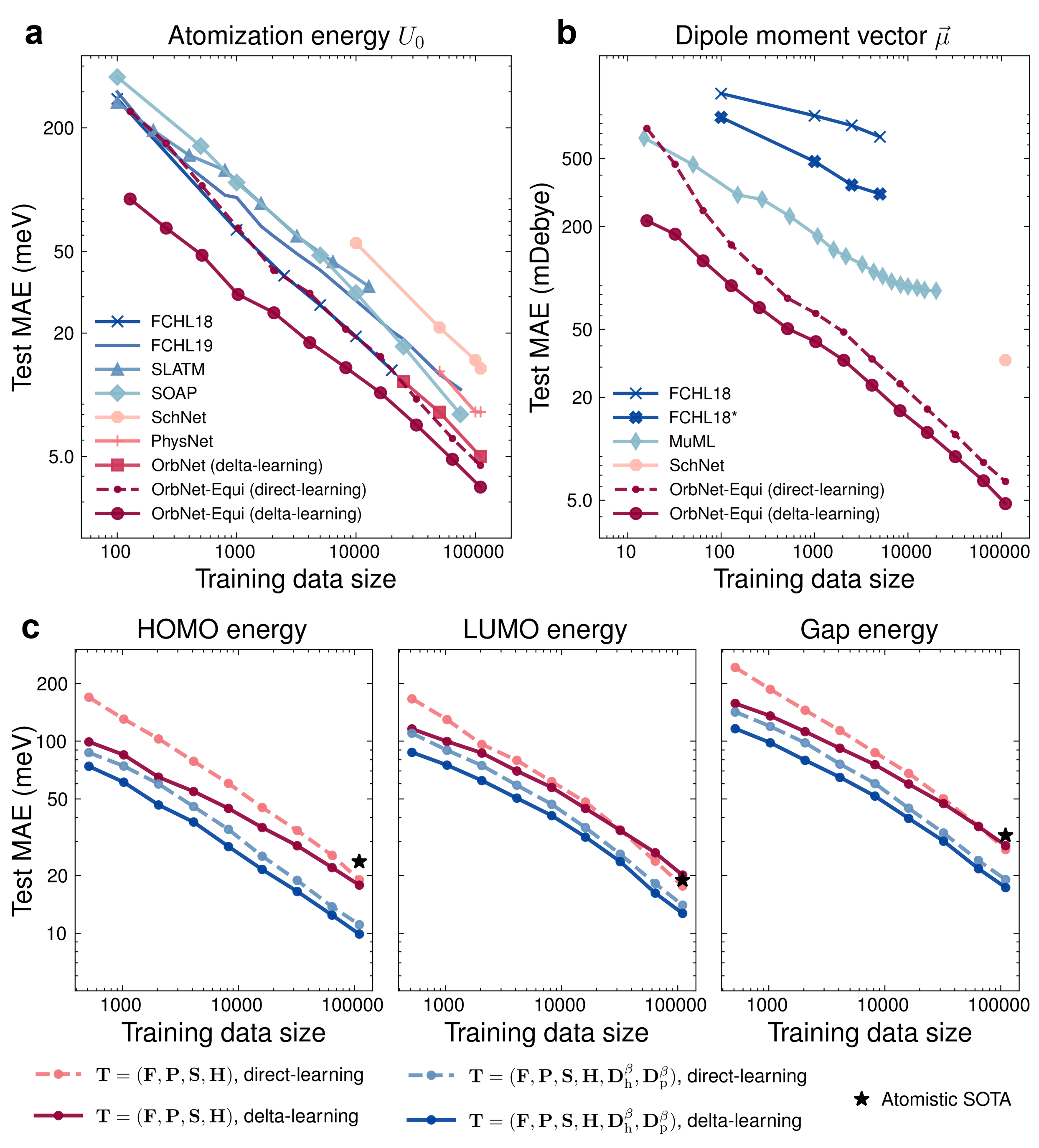}
\caption{Model performance on the QM9 dataset. (a-b) Test mean absolute error (MAE) of OrbNet-Equi is shown as functions of the number of training samples, along with previously reported results from task-specific ML methods (FCHL18\cite{Faber2018}, FCHL19\cite{Christensen2019}, SLATM\cite{huang2017efficient}, SOAP\cite{bartok2017machine}, FCHL18*\cite{christensen2019operators}, MuML\cite{veit2020predicting}) and deep-learning-based methods (SchNet\cite{schutt2017schnet}, PhysNet\cite{unke2019physnet}, OrbNet~\cite{orbnet1}) for targets (a) electronic energy $U_0$ and (b) molecular dipole moment vector $\vec{\mu}$ on the QM9 dataset. Results for OrbNet-Equi models trained with direct-learning and delta-learning are shown in dashed and solid lines, respectively. 
(c) Incorporating energy-weighted density matrices to improve data efficiency on learning frontier orbital properties. The HOMO, LUMO, and HOMO-LUMO gap energy test MAEs of OrbNet-Equi are shown as functions of the number of training samples. For models with the default feature set (red curves), the reduction in test MAE for delta-learning over direct-learning models gradually diminishes as the training data size grows. The LUMO and gap energy MAE curves exhibit a crossover around 32k-64k training samples, thereafter direct-learning models outperform delta-learning models. In contrast, when the energy-weighted density matrix features are supplied (blue curves), the test MAE curves between direct-learning and delta-learning models remain gapped when the training data size is varied. The black stars indicate the lowest test MAEs achieved by Atomistic ML methods (SphereNet~\cite{spherenet}) trained with 110k samples.
} 
\label{fig:qm9}
\end{figure}

\begin{figure}[h!]
    \centering
    \includegraphics[width=0.8\linewidth]{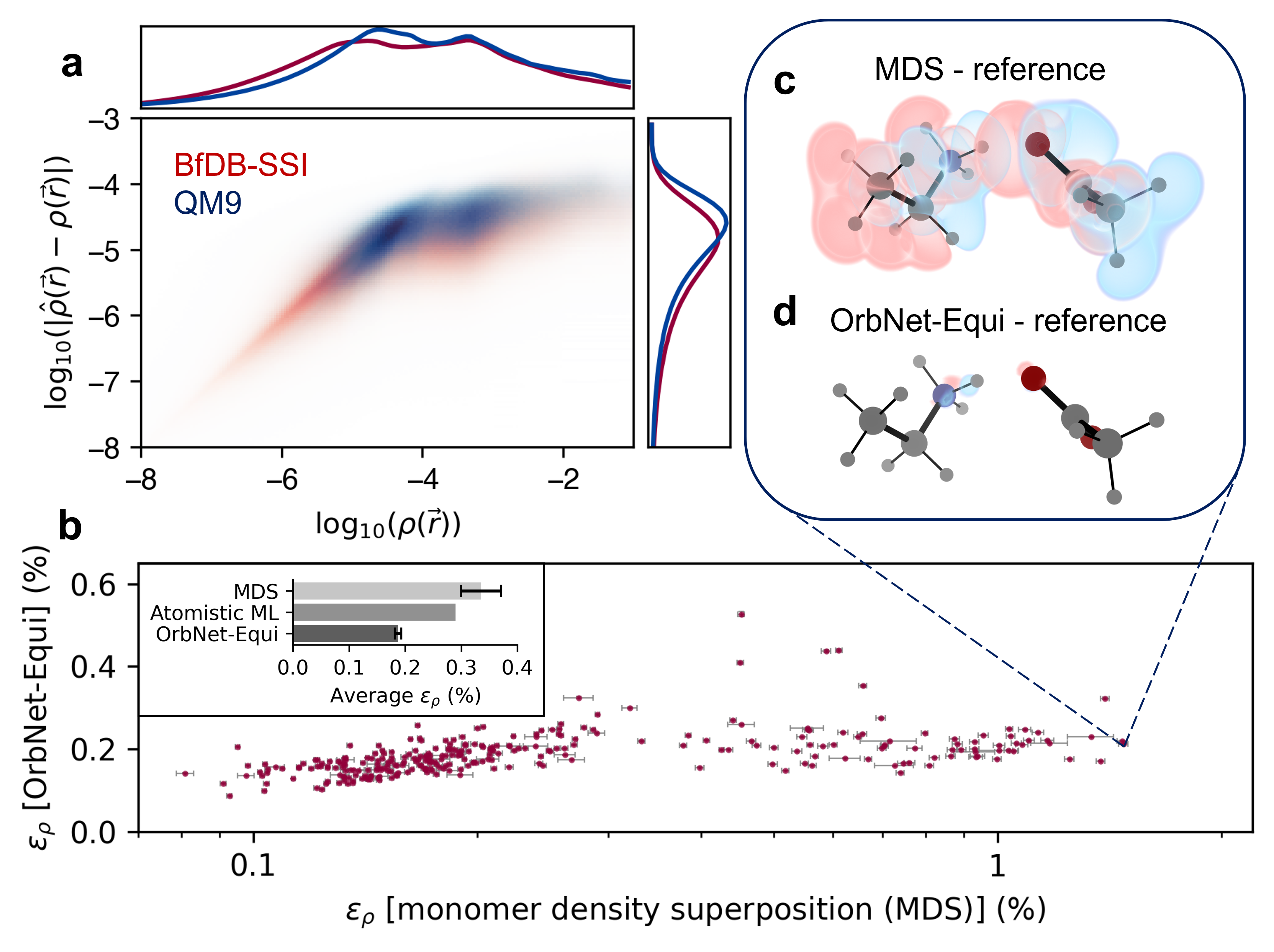}
    \caption{Learning electron charge densities for organic and biological motif systems. (a) 2D heatmaps of the log-scale reference density $\rho(\vec{r})$ and the log-scale OrbNet-Equi density prediction error $\lvert\hat{\rho}(\vec{r})-\rho(\vec{r})\rvert$ (both in $a_0^{-3}$). The heatmaps are calculated by sampling real-space query points $\vec{r}\in\mathbb{R}^3$ for all molecules in the (red) BfDB-SSI test set and (blue) QM9 test set. The nearly-linear relationship for $\log_{10}(\rho(\vec{r}))<-4$ low-density regions reveals that OrbNet-Equi-predicted densities possess a physical long-range decay behavior. Distributions of $\log_{10}(\rho(\vec{r}))$ and $\log_{10}(\lvert\hat{\rho}(\vec{r})-\rho(\vec{r})\rvert)$ are plotted within the marginal charts. (b) The $L^1$ density errors $\varepsilon_{\rho}$ of OrbNet-Equi are plotted against the $\varepsilon_{\rho}$ of densities obtained through monomer density superposition (MDS), across the BfDB-SSI test set. Error bars mark the 99\% confidence intervals of $\varepsilon_{\rho}$ for individual samples. The inset figure shows the average $\varepsilon_{\rho}$ for MDS, an Atomistic ML method~\cite{fabrizio2019electron}, and OrbNet-Equi predictions on the BfDB-SSI test set. OrbNet-Equi yields the lowest average prediction error and consistently produces accurate electron densities for cases where inter-molecular charge transfer is substantial. (c-d)  Visualization of density deviation maps 
    for (c) MDS and (d) OrbNet-Equi-predicted densities on the $\textrm{Glu}^{-}/\textrm{Lys}^{+}$ system (SSI-139GLU-144LYS-1), a challenging example from the BfDB-SSI test set. Red isosurfaces correspond to $\Delta\rho = -0.001\ a_0^{-3}$ and blue isosurfaces correspond to $\Delta\rho = +0.001\ a_0^{-3}$, where $\Delta\rho$ is the model density subtracted by the DFT reference density.}
    \label{fig:density}
\end{figure}

\begin{figure}[h!]
    \centering
    \includegraphics[width=0.7\linewidth]{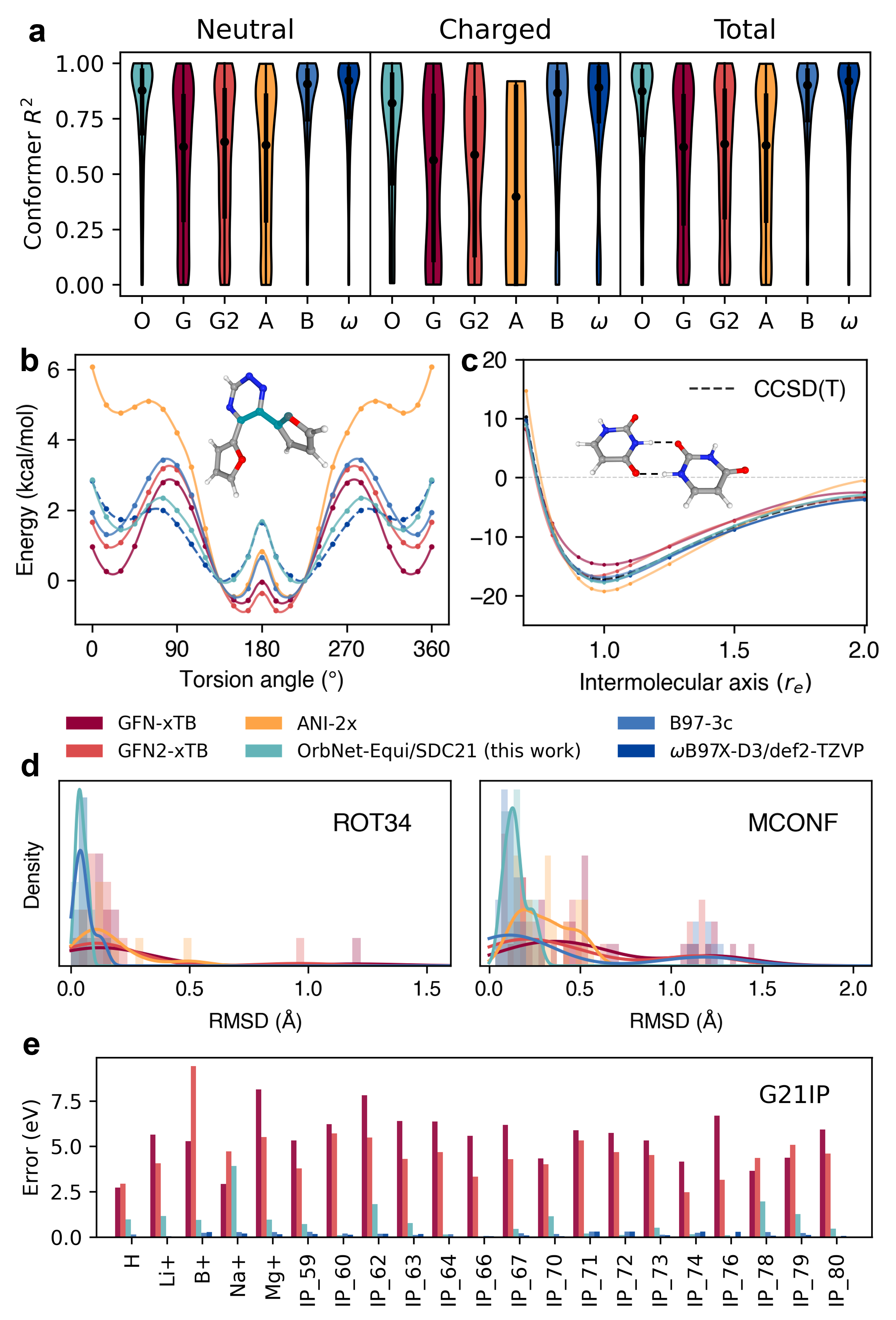}
    \caption{OrbNet-Equi/SDC21 infers diverse downstream properties at an efficiency of semi-empirical tight-binding calculations. (a) Conformer energy ranking on the Hutchison dataset of drug-like molecules. The horizontal axis is labelled with acronyms indicating each method (O: OrbNet-Equi/SDC21 (this work); G: GFN-xTB; G2: GFN2-xTB; A: ANI-2x; B: B97-3c; $\omega$: $\omega$B97X-D3/def2-TZVP). The y-axis corresponds to the molecule-wise $R^2$ between predictions and the reference (DLPNO-CCSD(T)) conformer energies. Violin plots display the distribution of $R^2$ scores for each method over the (left) neutral, (middle) charged, and (right) all molecules from the Hutchison dataset. Medians and first/third quantiles are shown as black dots and vertical bars. (b) A torsion profiles example from the TorsionNet500 benchmark. All predicted torsion scans surfaces are aligned to the true global minima of the highest level of theory ($\omega$B97X-D3/def2-TZVP) results, with spline interpolations.  (c) A uracil-uracil base pair example for non-covalent interactions. The dimer binding energy curves are shown as functions of the intermolecular axis ($r_e$) where $r_e=1.0$ corresponds to the distance of optimal binding energy. (d) Geometry optimization results on the (left) ROT34 and (right) MCONF datasets. Histograms and kernel density estimations of the symmetry-corrected RMSD scores (Methods~\ref{si_summary}) with respect to the reference DFT geometries are shown for each test dataset. (e) Evidence of zero-shot model generalization on radical systems. OrbNet-Equi/SDC21 yields prediction errors drastically lower than semi-empirical QM methods for adiabatic ionization potential on the G21IP dataset, achieving accuracy comparable to DFT on 7 out of 21 test cases.}
    \label{fig:6}
\end{figure}

\begin{figure}[h!]
    \centering
    \includegraphics[width=\linewidth]{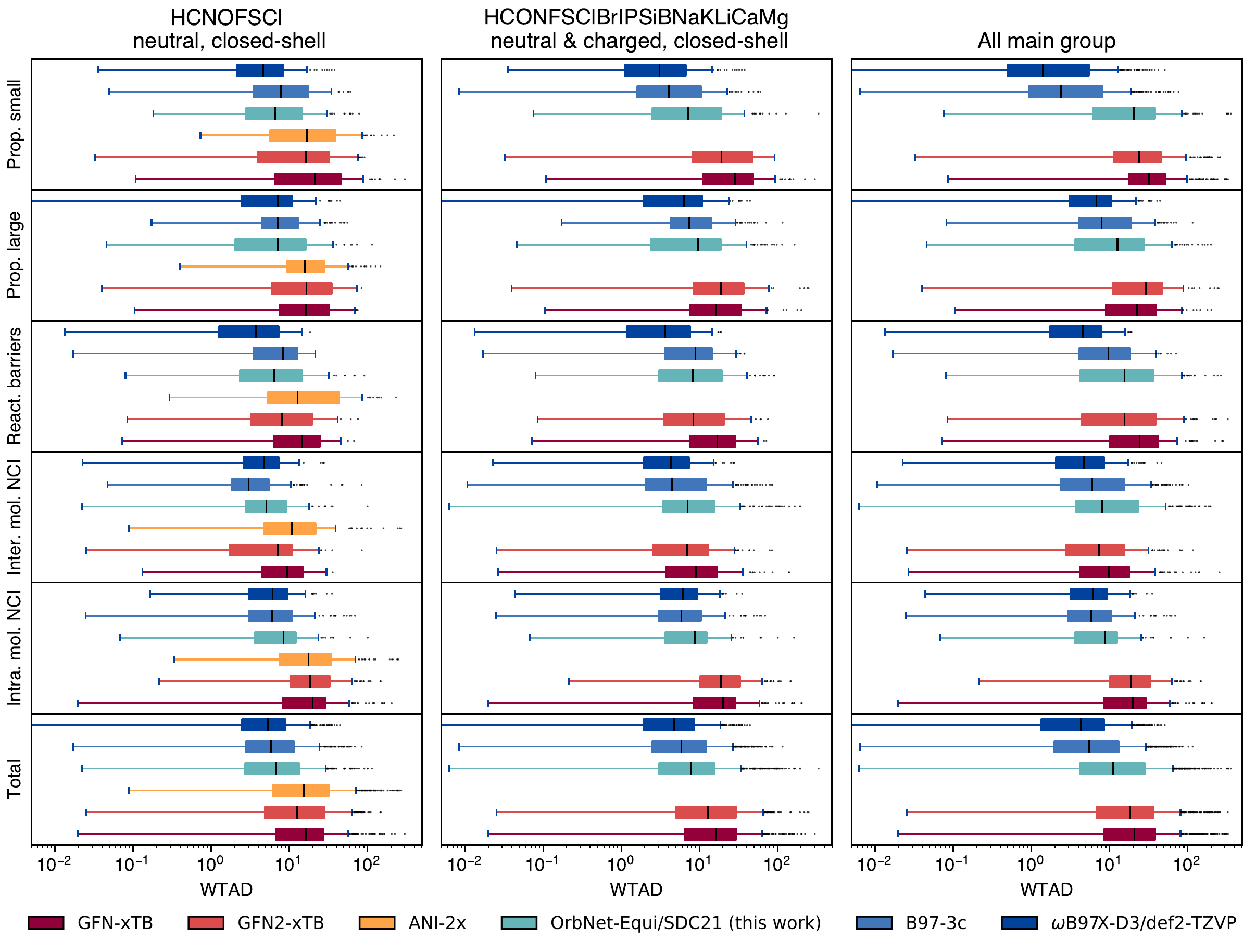}
    \caption{Assessing model performance on tasks from the GMTKN55 challenge. Box plots depict the distributions of task-difficulty-weighted absolution deviations (WTAD, see Methods~\ref{si_summary}) filtered by chemical elements and electronic states (a) supported by the ANI-2x model; (b) appeared in the dataset used for training OrbNet-Equi/SDC21; (c) all reactions. Statistics are categorized by each class of tasks in the GMTKN55 benchmark, as shown in y-axis labels. Prop. small: Basic properties and reaction energies for small systems; Prop. large: Reaction energies for large systems and isomerisation reactions; React. barriers: Reaction barrier heights; Inter. mol. NCI: Intermolecular noncovalent interactions; Intra. mol. NCI: Intramolecular noncovalent interactions; Total: total statistics of all tasks.}
    \label{fig:7}
\end{figure}

\clearpage


\newpage
\appendix

\setcounter{section}{0}
\setcounter{table}{0}
\setcounter{figure}{0}
\setcounter{equation}{0}
\setcounter{definition}{0}
\setcounter{theorem}{0}
\setcounter{corollary}{0}
\renewcommand{\thesection}{S\arabic{section}}
\renewcommand{\thetable}{S\arabic{table}}
\renewcommand{\thefigure}{S\arabic{figure}}
\renewcommand{\theequation}{S\arabic{equation}}
\renewcommand{\thedefinition}{S\arabic{definition}}
\renewcommand{\thetheorem}{S\arabic{theorem}}
\renewcommand{\thecorollary}{S\arabic{corollary}}
\allowdisplaybreaks

\section{Additional neural network details}
\label{si_implementation}

\subsection{Efficient GPU evaluation of spherical harmonics and Clebsch-Gordan coefficients}
\label{si_o3}

All $\mathrm{O(3)}$-representation related operations are implemented through element-wise operations on arrays and gather-scatter operations, without the need of recursive computations that can be difficult to parallelize on GPUs at runtime.  
The real spherical harmonics (RSHs) are computed based on Equations 6.4.47-6.4.50 of \cite{helgaker2014molecular}, which reads:
\begin{align}
    Y_{lm}(\vec{r}) &= N^\mathrm{S}_{lm} \sum_{t=0}^{[(l-\lvert m \rvert /2)]} \sum_{u=0}^{t} \sum_{v=v_m}^{[\lvert m\rvert/2-v_m]+v_m} C_{tuv}^{lm} (\frac{x}{\lVert r \rVert})^{2t+\lvert m \rvert-2(u+v)} (\frac{y}{\lVert r \rVert})^{2(u+v)} (\frac{z}{\lVert r \rVert})^{l-2t-\lvert m \rvert} \\
    C^{lm}_{tuv} &= (-1)^{t+v-v_m} (\frac{1}{4})^t \binom{l}{t} \binom{l-t}{\lvert m \rvert + t} \binom{t}{u} \binom{\lvert m \rvert}{2v} \\
    N^\mathrm{S}_{lm} &= \frac{1}{2^{\lvert m \rvert}l!} \sqrt{\frac{2(l+\lvert m \rvert)!(l-\lvert m \rvert)!}{2^{\delta_{0m}}}} \\
v_m &=
\begin{cases}
0 & \text{if  } m \geq 0 \\
\frac{1}{2} & \text{if  } m<0 \\
\end{cases}
\end{align}
where $[\cdot]$ is the floor function. The above scheme only requires computing element-wise powers of 3D coordinates and a linear combination with pre-tabulated coefficients. 
The Clebsch-Gordan (CG) coefficients are first tabulated using their explicit expressions for complex spherical harmonics (CSHs) based on Equation 3.8.49 of Ref.~\citenum{sakurai1995modern}, and are then converted to RSH CG coefficients with the transformation matrix between RSHs and CSHs~\cite{blanco1997evaluation}.






\subsection{Multiple input channels} UNiTE is naturally extended to inputs that possess extra feature dimensions, as in the case of AO features $\mathbf{T}$ described in Section \ref{si_xtb} the extra dimension equals the cardinality of selected QM operators. Those stacked features is processed by a learnable linear layer $\mathbf{W}^{\mathrm{in}}$ 
resulting in a fixed-size channel dimension. Each channel is then shared among a subset of convolution channels (indexed by $i$), instead of using one convolution kernel for all channels $i$. 
For the numerical experiments of this work, $\mathbf{T}$ are mixed into $I$ input channels by $\mathbf{W}^{\mathrm{in}}$ and we assign a convolution channel to each input channel.

\subsection{Restricted summands in Clebsch-Gordan coupling} For computational efficiency, in the Clebsch-gordan coupling \eqref{eq:cgp} (i.e., \eqref{eq:si_cgp}) of a point-wise interaction block, we further restrict the angular momentum indices $(l_1, l_2)$ within the range $\{ (l_1, l_2); l_1+l_2 \leq l_\mathrm{max}, l_1 \leq n, l_2 \leq n \}$ where $l_\mathrm{max}$ is the maximum angular momentum considered in the implementation.


\subsection{Incorporating geometric information} Because the point cloud of atomic coordinates $\mathbf{x}$ is available in addition to the atomic-orbital-based inputs $\mathbf{T}$, we incorporated such geometric information through the following modified message-passing scheme to extend \eqref{eq:mp}:
\begin{equation}
    \label{eq:qc_accu}
    (\tilde{\mathbf{m}}_A^t)_{lm} = \sum_{B\neq A} \bigoplus\limits_{i,j} \Big( (\mathbf{m}_{AB}^{t,i})_{lm} + Y_{lm}(\hat{x}_{AB}) \big( \mathbf{W}^{l,t}_{i} \lvert \lvert \mathbf{m}_{AB}^{t,i}\lvert \lvert  \big) \Big) \cdot {\alpha}_{AB}^{t,j} 
\end{equation}
where $Y_{lm}$ denotes a spherical harmonics of degree $l$ and order $m$, $\hat{x}_{AB}\defeq \frac{\vec{x}_{AB}}{\lvert \lvert \vec{x}_{AB}\lvert \lvert }$ denotes the direction vector between atomic centers A and B, and 
$\mathbf{W}^{l,t}_{i}$ are learnable linear functions.

\subsection{Pooling layers}
\label{si_pooling}

We define schemes for learning different classes of chemical properties with OrbNet-Equi 
without modifying the base UNiTE model architecture. We use $A$ to denote an atom index, $\lvert A\lvert $ to denote the total number of atoms in the molecule, $z_A \in \mathbb{N}^{+}$ to denote the atomic number of atom $A$, and $\vec{x}_A \in \mathbb{R}^{3}$ to denote the atomic coordinate of atom $A$.

\subsubsection{Energetic properties} 
\label{sec:pooling_energy}
A representative target in this family is the molecular electronic energy $E(\mathbf{x})$ (i.e., $U_0$ in the convention of QM9), which is rotation-invariant and proportional to the system size (i.e., extensive). The pooling operation is defined as:
\begin{equation}
    y_\theta = \sum_A \mathbf{W}_{\mathrm{o}} \cdot \lvert \lvert \mathbf{h}^{t_f}_A\lvert \lvert  + b^{\mathrm{o}}_{z_A} 
\end{equation}
which is a direct summation over atom-wise contributions. $\mathbf{W}_{\mathrm{o}}$ is a learnable linear layer and $b^{\mathrm{o}}_{z_A}$ are learnable biases for each atomic number $z$. To account for nuclei contributions to molecular energies, 
we initialize $b^{\mathrm{o}}_{z}$ from a linear regression on the training labels with respect to $\{z_A\}$ to speed up training on those tasks. This scheme is employed for learning $U_0$, $U$, $H$, $G$, ZPVE and $c_v$ on QM9, the energies part in MD17 and for the OrbNet-Equi/SDC21 model. 


\subsubsection{Dipole moment $\vec{\mu}$}
\label{sec:pooling_dipole}
The dipole moment $\vec{\mu}$ can be thought as a vector in $\mathbb{R}^{3}$. It is modelled as a combination of atomic charges $q_A$ and atomic dipoles $\vec{\mu}_A$, and the pooling operation is defined as
\begin{align}
    \vec{\mu}_{\theta} &= \sum_A(\vec{R}_A \cdot q_A + \vec{\mu}_A) \\ 
    q_A &= q'_A - \Delta q \quad \textrm{where} \quad \Delta q \defeq \frac{\sum_A q'_A}{\lvert A\lvert } \label{eq:compensation} \\
    q'_A &\defeq \mathbf{W}_{\mathrm{o,0}} \cdot (\mathbf{h}^{t_f}_A)_{l=0,p=1} + b^{\mathrm{o}}_{z_A}  \\
    (\vec{\mu}_A)_{m} &\defeq \mathbf{W}_{\mathrm{o,1}} \cdot (\mathbf{h}^{t_f}_A)_{l=1,p=1,m} \quad \textrm{where} \quad m\in \{ x,y,z\}
\end{align}
where $\mathbf{W}_{\mathrm{o,0}}$ and $\mathbf{W}_{\mathrm{o,1}}$ are learnable linear layers. Equation \ref{eq:compensation} ensures the translation invariance of the prediction through charge neutrality. 


Note that OrbNet-Equi is trained by directly minimizing a loss function $\mathcal{L}(\vec{\mu}, \vec{\mu}_{\theta})$ between the ground truth and the predicted molecular dipole moment vectors. For the published QM9 reference labels~\cite{qm9} only the dipole norm $\mu \defeq \lvert \lvert \vec{\mu}\lvert \lvert $ is available; we use the same pooling scheme to readout $\vec{\mu}_{\theta}$ but train on $\mathcal{L}(\mu, \lvert \lvert \vec{\mu}_{\theta}\lvert \lvert )$ instead to allow for comparing to other methods in Table~\ref{table:qm9}.

\subsubsection{Polarizability $\alpha$} For isotropic polarizability $\alpha$, the pooling operation is defined as
\begin{align}
    \alpha_{\theta} &= \sum_A(\alpha_A + \vec{R}_A \cdot \vec{p}_A) \\ 
    \alpha_A &\defeq \mathbf{W}_{\mathrm{o,0}} \cdot (\mathbf{h}^{t_f}_A)_{l=0,p=1} + b^{\mathrm{o}}_{z_A}  \\
    \vec{p}_A &= \vec{p'}_A - \Delta \vec{p} \quad \textrm{where} \quad \Delta \vec{p} \defeq \frac{\sum_A \vec{p'}_A}{\lvert A\lvert } \\
    (\vec{p'}_A)_{m} &\defeq \mathbf{W}_{\mathrm{o,1}} \cdot (\mathbf{h}^{t_f}_A)_{l=1,p=1,m} \quad \textrm{where} \quad m\in \{ x,y,z\}
\end{align}

\subsubsection{Molecular orbital properties} 
\label{si_mo_feats}

For frontier molecular orbital energies, a \textit{global-attention} based pooling is employed to produce intensive predictions:
\begin{align}
    a_A &= \mathrm{Softmax}(\mathbf{W}_{\mathrm{a}} \cdot \lvert \lvert \mathbf{h}^{t_f}_A\lvert \lvert ) \defeq \frac{\mathbf{W}_{\mathrm{a}} \cdot \lvert \lvert \mathbf{h}^{t_f}_A\lvert \lvert }{\sum_A \mathbf{W}_{\mathrm{a}} \cdot \lvert \lvert \mathbf{h}^{t_f}_A\lvert \lvert } \\
    y_\theta &= \sum_A a_A \cdot (\mathbf{W}_{\mathrm{o}} \cdot \lvert \lvert \mathbf{h}^{t_f}_A\lvert \lvert  + b^{\mathrm{o}}_{z_A} )
\end{align}
where $\mathbf{W}_{\mathrm{a}}$ and $\mathbf{W}_{\mathrm{o}}$ are learnable linear layers and $b^{\mathrm{o}}_{z_A}$ are learnable biases for each atomic number $z$. Similar to energy tasks, we initialize $b^{\mathrm{o}}_{z}$ from a linear fitting on the targets to precondition training.


We take the difference between the predicted HOMO energies ($\epsilon_{\mathrm{HOMO}}$) and LUMO energies ($\epsilon_{\mathrm{LUMO}}$) as the HOMO-LUMO Gap ($\Delta\epsilon$) predictions. 



\subsubsection{Electronic spatial extent $\langle R^{2} \rangle$} The pooling scheme for $\langle R^{2} \rangle$ is defined as:
\begin{align}
    \langle R^{2} \rangle_{\theta} &= \sum_A(\lvert \lvert \vec{R}_A - \vec{R}_0\lvert \lvert ^{2} \cdot q_A + s_A) \\ 
    \vec{R}_0 &\defeq \frac{\sum_A (\vec{R}_A \cdot q_A + \vec{\mu}_A) }{\sum_A q_A} \\
    q_A &\defeq \mathbf{W}_{\mathrm{o,0}} \cdot (\mathbf{h}^{t_f}_A)_{l=0,p=1} + b^{\mathrm{o}}_{z_A}  \\
    (\vec{\mu}_A)_{m} &\defeq \mathbf{W}_{\mathrm{o,1}} \cdot (\mathbf{h}^{t_f}_A)_{l=1,p=1,m} \quad \textrm{where} \quad m\in \{ x,y,z\} \\
    s_A &\defeq \mathbf{W}_{\mathrm{o,2}} \cdot (\mathbf{h}^{t_f}_A)_{l=0,p=1}
\end{align}
where $\mathbf{W}_{\mathrm{o,0}}$, $\mathbf{W}_{\mathrm{o,1}}$ and $\mathbf{W}_{\mathrm{o,2}}$ are learnable linear layers.

\subsubsection{Electron densities $\rho(\vec{r})$} 
\label{sec:dens_pooling}

Both the ground truth and predicted electron densities $\rho(\vec{r})$ are represented as a superposition of atom-centered density fitting basis $\{\chi\}$,
\begin{equation}
    \rho(\vec{r}) = \sum_A^{N_\mathrm{atom}} \sum_{l}^{l_\mathrm{max}(z_A)} \sum_{m=-l}^{l} \sum_{n}^{n_\mathrm{max}(z_A, l)} d_{A}^{n l m} \ \chi_A^{n l m}(\vec{r}) 
\end{equation}
similar to the approach employed in~\cite{fabrizio2019electron}; here we use the def2-TZVP-JKFIT density fitting basis for $\{\chi\}$. Computational details regarding 
obtaining the reference density coefficients $d_{A}^{n l m}$ are given in Section~\ref{si_qcore}, and the training loss function is defined in SI~\ref{si_training_density}. The pooling operation to predict $\rho(\vec{r})$ from UNiTE is defined as
\begin{align}
\label{eq:dens_pooling}
    \hat{d}_{A}^{n l m} \defeq \big( \mathbf{W}^{\mathrm{d}}_{z_A, l} \cdot (\mathbf{h}^{t_f}_A)_{l,p=1,m} \big)_n
\end{align}
where $\mathbf{W}^{\mathrm{d}}_{z, l}$ are learnable weight matrices specific to each atomic number $z$ and angular momentum index $l$, and $z_A$ denotes the atomic number of atom $A$. This atom-centered expansion scheme compactly parameterizes the model-predicted density $\hat{\rho}(\vec{r})$. We stress that all UNiTE neural network parameters except for this density pooling layer \eqref{eq:dens_pooling} are independent of the atomic numbers $z$.

\subsection{Time complexity}
The asymptotic time complexity of UNiTE model inference is $\mathcal{O}(B N I )$, where $B$ is the number of 
non-zero elements in $\mathbf{T}$, and
$I$ denotes the number of convolution channels in a convolution block \eqref{eq:conv}.
This implies UNiTE scales as $\mathcal{O}(N (nd)^{N})$ if the input is dense, but can achieve a lower time complexity for sparse inputs, e.g., when long-range cutoffs are applied.
We note that in each convolution block \eqref{eq:conv} the summand $T_{\vec{u},\vec{v}} \cdot \prod_{j=2}^{N}    \big( \rho_{u_j}(\mathbf{h}^{t}_{u_j}) \big)^i_{v_j} \neq 0$ only if the tensor coefficient $T_{\vec{u}, \vec{v}} \neq 0$; therefore \eqref{eq:conv} can be exactly evaluated using $((N-1) B I)$ arithmetic operations. In each message passing block \eqref{eq:mp} the number of arithmetic operations scales as $\mathcal{O}(B’I)$ where $B'$ is the number of indices $\vec{u}$ such that $\mathbf{m}^{t}_{\vec{u}} \neq 0$, and $B’ \leq B$. The embedding block $\phi$ and the point-wise interaction block $\psi$ has $\mathcal{O}(d)$ time complexities since they act on each point independently and do not contribute to the asymptotic time complexity.


\section{Theoretical results}
\label{si_theory}


We formally introduce the problem of interest, restate the definitions of the building blocks of UNiTE (Methods~\ref{sec:unite}) using more formal notations, and prove the theoretical results claimed in this work. We first generalize the input data domain to a generic class of tensors beyond quantum chemistry quantities; for brevity we call such inputs \textit{N-body tensors}.  


\subsection{$N$-body tensors (informal)}

We are interested in 
a class of 
tensors $\mathbf{T}$, for which each  
sub-tensor $\mathbf{T}(u_1, u_2, \cdots, u_N)$
describes relation among a collection of $N$ geometric objects 
defined in an $n$-dimensional physical space. For simplicity, we will first introduce the tensors of interest using a special case based on point clouds embedded in the $n$-dimensional Euclidean space, associating a (possibly different) set of orthogonal basis with each point's neighbourhood. 
In this setting, our main focus is the change of the order-N tensor's coefficients when applying 
$n$-dimensional rotations and reflections 
to the local reference frames. 


\begin{definition}[$N$-body tensor] \label{def:nbti}
Let $\{ \mathbf{x}_1, \mathbf{x}_2, \cdots, \mathbf{x}_d \}$ be $d$ points in $\mathbb{R}^{n}$ for each $u \in \{1, 2, \cdots, d\}$. For each point index $u$, we define an orthonormal basis (local reference frame) $\{\mathbf{e}_{u;v_u} \}$ centered at $\mathbf{x}_u$\footnotemark, and denote the space spanned by the basis as $V_u \defeq \mathrm{span}(\{\mathbf{e}_{u;v_u} \}) \subseteq \mathbb{R}^{n}$.  We consider a tensor $\hat{\mathbf{T}}$ defined via $N$-th direct products of the `concatenated' basis $\{\mathbf{e}_{u;v_u}; (u, v_u) \}$:
\footnotetext{We additionally allow for $\mathbf{0} \in \{\mathbf{e}_{u;v_u} \}$ to represent features in $\mathbf{T}$ that transform as scalars.}
\begin{equation}
    \hat{\mathbf{T}} \defeq \sum_{\vec{u}, \vec{v}} T\big( (u_1;v_1), (u_2;v_2), \cdots, (u_N;v_N) \big) \ \mathbf{e}_{u_1; v_1} \otimes \mathbf{e}_{u_2; v_2} \otimes \cdots \otimes \mathbf{e}_{u_N;{v_N}}
\end{equation}
$\hat{\mathbf{T}}$ is a tensor of order-$N$ and is an element of $(\bigoplus_{u=1}^{d} V_u)^{\otimes N}$. We call its coefficients $\mathbf{T}$ an \textit{N-body tensor} if $\mathbf{T}$ is invariant to global translations ($\forall \, \mathbf{x}_0 \in \mathbb{R}^{n}, \mathbf{T}[\mathbf{x}] = \mathbf{T}[\mathbf{x}+\mathbf{x}_0]$, and is symmetric:
\begin{equation}
    T\big( (u_1;v_1), (u_2;v_2), \cdots, (u_N;v_N) \big) = T\big( (u_{\sigma_1};v_{\sigma_1}), (u_{\sigma_2};v_{\sigma_2}), \cdots, (u_{\sigma_N};v_{\sigma_N}) \big)
\end{equation}
where $\sigma$ denotes arbitrary permutation on its dimensions $\{1, 2, \cdots, N\}$. Note that each sub-tensor, $\mathbf{T}_{\vec{u}}$, does not have to be symmetric.  The shorthand notation $\vec{u} := (u_1, u_2, \cdots, u_N)$ indicates a subset of $N$ points in $\{ \mathbf{x}_1, \mathbf{x}_2, \cdots, \mathbf{x}_d \}$ which then identifies a sub-tensor\footnotemark $\mathbf{T}_{\vec{u}} := \mathbf{T}(u_1, u_2, \cdots, u_N)$ in the $N$-body tensor $\mathbf{T}$; $\vec{v} \defeq (v_1, v_2, \cdots, v_N)$ index a coefficient $T_{\vec{u}}(v_1, v_2, \cdots, v_N) \defeq T\big( (u_1;v_1), (u_2;v_2), \cdots, (u_N;v_N) \big)$ in a sub-tensor $\mathbf{T}_{\vec{u}}$, where each index $v_j \in \{1, 2, \cdots, \mathrm{dim}(V_{u_j}) \}$ for $j \in \{1, 2, \cdots, N \}$. 
\end{definition}

\begin{figure}[h!]
    \centering
    \includegraphics[width=0.8\linewidth]{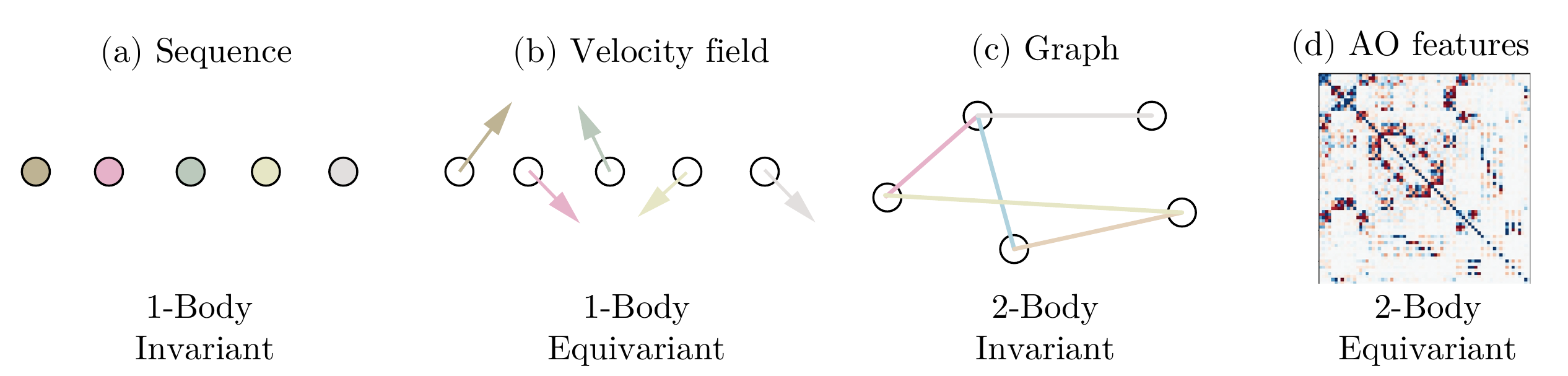}
    \caption{Examples of $N$-body tensors. }
    \label{fig:examples}
\end{figure}

\begin{figure}[h!]
    \centering
    \includegraphics[width=0.8\linewidth]{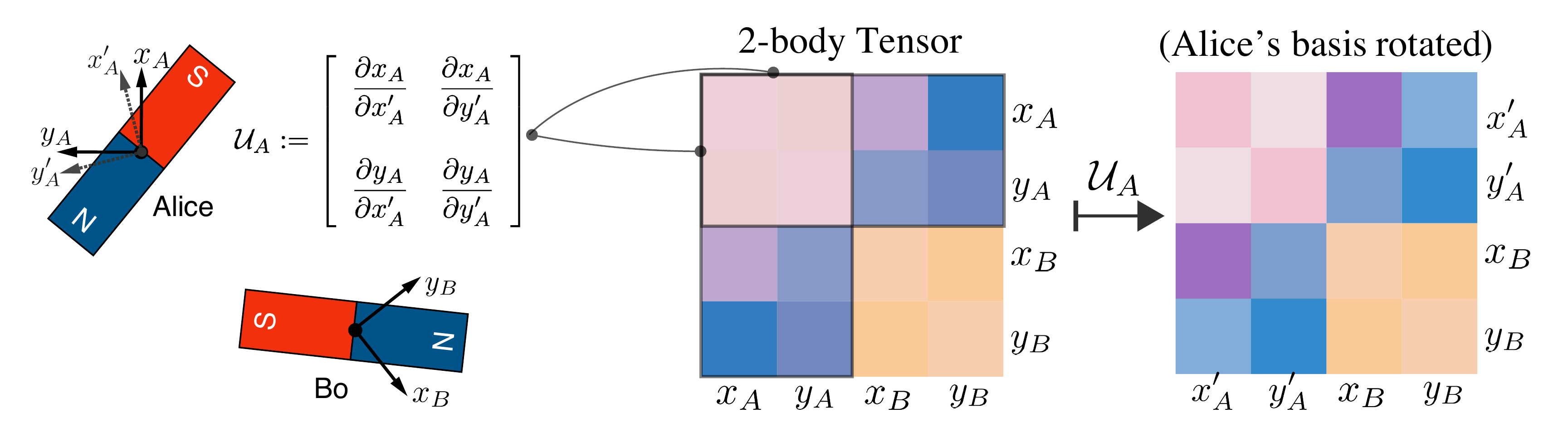}
    \caption{Illustrating an $N$-body tensor with $N=2$. Imagine Alice and Bo are doing experiments with two bar magnets without knowing each other's reference frame. The magnetic interactions depend on both bar magnets' orientations and can be written as a 2-body tensor. When Alice make a rotation on her reference frame, sub-tensors containing index $A$ are transformed by a unitary matrix $\mathcal{U}_A$, giving rise to the 2-body tensor coefficients in the transformed basis. We design neural network to be equivariant to all such local basis transformations.}
    \label{fig:magnet}
\end{figure}

\footnotetext{For example, if there are $d=5$ points defined in the 3-dimensional Euclidean space $\mathbb{R}^{3}$ and each point is associated with a standard basis $(x, y, z)$, then for the example of N=4, there are $5^{4}$ sub-tensors and each sub-tensor $\mathbf{T}_{(u_1, u_2, u_3, u_4)}$ contains $3^{4}=81$ elements with indices spanning from $xxxx$ to $zzzz$. In total, $\mathbf{T}$ contains $(5\times3)^{4}$ coefficients. The coefficients of $\mathbf{T}$ are in general complex-valued as formally discussed in Definition S2, but are real-valued for the special case introduced in Definition~\ref{def:nbti} .}

We aim to build neural networks $\mathcal{F}_{\theta} : (\bigoplus_{u=1}^{d} V_u)^{\otimes N} \rightarrow \mathcal{Y}$ that map 
$\hat{\mathbf{T}}$ to order-1 tensor- or scalar-valued outputs $\mathbf{y} \in \mathcal{Y}$. While $\hat{\mathbf{T}}$ 
is independent of the choice of local reference frame $\mathbf{e}_{u}$, its coefficients $\mathbf{T}$ (i.e. the $N$-body tensor) vary when rotating or reflecting the basis $\mathbf{e}_{u} \defeq \{\mathbf{e}_{u;v_u}; v_u \}$, i.e. acted by an element $\mathcal{U}_u \in \mathrm{O(n)}$.
Therefore, the neural network $\mathcal{F}_{\theta}$ should be constructed equivariant with respect to those reference frame transformations.

\subsection{Equivariance} 
For a map $f: \mathcal{V} \rightarrow \mathcal{Y}$ and a group $G$, $f$ is said to be $G$-equivariant if for all $g \in G$ and $\mathbf{v} \in \mathcal{V}$, $\varphi'(g) \cdot f(\mathbf{v}) = f(\varphi(g) \cdot \mathbf{v})$ where $\varphi(g)$ and $\varphi'(g)$ are the group representation of element $g$ on $\mathcal{V}$ and $\mathcal{Y}$, respectively.  In our case, the group $G$ is composed of 
(a) Unitary transformations $\mathcal{U}_u$ locally applied to basis: $\mathbf{e}_{u} \mapsto \mathcal{U}_u^{\dagger} \cdot \mathbf{e}_{u}$, which are rotations and reflections for $\mathbb{R}^{n}$. $\mathcal{U}_u$ induces transformations on tensor coefficients: $\mathbf{T}_{\vec{u}} \mapsto (\mathcal{U}_{u_1} \otimes \mathcal{U}_{u_2} \otimes \cdots \otimes  \mathcal{U}_{u_N}) \mathbf{T}_{\vec{u}}$, and an intuitive example for infinitesimal basis rotations in $N=2,n=2$ is shown in Figure \ref{fig:magnet}; 
(b) Tensor index permutations: $(\vec{u},\vec{v}) \mapsto \sigma(\vec{u},\vec{v})$; 
(c) Global translations: $\mathbf{x} \mapsto \mathbf{x}+\mathbf{x}_0$. 
For conciseness, we borrow the term $G$-equivariance to say $\mathcal{F}_{\theta}$ is equivariant to all the symmetry transformations listed above.





\subsection{$N$-body tensors}

Here we generalize the definition of N-body tensors to the basis of irreducible group representations instead of a Cartesian basis. The atomic orbital features discussed in the main text fall into this class, since the angular parts of atomic orbitals (i.e., spherical harmonics $Y_{lm}$) form the basis of the irreducible representations of group $\mathrm{SO(3)}$.

\begin{definition} \label{def:nbt}
 Let $G_1, G_2, \cdots, G_d$ denote unitary groups where $G_u \subset \mathrm{U(n)}$ are closed subgroups of $\mathrm{U(n)}$ for each $u\in \{1, 2, \cdots, d \}$. We denote $G \defeq G_1 \times G_2 \times \cdots \times G_d$. Let $(\pi_L, \mathbb{V}^L)$ denote a irreducible unitary representation of $\mathrm{U(n)}$ labelled by $L$. For each $u \in \{0, 1, \cdots, d\}$, we assume there is a finite-dimensional Banach space $V_u \simeq \bigoplus_{L} (\mathbb{V}^L)^{\oplus K_L}$ where $K_L \in \mathbb{N}$ is the multiplicity of $\mathbb{V}^L$ (e.g. the number of feature channels associated with representation index $L$), with basis $\{\bm{\pi}_{L, M} \}_u$ such that $\mathrm{span}(\{\bm{\pi}_{L, M, u} ; k, L, M\}) = V_u$ for each $u\in \{1, 2, \cdots, d \}$ and $k\in \{1, 2, \cdots, K_L \}$, and $\mathrm{span}(\{\bm{\pi}_{L, M, u} ; M\}) \simeq \mathbb{V}^{L}$ for each $u, L$. We denote $\mathcal{V}\defeq \bigoplus_{u} V_u$, and index notation $v\defeq (k, L, M)$. For a tensor $\hat{\mathbf{T}} \in \mathcal{V}^{\otimes N}$, we call the coefficients $\mathbf{T}$ of $\hat{\mathbf{T}}$ in the $N$-th direct products of basis  $\{\bm{\pi}_{L, M, u}; L, M, u \}$ an $N$-body tensor, if $\hat{\mathbf{T}} = \sigma( {\hat{\mathbf{T}}} )$ for any permutation $\sigma \in \mathrm{Sym}(N)$ (i.e. permutation invariant). 
\end{definition}

Note that the vector spaces $V_u$ do not need to be embed in the same space $\mathbb{R}^{n}$ as in the special case from Definition S1, but can be originated from general `parameterizations' $u \mapsto V_u$, e.g., coordinate charts on a manifold.

\begin{corollary}
If $V_u = \mathbb{C}^{n}$, $G_u = \mathrm{U(n)}$ and $\bm{\pi}_{L, M, u} = \mathbf{e}_M$ where $\{ \mathbf{e}_M \}$ is a standard basis of $\mathbb{C}^{n}$, then $\mathbf{T}$ is an $N$-body tensor if $\hat{\mathbf{T}}$ is permutation invariant.
\end{corollary}

\begin{proof} For $V_u = \mathbb{C}^{n}$, $\pi: G_u \rightarrow \mathrm{U}(\mathbb{C}^{n})$ is a fundamental representation of $\mathrm{U(n)}$. Since the fundamental representations of a Lie group are irreducible, it follows that $\{ \mathbf{e}_M \}$ is a basis of a irreducible representation of $\mathrm{U(n)}$, and $\mathbf{T}$ is an $N$-body tensor. \end{proof}

Similarly, when $V_u = \mathbb{R}^{n}$ and $G_u = \mathrm{O(n)} \subset \mathrm{U(n)}$, $\mathbf{T}$ is an $N$-body tensor if $\hat{\mathbf{T}}$ is permutation invariant. Then we can recover the special case based on point clouds in $\mathbb{R}^{n}$ in Definition~\ref{def:nbti}. 


Procedures for constructing complete bases for irreducible representations of $\mathrm{U(n)}$ with explicit forms are 
established~\cite{molev2002gelfand}. A special case is $G_u = \mathrm{SO(3)}$, for which a common construction of a complete set of $\{\bm{\pi}_{L, M} \}_u$ is using the spherical harmonics $ \bm{\pi}_{L, M, u} \defeq  Y_{lm} $; this is an example that polynomials $Y_{lm}$ can be constructed as a basis of square-integrable functions on the 2-sphere $L^{2}(S^{2})$ and consequently as a basis of the irreducible representations $(\pi_L, \mathbb{V}^L)$ for all $L$~\cite{hall2013quantum}. 



\subsection{Decomposition of diagonals $\mathbf{T}_u$}

\label{si_wigner}

We consider the algebraic structure of the diagonal sub-tensors $\mathbf{T}_u$, which can be understood from tensor products of irreducible representations.

First we note that for a sub-tensor $\mathbf{T}_{\vec{u}} \in V_{u_1} \otimes V_{u_2} \otimes \cdots \otimes V_{u_N}$, the action of $g \in G$ is given by
\begin{equation}
    g \cdot \mathbf{T}_{\vec{u}} = (\pi(g_{u_1}) \otimes \pi(g_{u_2}) \otimes \cdots \otimes \pi(g_{u_N}) )\mathbf{T}_{\vec{u}}
\end{equation}
for diagonal sub-tensors $\mathbf{T}_{u}$, this reduces to the action of a diagonal sub-group 
\begin{equation}
    g \cdot \mathbf{T}_{u} = (\pi(g_{u}) \otimes \pi(g_{u}) \otimes \cdots \otimes \pi(g_{u}) )\mathbf{T}_{u}
\end{equation}
which forms a representation of $G_u \in \mathrm{U}(n)$ on $V_u^{\otimes N}$. According to the isomorphism $V_u \simeq \bigoplus_{L} (\mathbb{V}^L)^{\oplus K_L}$ in Definition S2 we have $\pi(g_u) \cdot \mathbf{v} = \bigoplus_{L} U^{L}_{g_u} \cdot \mathbf{v}_{L}$ for $\mathbf{v} \in V_u$ where $\mathbf{v}_L \in \mathbb{V}^{L}$, more explicitly
\begin{equation}
    g \cdot \mathbf{T}_{u}(\vec{k}, \vec{L}) = ( U^{L_1}_{g_u} \otimes U^{L_2}_{g_u} \otimes \cdots \otimes U^{L_N}_{g_u} ) \, \mathbf{T}_{u}(\vec{k}, \vec{L})
\end{equation}
where we have used the shorthand notation $\mathbf{T}_{u}(\vec{k}, \vec{L}) \defeq \mathbf{T}_{u}\big((k_1, L_1), (k_2, L_2), \cdots, (k_N, L_N)\big)$ and $U^{L}_{g_u}$ denotes the unitary matrix representation of $g_u \in \mathrm{U(n)}$ on $\mathbb{V}^{L}$ expressed in the basis $\{\bm{\pi}_{L,M,u}; M\}$, on the vector space $\mathbb{V}^{L}$ for the irreducible representation labelled by $L$. Therefore $\mathbf{T}_{u}(\vec{k}, \vec{L}) \in \mathbb{V}^{L_1} \otimes \mathbb{V}^{L_2} \otimes \cdots \otimes \mathbb{V}^{L_N}$ is the representation space of an $N$-fold tensor product representations of $\mathrm{U(n)}$. We note the following theorem for the decomposition of $\mathbf{T}_{u}(\vec{k}, \vec{L})$:

\begin{theorem}[Theorem 2.1 and Lemma 2.2 of~\cite{klink1996multiplicity}]
 The representation of $\mathrm{U(n)}$ on the direct product of $\mathbb{V}^{L_1}, \mathbb{V}^{L_2}, \cdots, \mathbb{V}^{L_N}$ decomposes into direct sum of irreducible representations:
\begin{equation}
    \mathbb{V}^{L_1} \otimes \mathbb{V}^{L_2} \otimes \cdots \otimes \mathbb{V}^{L_N} \simeq \bigoplus_{L} \bigoplus_{\nu}^{\mu(L_1, L_2, \cdots, L_N; L)} \mathbb{V}^{L;\nu}
\end{equation}
and 
\begin{equation}
    \label{eq:decomp}
    \sum_{L} \mu(L_1, L_2, \cdots, L_N; L) \, \mathrm{dim}(\mathbb{V}^{L}) = \prod_{u=1}^{N} \mathrm{dim}(\mathbb{V}^{L_u})
\end{equation}
where $\mu(L_1, L_2, \cdots, L_N; L)$ is the multiplicity of $L$ denoting the number of replicas of $\mathbb{V}^{L}$ being present in the decomposition of $\mathbb{V}^{L_1} \otimes \mathbb{V}^{L_2} \otimes \cdots \otimes \mathbb{V}^{L_N}$. 
\end{theorem}

Note that we have abstracted the labelling details for $\mathrm{U(n)}$ irreducible representations into the index $L$. See \cite{klink1996multiplicity} for proof and details on representation labelling. We now state the following result for generating order-1 representations (Materials and methods, ~\ref{sec:wigner}): 


\begin{corollary} 
There exists an invertible linear map $\psi: V_u^{\otimes N} \rightarrow V^{\star}_u \defeq \bigoplus (\mathbb{V}^{L})^{\oplus \mu(L;V_u)}$ where $\mu(L;V_u) \in \mathbb{N}$, such that for any $\mathbf{T}_{u}$, $L$ and $\nu \in \{1, 2, \cdots , \mu(L;V_u) \}$, $\psi(g_u \cdot \mathbf{T}_{u})_{\nu,L} = U^{L}_{g_u} \cdot \psi(\mathbf{T}_{u})_{\nu,L}$ if $\mu(L;V_u)>0$.
\end{corollary}

\begin{proof} 
First note that each block $\mathbf{T}_{u}(\vec{k}, \vec{L})$ of $\mathbf{T}_u$ is an element of $\mathbb{V}^{L_1} \otimes \mathbb{V}^{L_2} \otimes \cdots \otimes \mathbb{V}^{L_N}$ up to an isomorphism. \eqref{eq:decomp} in Theorem S1 states there is an invertible linear map $\psi_{\vec{L}}: \mathbb{V}^{L_1} \otimes \mathbb{V}^{L_2} \otimes \cdots \otimes \mathbb{V}^{L_N} \rightarrow \bigoplus_{L} (\mathbb{V}^{L})^{\oplus {\mu(L_1, L_2, \cdots, L_N; L)}}$, such that $\tau(g_u) = (\psi_{\vec{L}})^{-1} \circ \pi(g_u) \circ \psi_{\vec{L}}$ for any $g_u \in G_u$, where $\tau: G_u \rightarrow \mathrm{U}(\mathbb{V}^{L_1} \otimes \mathbb{V}^{L_2} \otimes \cdots \otimes \mathbb{V}^{L_N})$ and $\pi: G_u \rightarrow \mathrm{U}(\bigoplus_{L} (\mathbb{V}^{L})^{\oplus {\mu(L_1, L_2, \cdots, L_N; L)}})$ are representations of $G_u$. Note that $\pi$ is defined as a direct sum of irreducible representations of $\mathrm{U(n)}$, i.e. $\pi(g_u)  \psi_{\vec{L}}(\mathbf{T}_u(\vec{k}, \vec{L})) \defeq \bigoplus_{\nu,L} U^{L}_{g_u} \big( \psi_{\vec{L}}(\mathbf{T}_u(\vec{k}, \vec{L}))) \big)_{\nu,L}$. Note that $\psi(\mathbf{T}_u) \defeq \bigoplus_{L} \bigoplus_{\vec{k}, \vec{L}} \psi_{\vec{L}}(\mathbf{T}_u(\vec{k}, \vec{L}))_{L}$ and $\mu(L, V_u)\defeq \sum_{\vec{k}, \vec{L}} \mu(L_1, L_2, \cdots, L_N; L)$ directly satisfies $\psi(g_u \cdot \mathbf{T}_{u})_{\nu,L} = \psi(\tau(g_u) \mathbf{T}_{u})_{\nu,L} = U^{L}_{g_u} \psi(\mathbf{T}_{u})_{\nu,L}$ for $\nu \in \{1, 2, \cdots , \mu(L;V_u) \}$. Since each $\psi_{\vec{L}}$ are finite-dimensional and invertible, it follows that the finite direct sum $\psi$ is invertible.
\end{proof}

For Hermitian tensors, we conjecture the same result for $\mathrm{SU}(2)$, $\mathrm{O}(2)$ and $\mathrm{O}(3)$ as each irreducible representation is isomorphic to its complex conjugate. 

We then formally restate the proposition in Methods~\ref{sec:wigner} which was originally given for orthogonal representations of $\mathrm{O(3)}$ (i.e., the real spherical harmonics):
\begin{corollary} \label{cor:wigner}
For each $L$ where $\mu(L;V_u)>0$, there exist $n_L \times \mathrm{dim}(V_u)^N$ $\mathbf{T}$-independent coefficents $Q^{\vec{v}}_{\nu,L,M}$ parameterizing the linear transformation $\psi$ that performs $\mathbf{T}_{\vec{u}} \mapsto \mathbf{h}_u \defeq \psi(\mathbf{T}_{u})$, if $u_1=u_2=\cdots=u_N=u$:
\begin{equation}
    \label{eq:l1}
    \big( \psi( \mathbf{T}_{u} ) \big)_{\nu, L, M} \defeq \sum_{\vec{v}} T_{u}(v_1, v_2, \cdots v_N) \, Q^{\vec{v}}_{\nu,L,M} \quad \textrm{for} \quad \nu\in\{1,2,\cdots,n_L\}
\end{equation}
such that the linear map $\psi$ is injective, $\sum_{L} n_L \leq \mathrm{dim}(V_u)^N$, and for each $g_u\in G_u$:
\begin{equation}
    \label{eq:cg}
    \psi\big( g_u \cdot \mathbf{T}_{u} \big)_{\nu,L} = U^{L}_{g_u} \big( \psi(\mathbf{T}_{u}) \big)_{\nu,L} 
\end{equation}
\end{corollary} 

\begin{proof} According to Definition~\ref{def:nbt}, a complete basis of $V_u^{\otimes N}$ is given by $\{ \bm{\pi}_{L_1, M_1, u} \otimes \bm{\pi}_{L_2, M_2, u} \otimes \cdots \otimes \bm{\pi}_{L_N, M_N, u}; (\vec{k}, \vec{L}, \vec{M}) \}$ and a complete basis of $(V^{\star}_u)_L$ is $\{ \bm{\pi}_{L, M, u}; M \}$. Note that $V_u$ and $(V^{\star}_u)_L$ are both finite dimensional.  Therefore an example of $\mathbf{Q}_L$ is the $\mathrm{dim}(V_u)^N \times \mu(L;V_u)$ matrix representation of the bijective map $\psi$ in the two basis, which proves the existence. \end{proof}

Note that Corollary S3 does not guarantee the resulting order-1 representations $\mathbf{h}_u \defeq \psi(\mathbf{T}_{u})$ (i.e. vectors in $V^{\star}_u$) to be invariant under permutations $\sigma$, as the ordering of $\{\nu\}_{L}$ may change under $\mathbf{T} \mapsto \sigma(\mathbf{T})$.
Hence, the symmetric condition on $\mathbf{T}$ is important to achieve permutation equivariance for the decomposition $V_u^{\otimes N} \rightarrow V^{\star}_u$; we note that $\mathbf{T}_u$ has a symmetric tensor factorization and is an element of $\mathrm{Sym}^{N}(V_u)$, then algebraically the existence of a permutation-invariant decomposition is ensured by the Schur-Weyl duality~\cite{fulton2013representation} giving the fact that all representations in the decomposition of $\mathrm{Sym}^{N}(V_u)$ must commute with the symmetric group $\mathrm{S}_N$. With the matrix representation $\mathbf{Q}$ in \eqref{eq:l1}, clearly for any $\sigma$, $\psi(\sigma(\mathbf{T}_u)) = \sigma(\mathbf{T}_u) \cdot \mathbf{Q} = \mathbf{T}_u \cdot \mathbf{Q} = \psi(\mathbf{T}_u)$.  For general asymmetric $N$-body tensors, we expect the realization of permutation equivariance to be sophisticated and may be achieved through tracking the Schur functors from the decomposition of $V_u^{\otimes N} \rightarrow V^{\star}_u$, which is considered out of scope of the current work.
Additionally, the upper bound $\sum_{L} n_L \leq \mathrm{dim}(V_u)^N$ is in practice often not saturated and the contraction \eqref{eq:l1} can be simplified. For example, when $N>2$ it suffices to perform permutation-invariant decomposition on symmetric $\mathbf{T}_u$ recursively through Clebsch-Gordan coefficients $\mathbf{C}$ 
which has the following property:
\begin{equation}
    \label{eq:cge}
    \mathbf{C}^{\nu,L}_{L_1;L_2} \cdot (U^{L_1}_{g_u} \otimes U^{L_2}_{g_u}) \cdot (\mathbf{C}^{\nu,L}_{L_1;L_2})^{\dagger} = U^{L}_{g_u} \quad \textrm{for} \quad \nu \in \{1, 2, \cdots, \mu(L_1, L_2; L)\}
\end{equation}
i.e., $\mathbf{C}$ parameterizes the isomorphism $\psi_{\vec{L}}$ of Theorem S2 for $N=2$, $\vec{L}=(L_1, L_2)$. Then $\psi$ can be constructed with the procedure $(V_u)^{\otimes N} \mapsto V'_u \otimes (V_u)^{\otimes (N-2)} \mapsto V''_u \otimes (V_u)^{\otimes (N-3)} \mapsto V^{\star}_u$ without explicit order-$N+1$ tensor contractions, where each reduction step can be parameterized using $\mathbf{C}$. 

Procedures for computing $\mathbf{C}$ in general are established~\cite{alex2011numerical,gliske2007algorithms}. For the main results reported in this work $\mathrm{O(3)}\simeq \mathrm{SO(3)} \times \mathbb{Z}_2$ is considered, where $\mu(L_1, L_2; L) \leq 1$ and the basis of an irreducible representation $\bm{\pi}_{L,M}$ can be written as $\bm{\pi}_{L,M} \defeq \lvert l,m,p\rangle$ where $p \in \{1, -1 \}$ and $m \in \{-l, -l+1, \cdots, l-1, l \}$. $\rvert l,m,p\rangle$ can be thought as a spherical harmonic $Y_{lm}$ but may additionally flips sign under point reflections $\mathcal{I}$ depending on the parity index $p$: $\mathcal{I} \, \lvert l,m,p\rangle = p \cdot (-1)^{l} \, \lvert l,m,p\rangle$ where $\forall \mathbf{x} \in \mathbb{R}^{3}, \mathcal{I}(\mathbf{x}) = -\mathbf{x}$. Clebsch-Gordan coefficients $\mathbf{C}$ for $\mathrm{O}(3)$ is given by:
\begin{equation}
    \label{eq:o3cg}
    C^{\nu=1, lpm}_{l_1 p_1 m_1; l_2 p_2 m_2} = C_{l_1 m_1; l_2 m_2}^{lm} \delta^{((-1)^{l_1+l_2+l})}_{p_1\cdot p_2\cdot p}
\end{equation}
where $C_{l_1 m_1; l_2 m_2}^{lm}$ are $\mathrm{SO(3)}$ Clebsch-Gordan coefficients. 
For $N=2$, the problem reduces to using Clebsch-Gordan coefficients to decompose $\mathbf{T}_{u}$ as a combination of matrix representations of \textit{spherical tensor operators} which are linear operators transforming under irreducible representation of $\mathrm{SO(3)}$ based on the the Wigner-Eckart Theorem (see~\cite{sakurai1995modern} for formal derivations). Remarkably, a recent work~\cite{lang2020wigner} discussed connections of a class of neural networks to the Wigner-Eckart Theorem in the context of operators in spherical CNNs, which also provides a thorough review on this topic.



 

Both $V_u$ and $V_u^{\star}$ are defined as direct sums of the representation spaces $\mathbb{V}^{L}$ of irreducible representations of $G_u$, but each $L$ may be associated with a different multiplicity $K_L$ or $K_L^{\star}$ (e.g. different numbers of feature channels). We also allow for the case that the definition basis $\{\mathbf{e}_{u;v} \}$ for the $N$-body tensor $\mathbf{T}$ differ from $\{\bm{\pi}_{u; L, M} \}$ by a known linear transformation $\mathbf{D}_u$ such that $ \mathbf{e}_{u;v} \defeq \sum_{L, M} (D_u)_{v}^{L,M}  \bm{\pi}_{u; L, M} $, or $(D_u)_{v}^{L,M} \defeq \langle \mathbf{e}_{u;v}, \bm{\pi}_{u; L, M} \rangle$ where $\langle \cdot, \cdot \rangle$ denotes an Hermitian inner product, and we additionally define if $K_L=0$, $\langle \mathbf{e}_{u;v}, \bm{\pi}_{u; L, M} \rangle \defeq 0$. We then give a natural extension to Definition S2:

\begin{definition}
  We extend the basis in Definition~\ref{def:nbt} for $N$-body tensors to $\{\mathbf{e}_{u;v} \}$ where $\mathrm{span}(\{\mathbf{e}_{u;v}; v \}) = V_u$, if 
\begin{equation}
    \mathbf{D}_u \cdot \pi^{2}(g_u) = \pi^{1}(g_u) \cdot \mathbf{D}_u \quad \forall g_u \in G_u
\end{equation}
where $\pi^{1}$ and $\pi^{2}$ are matrix representations of $g_u$ on $V_u \subset V_u^{\star}$ in basis $\{\mathbf{e}_{u;v} \}$ and in basis $\{\bm{\pi}_{u; L, M} \}$. Note that $\pi^{2}(g_{u}) \cdot \mathbf{v} = U^{L}_{g_u} \cdot \mathbf{v}$ for $\mathbf{v}\in \mathbb{V}^{L}$. 
\end{definition}

\subsection{Generalized neural network building blocks}
\label{sec:si_nnbb}

We clarify that in all the sections below $n$ refers to a feature channel index within a irreducible representation group labelled by $L$, which should not be confused with $\mathrm{dim}(V_u)$. More explicitly, we note $n \in \{1, 2, \cdots, N^{\mathrm{h}}_{L} \}$ where $N^{\mathrm{h}}_{L}$ is the number of vectors in the order-1 tensor $\mathbf{h}_u^t$ that transforms under the $L$-th irreducible representation $G_u$ (i.e. the multiplicity of $L$ in $\mathbf{h}_u^t$). $M \in \{1, 2, \cdots, \mathrm{dim}(\mathbb{V}^{L}) \}$ indicates the $M$-th component of a vector in the representation space of the $L$-th irreducible representation of $G_u$, corresponding to a basis vector $\bm{\pi}_{L,M,u}$. We also denote the total number of feature channels in $\mathrm{h}$ as $N^{\mathrm{h}} \defeq \sum_{L} N^{\mathrm{h}}_{L}$.

For a simple example, if the features in the order-1 representation $\mathbf{h}^{t}$ are specified by $L\in \{0, 1\}$, $N^{\mathrm{h}}_{L=0}=8$, $N^{\mathrm{h}}_{L=1}=4$, $\mathrm{dim}(\mathbb{V}^{L=0})=1$, and $\mathrm{dim}(\mathbb{V}^{L=1})=5$, then $N^{\mathrm{h}}=8+4=12$ and $\mathbf{h}^{t}_u$ is stored as an array with $\sum_{L} N^{\mathrm{h}}_{L} \cdot \mathrm{dim}(\mathbb{V}^{L}) = (8\times 1 + 4\times 5) = 28$ elements.

We reiterate that $\vec{u}\defeq (u_1, u_2, \cdots, u_N)$ is a sub-tensor index (location of a sub-tensor in the $N$-body tensor $\mathbf{T}$), and $\vec{v}\defeq (v_1, v_2, \cdots, v_N)$ is an element index in a sub-tensor $\mathbf{T}_{\vec{u}}$. 

\paragraph{Convolution and message passing.} We generalize the definition of a block convolution module \eqref{eq:conv} to order-$N$ and complex numbers:
\begin{equation}
    \label{eq:si_conv}
    (\mathbf{m}_{\vec{u}}^{t})^i_{v_1} = \sum_{v_2, \cdots, v_N} T_{\vec{u}}(v_1, v_2, \cdots, v_N) \prod_{j=2}^{N}    \big( \rho_{u_j}(\mathbf{h}^{t}_{u_j})^{*} \big)^i_{v_j}
\end{equation}
Message passing modules \eqref{eq:mp}-\eqref{eq:accu} is generalized to order $N$:
\begin{align}
    \label{eq:si_mp}
    \tilde{\mathbf{m}}_{u_1}^t &= \sum_{u_2, u_3, \cdots, u_{N}} \bigoplus\limits_{i,j} \ (\mathbf{m}_{\vec{u}}^{t})^{i}  \cdot \alpha_{\vec{u}}^{t,j} \\
    \mathbf{h}_{u_1}^{t+1} &= \phi\big( \mathbf{h}_{u_1}^{t}, \rho_{u_1}^{\dagger}(\tilde{\mathbf{m}}_{u_1}^t) \big)
    \label{eq:si_accu}
\end{align}

\paragraph{EvNorm.} 
We write the EvNorm operation \eqref{eq:EvNorm} as $\mathrm{EvNorm}: \mathbf{h} \mapsto ( \bar{\mathbf{h}}, \hat{\mathbf{h}})$ where
\begin{equation}
\label{eq:si_evnorm}
    \bar{h}_{n L} \defeq \frac{\lvert \lvert \mathbf{h}_{n L}\lvert \lvert -\mu^x_{n L}}{\sigma^x_{n L}} \quad \textrm{and} \quad
    \hat{h}_{n LM} \defeq \frac{x_{n LM}}{\lvert \lvert \mathbf{h}_{n L}\lvert \lvert  + 1/\beta_{n L} + \epsilon}
\end{equation}

\paragraph{Point-wise interaction $\phi$.} We adapt the notations and explicitly expand \eqref{eq:pi01}-\eqref{eq:pi02} for clarity. The operations within a point-wise interaction block $\mathbf{h}_u^{t+1} = \phi(\mathbf{h}_u^{t}, \mathbf{g}_u)$ are defined as :
\begin{align}
    (\mathbf{f}_u^{t})_{n L M} &= \big( \mathrm{MLP}_{1}(\bar{\mathbf{h}}_A^{t}) \big)_{n L} \, (\hat{\mathbf{h}}_u^{t})_{n LM} \quad \textrm{where} \quad (\bar{\mathbf{h}}_u^{t}, \hat{\mathbf{h}}_u^{t} ) = \mathrm{EvNorm}(\mathbf{h}_u^t) \label{eq:pi1} \\
    (\mathbf{q}_{u})_{n L M} &= (\mathbf{g}_{u})_{n L M} + \sum_{L_1, L_2}\sum_{M_1, M_2} (\mathbf{f}_u^t)_{n L_1 M_1} (\mathbf{g}_u)_{n L_2 M_2} \, C^{\nu(n), LM}_{L_1 M_1; L_2 M_2} \label{eq:pi2}\\
    (\mathbf{h}_u^{t+1})_{n L M} &= (\mathbf{h}_u^{t})_{n L M} + \big(\mathrm{MLP}_{2}(\bar{\mathbf{q}}_u) \big)_{n L} \, (\hat{\mathbf{q}}_u)_{n L M} \quad \textrm{where} \quad ( \bar{\mathbf{q}}_u, \hat{\mathbf{q}}_u ) = \mathrm{EvNorm}(\mathbf{q}_u)  \label{eq:pi3}
\end{align}
where $\nu: \mathbb{N}^{+} \rightarrow \mathbb{N}^{+}$ assigns an output multiplicity index $\nu$ to a group of feature channels $n$. 

For the special example of $\mathrm{O(3)}$ where the output multiplicity $\mu(L_1, L_2; L) \leq 1$ (see Theorem S1 for definitions), we can restrict $\nu(n) \equiv 1$ for all values of $n$, and \eqref{eq:pi2} can be rewritten as
\begin{equation}
\label{eq:si_cgp}
    (\mathbf{q}_{u})_{n l p m} = (\mathbf{g}_{u})_{nlpm} + \sum_{l_1, l_2}\sum_{m_1, m_2} \sum_{p_1, p_2} (\mathbf{f}_u^t)_{n l_1 p_1 m_1} (\mathbf{g}_u)_{n l_2 p_2 m_2} \, C_{l_1 m_1; l_2 m_2}^{lm} \, \delta^{((-1)^{l_1+l_2+l})}_{p_1\cdot p_2\cdot p}
\end{equation}
which is based on the construction of $C^{\nu(n), LM}_{L_1 M_1; L_2 M_2}$ in \eqref{eq:o3cg}. The above form recovers \eqref{eq:cgp}.

\paragraph{Matching layers.} Based on Definition S3, we can define generalized matching layers $\rho_u$ and $\rho^{\dagger}_u$ as
\begin{align}
    \big( \rho_u( \mathbf{h}_u^{t}) \big)_{v}^{i} &= \sum_{L,M} \big( \mathbf{W}_{L}^{i} \cdot  (\mathbf{h}_u^{t})_{LM} \big) \cdot \langle \mathbf{e}_{u;v}, \bm{\pi}_{u;L,M} \rangle \\
    \big( \rho_u^{\dagger}( \tilde{\mathbf{m}}_u^{t}) \big)_{L M} &= \sum_{v}  \mathbf{W}_{L}^{\dagger} \cdot  (\tilde{\mathbf{m}}_u^{t})_{v} \cdot \langle \bm{\pi}_{u;L,M} , \mathbf{e}_{u;v} \rangle \label{eq:matching}
\end{align}
where $\mathbf{W}_L^{i}$ are learnable ($1 \times N^{\mathrm{h}}_{L}$) matrices; $\mathbf{W}_L^{\dagger}$ are learnable ($N^{\mathrm{h}}_{L} \times (N^{\mathrm{i}} \, N^{\mathrm{j}})$) matrices where $N^{\mathrm{i}}$ denotes the number of convolution channels (number of allowed $i$ in \eqref{eq:conv}).

\subsection{$G$-equivariance}

With main results from Corollary S2 and Corollary S3 and basic linear algebra, the equivariance of UNiTE can be straightforwardly proven. $G$-equivariance of the Diagonal Reduction layer $\psi$ is stated in Corollary S3, and it suffices to prove the equivariance for other building blocks.

\textit{Proof of  $G$-equivariance for the convolution block \eqref{eq:si_conv}.} For any $g\in G$:
\begin{align*}
    & \sum_{v_2, \cdots, v_N} \big(g \cdot T_{\vec{u}}(v_1, v_2, \cdots, v_N)\big) \prod_{j=2}^{N}    \big( \rho_{u_j}(g \cdot \mathbf{h}^{t}_{u_j})^{*} \big)^i_{v_j} \\
    &= \sum_{v_2, \cdots, v_N} ((\bigotimes_{\vec{u}} \pi^{1}(g_{u_j}) \cdot T_{\vec{u}})(v_1, v_2, \cdots, v_N)) \prod_{j=2}^{N}    \big( \rho_{u_j}(\pi^{2}(g_{u_j}) \cdot \mathbf{h}^{t}_{u_j})^{*} \big)^i_{v_j} \\
    &= \sum_{v_2, \cdots, v_N} \big((\bigotimes_{\vec{u}} \pi^{1}(g_{u_j}) \cdot T_{\vec{u}})(v_1, v_2, \cdots, v_N)\big) \prod_{j=2}^{N}    \big( \sum_{L,M} (\mathbf{D}_{u_j})_{v_j}^{L,M} \cdot \big( \mathbf{W}_{L}^{i} \cdot  (\pi^{2}(g_{u_j}) \cdot \mathbf{h}^{t}_{u_j})_{LM} \big) \big)^{*} \big)^i \\
    &= \sum_{v_2, \cdots, v_N} \big((\bigotimes_{\vec{u}} \pi^{1}(g_{u_j}) \cdot T_{\vec{u}})(v_1, v_2, \cdots, v_N)\big) \prod_{j=2}^{N}    \big( ( \sum_{L,M} \mathbf{W}_{L}^{i} \cdot (\mathbf{D}_{u_j})_{v_j}^{L,M} \cdot (\pi^{2}(g_{u_j}) \cdot \mathbf{h}^{t}_{u_j})_{LM} \big) \big)^{*} \big)^i \\
    &\stackrel{\textrm{(S35)}}{=} \sum_{v_2, \cdots, v_N} \big((\bigotimes_{\vec{u}} \pi^{1}(g_{u_j}) \cdot T_{\vec{u}})(v_1, v_2, \cdots, v_N)\big) \prod_{j=2}^{N}    \big( \sum_{L,M} \big( \mathbf{W}_{L}^{i} \cdot (\pi^{1}(g_{u_j}) \cdot \mathbf{D}_{u_j}^{L,M} \cdot \mathbf{h}^{t}_{u_j})_{v_j} \big) \big)^{*} \big)^i \\
    &= \sum_{v_2, \cdots, v_N} \big((\bigotimes_{\vec{u}} \pi^{1}(g_{u_j}) \cdot T_{\vec{u}})(v_1, v_2, \cdots, v_N)\big) \prod_{j=2}^{N}  \big( \pi^{1}(g_{u_j})^{*} \cdot (\rho_{u_j}(\mathbf{h}^{t}_{u_j}))^{*} \big)^i_{v_j} \\
    &= \sum_{v_2, \cdots, v_N} \sum_{v'_1, v'_2, \cdots, v'_N} \big( (\pi^{1}(g_{u_j}))_{v_j, v'_j} \cdot T_{\vec{u}}(v'_1, v'_2, \cdots, v'_N)\big) \prod_{j=2}^{N} \sum_{v''_j}  (\pi^{1}(g_{u_j})^{*})_{v_j, v''_j} \cdot \big((\rho_{u_j}(\mathbf{h}^{t}_{u_j}))^{*} \big)^i_{v''_j} \\
    &= \sum_{v'1, v'_2, \cdots, v'_N} (\pi^{1}(g_{u_1}))_{v_1, v'_1} T_{\vec{u}}(v'_1, v'_2, \cdots, v'_N) \prod_{j=2}^{N} \sum_{v_j, v''_j} (\pi^{1}(g_{u_j}))_{v_j, v'_j} (\pi^{1}(g_{u_j})^{*})_{v_j, v''_j} \cdot \big( (\rho_{u_j}(\mathbf{h}^{t}_{u_j}))^{*} \big)^i_{v''_j} \\
    &= \sum_{v'_1, v'_2, \cdots, v'_N} (\pi^{1}(g_{u_1}))_{v_1, v'_1} T_{\vec{u}}(v'_1, v'_2, \cdots, v'_N) \prod_{j=2}^{N} \sum_{v''_j} \delta_{v'_j}^{v''_j}  \cdot \big( (\rho_{u_j}(\mathbf{h}^{t}_{u_j}))^{*} \big)^i_{v''_j} \\
    &= \sum_{v'_1} (\pi^{1}(g_{u_1}))_{v_1, v'_1}  \sum_{v'_2, \cdots, v'_N} T_{\vec{u}}(v'_1, v'_2, \cdots, v'_N) \prod_{j=2}^{N} \big( (\rho_{u_j}(\mathbf{h}^{t}_{u_j}))^{*} \big)^i_{v'_j} \\
    &= \sum_{v'_1} (\pi^{1}(g_{u_1}))_{v_1, v'_1} (\mathbf{m}_{\vec{u}}^{t})^i_{v'_1}  \\
    &= \pi^{1}(g_{u_1}) \cdot (\mathbf{m}_{\vec{u}}^{t})^i_{v'_1} =  (g \cdot (\mathbf{m}_{\vec{u}}^{t})^i)_{v_1} \\
\end{align*}

\textit{Proof of $G$-equivariance for the message passing block \eqref{eq:si_mp}-\eqref{eq:si_accu}.} From the invariance condition $g \cdot \alpha_{\vec{u}}^{t,j} = \alpha_{\vec{u}}^{t,j}$, clearly
\begin{align*}
    & \sum_{u_2, u_3, \cdots, u_{N}} \bigoplus\limits_{i,j} \ (g \cdot \mathbf{m}_{\vec{u}}^{t})^{i}  \cdot \alpha_{\vec{u}}^{t,j} \\
    &= \sum_{u_2, u_3, \cdots, u_{N}} \bigoplus\limits_{i,j} (\pi^{1}(g_{u_1}) \cdot (\mathbf{m}_{\vec{u}}^{t})^i) \cdot \alpha_{\vec{u}}^{t,j} \\
    &= \pi^{1}(g_{u_1}) \cdot \sum_{u_2, u_3, \cdots, u_{N}} \bigoplus\limits_{i,j} ((\mathbf{m}_{\vec{u}}^{t})^i) \cdot \alpha_{\vec{u}}^{t,j} \\
    &= \pi^{1}(g_{u_1}) \cdot \tilde{\mathbf{m}}_{u_1}^t = g \cdot \tilde{\mathbf{m}}_{u_1}^t
\end{align*}


\textit{Proof of $G$-equivariance for $\mathrm{EvNorm}$ \eqref{eq:si_evnorm}.} Note that the vector norm $\lvert \lvert \mathbf{x}_{n L}\lvert \lvert $ is invariant to unitary transformations $\mathbf{x}_{n L} \mapsto U^{L}_{g_u} \cdot \mathbf{x}_{n L}$. Then $\overline{(g \cdot \mathbf{x})} = \frac{\lvert \lvert \mathbf{x}_{n L}\lvert \lvert -\mu^x_{n L}}{\sigma^x_{n L}} = \bar{\mathbf{x}}$,
and $\widehat{(g \, \mathbf{x}_{n L})} = \frac{(\pi^{2}(g_{u}) \cdot x)_{n LM}}{\lvert \lvert \mathbf{x}_{n L}\lvert \lvert  + 1/\beta_{n L} + \epsilon} = \pi^{2}(g_{u}) \cdot \hat{\mathbf{x}}_{n L} = g\cdot \hat{\mathbf{x}}_{n L}$.

\textit{Proof of $G$-equivariance for the point-wise interaction block \eqref{eq:pi1}-\eqref{eq:pi3}. } Equivariances for \eqref{eq:pi1} and \eqref{eq:pi3} are direct consequences of the equivariance of EvNorm $\overline{(g \cdot \mathbf{x})} = \bar{\mathbf{x}}$ and $\widehat{(g \, \mathbf{x}_{n L})} = g\cdot \hat{\mathbf{x}}_{n L}$, if $g_u \cdot \mathbf{x}_{n L} = \pi^{2}(g_{u}) \cdot \mathbf{x}_{n L} \equiv U^{L}_{g_u} \cdot \mathbf{x}_{n L}$. Then it suffices to prove $g \cdot (\mathbf{q}_{u})_{n L} = U^{L}_{g_u} \cdot (\mathbf{q}_{u})_{n L}$, which is ensured by \eqref{eq:cge}:
\begin{align*}
    & (g_u \cdot \mathbf{g}_{u})_{n L M} + \sum_{L_1, L_2}\sum_{M_1, M_2} (g_u \cdot \mathbf{f}_{u}^t)_{n L_1 M_1} (g_u \cdot \mathbf{g}_u)_{n L_2 M_2} \, C^{\nu(n), LM}_{L_1 M_1; L_2 M_2} \\
    &= (U_{g_u}^{L} \cdot \mathbf{g}_{u})_{n L M} + \sum_{L_1, L_2}\sum_{M_1, M_2} (U_{g_u}^{L_1} \cdot \mathbf{f}_{u}^t)_{n L_1 M_1} (U_{g_u}^{L_2} \cdot \mathbf{g}_u)_{n L_2 M_2} \, C^{\nu(n), LM}_{L_1 M_1; L_2 M_2} \\
    &= (U_{g_u}^{L} \cdot \mathbf{g}_{u})_{n L M} + \sum_{L_1, L_2}\sum_{M_1, M_2}\sum_{M'_1, M'_2} (U_{g_u}^{L_1} \otimes U_{g_u}^{L_2})_{M_1,M'_1}^{M_2,M'_2} \cdot (\mathbf{f}_{u}^t)_{n L_1 M'_1} (\mathbf{g}_u)_{n L_2 M'_2} \, C^{\nu(n), LM}_{L_1 M_1; L_2 M_2} \\
    &\stackrel{\textrm{(S33)}}{=} (U_{g_u}^{L} \cdot \mathbf{g}_{u})_{n L M} + \sum_{L_1, L_2}\sum_{M'_1, M'_2} \sum_{M'} (U_{g_u}^{L})_{M, M'} \cdot \big( (\mathbf{f}_{u}^t)_{n L_1 M'_1} (\mathbf{g}_u)_{n L_2 M'_2} \, C^{\nu(n), LM'}_{L_1 M'_1; L_2 M'_2} \big) \\
    &= U_{g_u}^{L} \cdot \big( \mathbf{g}_{u} + \sum_{L_1, L_2}\sum_{M'_1, M'_2} (\mathbf{f}_{u}^t)_{n L_1 M'_1} (\mathbf{g}_u)_{n L_2 M'_2} \, C^{\nu(n), LM'}_{L_1 M'_1; L_2 M'_2} \big)_{n L M} \\
    &= U_{g_u}^{L} \cdot (\mathbf{q}_{u})_{n L M} = \pi^{2}({g_u}) \cdot (\mathbf{q}_{u})_{n L M} = g \cdot (\mathbf{q}_{u})_{n L M}
\end{align*}

For permutation equivariance, it suffices to realize $\sigma(\mathbf{T})\equiv\mathbf{T}$ due to the symmetric condition in Definition S2 so \eqref{eq:si_conv} is invariant under $\sigma$, the permutation invariance of $\psi$ (see Equation \ref{eq:l1}), and the actions of $\sigma$ on network layers in $\phi$ defined for a single dimension $\{(u;v)\}$ are trivial (since $\sigma(u) \equiv u$). 



\section{Supplementary numerical results}
\label{si_results}


\begin{table}[h!]
\caption{Test mean absolute errors (MAEs) on QM9 for atomistic deep learning models and OrbNet-Equi trained on 110k samples. 
OrbNet-Equi results on first 8 tasks are obtained by training on the residuals between the DFT reference labels and the tight-binding QM model estimations (delta-learning), because the tight-binding results for these targets can be directly obtained from the single-point calculation that featurizes the molecule. Results on the last 4 targets are obtained through directly training on the target properties (direct-learning). 
}
\label{table:qm9}
\centering
\begin{tabular}{m{0.1\textwidth}m{0.05\textwidth}m{0.06\textwidth}m{0.09\textwidth}m{0.1\textwidth}m{0.06\textwidth}m{0.09\textwidth}m{0.13\textwidth}}
\toprule
{Target} & {Unit} & {SchNet} & {Cormorant} & {DimeNet++} & {PaiNN} & {SphereNet} & {OrbNet-Equi} \\ 
\midrule
$\mu$  & {mD}    & {33}  & {38} & {29.7}  & {12}  & {26.9}  & { {6.3}$\pm$0.2} \\
$\epsilon_{\mathrm{HOMO}}$  & {meV} & {41}  & {32.9} & {24.6}  & {27.6}  & {23.6} & {{9.9}$\pm$0.02}   \\
$\epsilon_{\mathrm{LUMO}}$  & {meV} & {34} & {38}  &{19.5} &{20.4}  & {18.9} & {{12.7}$\pm$0.3}   \\
$\Delta\epsilon$ & {meV}    & {63} &{38} &{32.6} & {45.7} & {32.3} & {{17.3}$\pm$0.3}   \\
$U_0$ & {meV} & {14} & {22} & {6.3} & {5.9} & {6.3} & {{3.5}$\pm$0.1}    \\
$U$   & {meV} & {19} & {21} & {6.3} & {5.8} & {7.3} & {{3.5}$\pm$0.1}    \\
$H$   & {meV} & {14} & {21} & {6.5} & {6.0} & {6.4} & {{3.5}$\pm$0.1}    \\
$G$   & {meV} & {14} & {20} & {7.6} & {7.4} & {8.0} & {{5.2}$\pm$0.1}    \\
\midrule
$\alpha$ & $a_0^3$   & {0.235} & {0.085}  & {{0.044}} & {0.045} & {0.047}  & {{0.036}$\pm$0.002} \\
$\langle R^{2} \rangle$  & $a_0^2$ & {0.073} & {0.961}  & {0.331} & {0.066} & {0.292} & {{0.030}$\pm$0.001} \\
ZPVE & {meV} & {1.7} & {2.0} & {1.2} & {1.3} & {{1.1}} & {{1.11}$\pm$0.04}    \\
$c_v$ & $\mathrm{\frac{cal}{mol K}}$ & {0.033} & {0.026} & {0.023} & {0.024} & {0.022} & {{0.022}$\pm$0.001}  \\
\midrule
\midrule
std. MAE & \% & 1.76 & 1.44 & 0.98 & 1.01 & 0.94 & {0.47} \\
log. MAE & - & -5.2 & -5.0 & -5.7 & -5.8 & -5.7 & {-6.4} \\ 
\bottomrule
\end{tabular}
\end{table}

\begin{table}[h!]
\scriptsize
\caption{Test MAEs on the rMD17 dataset in terms of energies (in meV) and forces (in meV/\AA) for models trained on 1000 sample geometries for each molecular system. For OrbNet-Equi, both direct learning and delta learning results are reported. 
}
\label{table:rmd17}
\centering
\begin{tabular}{llcccc}
\toprule
Molecule & & FCHL19~\cite{rmd17} & NequIP ($l=3$)~\cite{nequip} & \makecell{OrbNet-Equi (direct learning)} & \makecell{OrbNet-Equi (delta learning)} \\ \midrule
\multirow{2}{*}{Aspirin}        & Energy  & 6.2    & 2.3 & 2.4 & 1.8 \\
               & Forces  & 20.9    & 8.5  &  7.6 & 6.1 \\ \midrule
\multirow{2}{*}{Azobenzene}        & Energy  & 2.8    & 0.7     &  1.1  & 0.63 \\
               & Forces  &  10.8   & 3.6  &  4.2 & 2.7  \\ \midrule
\multirow{2}{*}{Ethanol}        & Energy  & 0.9    & 0.4      & 0.62 & 0.42  \\
               & Forces & 6.2      & 3.4  & 3.7 & 2.6 \\ \midrule
\multirow{2}{*}{Malonaldehyde}  & Energy  & 1.5      & 0.8      & {1.2} & 0.80   \\
               & Forces  & 10.2      & 5.2  & 7.1 & 4.6  \\ \midrule
\multirow{2}{*}{Naphthalene}    & Energy  & 1.2     & 0.2      & {0.46}  & 0.27  \\
               & Forces  & 6.5      & 1.2  & 2.6 & 1.5  \\ \midrule
\multirow{2}{*}{Paracetamol}        & Energy  & 2.9    & 1.4     &  1.9 & 1.2 \\
               & Forces  &  12.2  & 6.9  &  7.1 & 4.5  \\ \midrule
\multirow{2}{*}{Salicylic Acid} & Energy  & 1.8      & 0.7      & {0.73} & 0.52  \\
               & Forces  & 9.5      & 4.0  & {3.8}  & 2.9 \\ \midrule
\multirow{2}{*}{Toluene}        & Energy  & 1.6      & 0.3      & {0.45} & 0.27   \\
               & Forces  & 8.8      & 1.6  & 2.5 & 1.6  \\ \midrule
\multirow{2}{*}{Uracil}         & Energy  & 0.4       & {0.4}  & {0.58} & 0.35   \\
               & Forces  & 4.2      & 3.2  & 3.8 & 2.4  \\ \midrule
\multirow{2}{*}{Benzene}        & Energy   & 0.3        & 0.04      & {0.07} & 0.02   \\
               & Forces & 2.6        & 0.3  & 0.73  & 0.27   \\ \bottomrule
\end{tabular}
\end{table}


\begin{table}[h!]
\scriptsize
\caption{OrbNet-Equi test force MAEs (in kcal/mol/\AA) on the original MD17 dataset using 1000 training geometries. }
\label{table:md17}
\centering
\begin{tabular}{lcc}
\toprule
Molecule & OrbNet-Equi (direct learning) & \makecell{OrbNet-Equi  (delta learning)} \\ \midrule
Aspirin  & {0.156} & 0.118 \\ \midrule
Ethanol  & {0.092} & 0.069 \\ \midrule
Malonaldehyde   & {0.159} & 0.128 \\ \midrule
Naphthalene  & {0.064} & 0.048 \\ \midrule
Salicylic Acid  & {0.097} & 0.067 \\ \midrule
Toluene  & {0.072}  & 0.057 \\ \midrule
Uracil   & {0.098} & 0.072  \\  \bottomrule
\end{tabular}
\end{table}

\begin{table}[h!]
\caption{OrbNet-Equi inference time breakdowns (mean/std in milliseconds) for the calculation of energy and forces on the Hutchison dataset~\cite{Hutch}.}
\centering
\label{table:infer}
\begin{tabular}{cc|cc}
\toprule
Feature generation & NN inference & NN back propagation & Nuclear gradients calculation \\
\midrule
85.8 $\pm$ 40.1 & 181 $\pm$ 83 & 273 $\pm$ 73 & 33.2 $\pm$ 1.8\\
\bottomrule
\end{tabular}
\end{table}

\begin{table}[h!]
\tiny
\caption{Summary statistics of representative semi-empirical quantum mechanics (GFN-xTB and GFN2-xTB), machine learning (ANI-2x), density functional theory (B97-3c) methods and OrbNet-Equi/SDC21 on down-steam tasks. See section~\ref{si_downstream} regarding results on geometry optimization tasks.
}
\label{table:downstream}
\centering
\begin{tabular}{cccccccc}
\toprule
Task & Dataset & Metric & GFN-xTB & GFN2-xTB & ANI-2x & B97-3c & OrbNet-Equi/SDC21 \\ \midrule
Conformer ordering & Hutchison~\cite{Hutch} & Med. $R^2$ / DLPNO-CCSD(T) & 0.62$\pm$0.04 & 0.64$\pm$0.04 & 0.63$\pm$0.06 & {0.90$\pm$0.01} & {0.87$\pm$0.02} \\
Conformer ordering & Hutchison~\cite{Hutch} & Med. $R^2$ / $\omega$B97X-D3/def2-TZVP & 0.64$\pm$0.04 & 0.69$\pm$0.04 & 0.68$\pm$0.04 & 0.97$\pm$0.01 & 0.96$\pm$0.01 \\
Torsion profiles & TorsionNet~\cite{rai2020torsionnet} & MAE\tablefootnote{With respect to $\omega$B97X-D3/def2-TZVP. Note that ANI-2x is trained on a different DFT theory and the number is provided for reference only.} (kcal/mol) & 0.948$\pm$0.017 & 0.731$\pm$0.013 & 0.893$\pm$0.017 & 0.284$\pm$0.006 & {0.173$\pm$0.003} \\
Geometry optimization & ROT34~\cite{rot34} & Avg. RMSD (Å) & 0.227$\pm$0.087 & 0.210$\pm$0.072 & - & {0.063$\pm$0.013} & {0.045$\pm$0.005} \\[-1pt]
Geometry optimization & MCONF~\cite{mconf} & Avg. RMSD (Å) & 0.899$\pm$0.106 & 0.603$\pm$0.064 & - & 0.511$\pm$0.072 & {0.227$\pm$0.042} \\ 
\bottomrule
\end{tabular}
\end{table}

\begin{table}[h!]
\small
\caption{MAEs (in kcal/mol) of binding energy predictions on the S66x10~\cite{smith2016revised} dataset, computed for different inter-molecular distances in $r_\mathrm{e}$ unit with $r_\mathrm{e}$ being the equilibrium inter-molecular distance. $\omega$B97X-D3/def2-TZVP binding energy results are used as the reference level of theory. Standard errors of the mean are reported in the parentheses. }
\label{table:s66x10}
\centering
\begin{tabular}{llllll}
\toprule
Distance ($r_\mathrm{e}$) & GFN-xTB        & GFN2-xTB       & ANI-2x         & B97-3c         & OrbNet-Equi/SDC21    \\ \midrule
0.7      & 6.7584(2.1923) & 6.8887(2.2193) & 2.3236(0.5964) & 1.7856(0.6036) & 1.6443(0.4657) \\
0.8      & 2.6225(0.7901) & 2.8569(0.8791) & 1.1433(0.2438) & 0.9751(0.2456) & 0.9241(0.2836) \\
0.9      & 1.4087(0.1956) & 1.3301(0.2715) & 1.0103(0.1603) & 0.5922(0.1034) & 0.5336(0.1515) \\
0.95     & 1.4365(0.1694) & 1.2018(0.1807) & 0.9752(0.1589) & 0.5018(0.0837) & 0.4341(0.1124) \\
1.0      & 1.5552(0.1730) & 1.1927(0.1773) & 0.9688(0.1484) & 0.4433(0.0673) & 0.3540(0.0881) \\
1.05     & 1.5962(0.1740) & 1.1960(0.1845) & 0.9501(0.1461) & 0.3756(0.0525) & 0.3090(0.0872) \\
1.1      & 1.5577(0.1751) & 1.1800(0.1848) & 0.9404(0.1620) & 0.3049(0.0435) & 0.3328(0.0946) \\
1.25     & 1.2690(0.1694) & 0.9764(0.1715) & 0.9645(0.1706) & 0.1344(0.0241) & 0.4432(0.0736) \\
1.5      & 0.8270(0.1533) & 0.5764(0.1211) & 0.8503(0.1362) & 0.0610(0.0165) & 0.4697(0.0687) \\
2.0      & 0.3346(0.0899) & 0.1664(0.0370) & 0.7139(0.2071) & 0.0294(0.0078) & 0.2820(0.0744) \\ \bottomrule
\end{tabular}
\end{table}

\begin{table}[h!]
\scriptsize
\caption{Subset-averaged WTMAD-2 ($\text{WTMAD-2}_i = \frac{1}{N_i} \sum_{j} \text{WTAD}_{i,j}$, see Methods~\ref{si_summary}) on the GMTKN55 collection of benchmarks, reported for all potential energy methods considered in this study against the CCSD(T)/CBS reference values. Standard errors of the mean are reported in the parentheses. For cases in which no reaction within a subset is supported by a method, results are marked as ``-''. The OrbNet-Equi/SDC21 (filtered) column corresponds to OrbNet-Equi/SDC21 evaluated on reactions that consist of chemical elements and electronic states appeared in the SDC21 training dataset, as shown in Figure~\ref{fig:7}b.}
\label{table:gmtkn55}
\centering
\begin{tabular}{ccccccccc}
\toprule
Group                         & Subset    & GFN1-xTB     & GFN2-xTB    & ANI-2x     & B97-3c    & $\omega$B97xD3   & OrbNet-Equi/SDC21  & OrbNet-Equi/SDC21 (filtered) \\ \midrule
\multirow{18}{*}{Prop. small} & W4-11     & 32.5(1.5)    & 22.0(1.2)   & -          & 1.4(0.1)  & 0.7(0.1)  & 31.1(1.5)    & -                      \\
                              & G21EA     & 392.6(185.7) & 158.3(12.3) & -          & 13.9(1.3) & 12.0(1.4) & 254.1(182.7) & -                      \\
                              & G21IP     & 31.2(2.0)    & 26.4(2.0)   & -          & 0.8(0.1)  & 0.7(0.1)  & 8.1(1.3)     & -                      \\
                              & DIPCS10   & 26.2(2.6)    & 23.9(5.4)   & -          & 0.4(0.1)  & 0.5(0.1)  & 4.4(2.2)     & 4.4(2.5)               \\
                              & PA26      & 48.8(0.5)    & 49.0(0.6)   & -          & 1.7(0.2)  & 1.0(0.1)  & 2.5(0.4)     & 2.5(0.4)               \\
                              & SIE4x4    & 163.6(27.8)  & 108.5(14.2) & -          & 38.0(5.2) & 20.5(3.1) & 133.0(30.1)  & -                      \\
                              & ALKBDE10  & 39.2(8.7)    & 35.5(9.0)   & -          & 4.5(1.0)  & 3.0(0.8)  & 42.6(7.3)    & -                      \\
                              & YBDE18    & 18.6(2.7)    & 22.1(2.7)   & -          & 5.9(1.0)  & 2.5(0.4)  & 23.9(2.0)    & 20.3(1.7)              \\
                              & AL2X6     & 24.0(6.5)    & 23.2(3.7)   & -          & 3.5(1.2)  & 4.9(0.5)  & 20.1(4.8)    & -                      \\
                              & HEAVYSB11 & 23.6(2.8)    & 6.0(1.6)    & -          & 2.5(0.5)  & 2.5(0.3)  & 28.1(5.2)    & -                      \\
                              & NBPRC     & 22.5(6.4)    & 21.6(6.1)   & -          & 3.2(1.0)  & 3.4(1.0)  & 23.8(4.5)    & 23.8(4.5)              \\
                              & ALK8      & 47.7(19.0)   & 21.7(7.4)   & -          & 3.2(0.8)  & 3.7(1.2)  & 68.2(40.9)   & 68.2(40.9)             \\
                              & RC21      & 35.1(4.5)    & 37.7(4.4)   & -          & 10.2(1.3) & 5.2(0.7)  & 36.2(4.9)    & -                      \\
                              & G2RC      & 32.5(5.7)    & 24.3(4.3)   & 34.3(7.9)  & 9.2(1.4)  & 5.1(0.8)  & 16.2(2.9)    & 16.2(3.0)              \\
                              & BH76RC    & 56.2(8.7)    & 49.2(9.6)   & 218.5(-)   & 9.7(2.4)  & 6.1(0.8)  & 46.0(7.0)    & 45.1(14.8)             \\
                              & FH51      & 22.1(2.8)    & 20.9(3.3)   & 24.5(4.2)  & 8.1(1.1)  & 4.5(0.5)  & 9.4(2.1)     & 9.4(2.1)               \\
                              & TAUT15    & 108.1(21.9)  & 18.3(4.8)   & 46.9(9.9)  & 31.9(4.8) & 19.6(3.1) & 21.8(5.5)    & 21.8(5.5)              \\
                              & DC13      & 38.9(10.1)   & 33.9(7.6)   & 23.5(13.3) & 11.7(2.0) & 7.0(1.5)  & 32.8(13.7)   & 13.3(3.7)              \\ \midrule
\multirow{9}{*}{Prop. large}      & MB16-43 & 18.5(2.2)  & 31.8(3.1)  & -          & 3.3(0.4)  & 4.9(0.4)  & 16.5(2.6)  & 31.3(18.6) \\
                              & DARC      & 27.7(1.6)    & 31.1(2.2)   & 9.6(1.7)   & 7.6(1.0)  & 2.2(0.7)  & 1.3(0.3)     & 1.3(0.3)               \\
                              & RSE43     & 50.7(3.7)    & 56.9(3.9)   & -          & 26.1(2.0) & 10.7(0.8) & 44.2(7.3)    & -                      \\
                              & BSR36     & 8.2(0.7)     & 9.7(1.6)    & 31.0(3.3)  & 6.7(0.5)  & 15.3(1.6) & 24.1(2.4)    & 24.1(2.4)              \\
                              & CDIE20    & 28.6(4.8)    & 25.3(4.1)   & 50.9(9.5)  & 27.8(3.1) & 10.1(2.3) & 16.0(3.6)    & 16.0(3.6)              \\
                              & ISO34     & 24.6(3.7)    & 26.9(4.2)   & 50.5(39.8) & 7.3(1.4)  & 4.6(0.7)  & 7.3(3.3)     & 7.3(3.3)               \\
                              & ISOL24    & 28.3(4.4)    & 30.3(4.7)   & 18.9(4.1)  & 13.5(2.8) & 7.1(1.2)  & 7.6(1.3)     & 7.6(1.3)               \\
                              & C60ISO    & 4.6(1.0)     & 3.4(0.9)    & 26.0(2.9)  & 3.6(1.1)  & 7.8(1.2)  & 2.3(0.4)     & 2.3(0.4)               \\
                              & PArel     & 55.8(13.2)   & 72.0(18.5)  & -          & 22.1(6.2) & 8.2(2.0)  & 43.5(9.3)    & 43.5(9.3)              \\ \midrule
\multirow{7}{*}{React. barriers}  & BH76    & 64.2(6.6)  & 59.9(6.6)  & 68.9(55.2) & 21.0(1.7) & 6.9(0.6)  & 57.0(6.0)  & 35.7(6.7)  \\
                              & BHPERI    & 25.4(2.1)    & 27.9(2.3)   & 65.7(9.6)  & 12.5(0.8) & 7.8(0.9)  & 10.5(2.4)    & 10.5(2.4)              \\
                              & BHDIV10   & 10.5(2.3)    & 10.2(2.6)   & 13.8(7.2)  & 7.3(1.5)  & 1.3(0.3)  & 8.2(2.1)     & 8.2(2.1)               \\
                              & INV24     & 10.4(2.2)    & 5.9(1.0)    & 26.2(5.7)  & 3.5(0.8)  & 2.9(0.8)  & 20.1(4.9)    & 20.1(4.9)              \\
                              & BHROT27   & 21.5(3.2)    & 10.6(1.6)   & 12.9(3.2)  & 5.5(1.0)  & 4.3(0.7)  & 5.4(0.9)     & 5.4(0.9)               \\
                              & PX13      & 14.1(3.1)    & 4.7(1.1)    & 22.7(7.4)  & 12.1(0.8) & 5.4(0.7)  & 22.1(5.7)    & 22.1(5.7)              \\
                              & WCPT18    & 8.6(1.3)     & 6.2(1.1)    & 10.0(1.6)  & 8.9(1.2)  & 3.5(0.6)  & 12.1(1.6)    & 12.1(1.6)              \\ \midrule
\multirow{12}{*}{Inter. mol. NCI} & RG18    & 31.8(7.1)  & 11.0(3.1)  & -          & 11.8(3.0) & 11.1(1.8) & 53.6(13.1) & -          \\
                              & ADIM6     & 17.1(2.8)    & 19.5(4.2)   & 5.8(1.2)   & 8.9(2.1)  & 6.2(2.2)  & 4.5(1.2)     & 4.5(1.2)               \\
                              & S22       & 10.4(1.7)    & 5.9(0.9)    & 11.7(2.8)  & 2.2(0.4)  & 2.8(0.5)  & 4.1(0.6)     & 4.1(0.6)               \\
                              & S66       & 11.2(0.8)    & 7.6(0.6)    & 11.5(1.2)  & 3.4(0.4)  & 5.4(0.4)  & 5.1(0.5)     & 5.1(0.5)               \\
                              & HEAVY28   & 30.0(9.6)    & 27.8(5.0)   & -          & 36.8(4.0) & 12.1(2.2) & 54.3(8.8)    & -                      \\
                              & WATER27   & 5.2(0.7)     & 2.1(0.3)    & 33.4(8.1)  & 6.6(0.9)  & 10.0(1.6) & 12.0(1.5)    & 12.0(1.5)              \\
                              & CARBHB12  & 6.3(1.4)     & 16.9(6.7)   & 58.3(17.1) & 19.5(4.3) & 7.8(1.2)  & 19.7(3.9)    & 19.7(3.9)              \\
                              & PNICO23   & 31.0(6.7)    & 14.7(2.7)   & 251.8(7.3) & 21.8(2.6) & 5.0(0.7)  & 39.1(9.5)    & 39.1(9.5)              \\
                              & HAL59     & 16.6(2.6)    & 15.8(1.7)   & 74.4(41.7) & 20.1(2.9) & 4.2(0.4)  & 33.5(6.0)    & 33.5(6.0)              \\
                              & AHB21     & 11.8(2.6)    & 7.5(1.2)    & -          & 8.3(1.2)  & 8.6(1.2)  & 18.6(3.7)    & 18.6(3.7)              \\
                              & CHB6      & 8.4(3.9)     & 11.5(2.3)   & -          & 2.9(1.2)  & 2.8(0.9)  & 24.0(10.7)   & 24.0(10.7)             \\
                              & IL16      & 3.0(0.6)     & 2.2(0.3)    & -          & 1.2(0.3)  & 1.1(0.2)  & 2.5(0.4)     & 2.5(0.4)               \\ \midrule
\multirow{9}{*}{Intra. mol. NCI}  & IDISP   & 26.1(14.4) & 27.1(17.0) & 82.5(45.7) & 15.6(5.3) & 11.1(3.6) & 19.7(8.9)  & 19.7(8.9)  \\
                              & ICONF     & 45.7(13.5)   & 28.3(4.8)   & 73.3(27.2) & 6.6(1.5)  & 5.9(1.4)  & 29.8(10.3)   & 29.8(10.3)             \\
                              & ACONF     & 20.5(3.4)    & 6.0(1.3)    & 4.8(1.1)   & 6.6(0.8)  & 2.7(0.3)  & 1.6(0.4)     & 1.6(0.4)               \\
                              & Amino20x4 & 26.0(2.1)    & 22.2(2.2)   & 23.2(2.2)  & 7.6(0.7)  & 6.1(0.5)  & 7.0(0.6)     & 7.0(0.6)               \\
                              & PCONF21   & 76.0(14.1)   & 61.6(10.1)  & 77.2(18.5) & 29.0(5.4) & 11.7(2.0) & 17.8(2.5)    & 17.8(2.5)              \\
                              & MCONF     & 16.5(1.2)    & 19.7(1.5)   & 9.4(0.9)   & 3.8(0.4)  & 5.5(0.4)  & 5.4(0.5)     & 5.4(0.5)               \\
                              & SCONF     & 30.9(12.0)   & 20.3(6.2)   & 30.5(7.0)  & 9.5(2.1)  & 3.7(1.2)  & 6.6(1.3)     & 6.6(1.3)               \\
                              & UPU23     & 12.3(1.6)    & 28.9(2.9)   & -          & 5.0(0.8)  & 9.4(1.0)  & 10.7(1.5)    & 10.7(1.5)              \\
                              & BUT14DIOL & 19.4(1.6)    & 25.4(1.8)   & 28.8(1.7)  & 8.4(0.5)  & 8.3(0.4)  & 11.9(0.5)    & 11.9(0.5)              \\ \bottomrule
\end{tabular}
\end{table}

\subsection{The QM9 dataset}

\label{sec:si_qm9_uncer}

We provide the QM9 MAEs on all 12 target properties as reported in Table~\ref{table:qm9}. The standardized MAE and standardized log MAE in Table~\ref{table:qm9} are computed following the Appendix C of~\cite{DimeNet}. Uncertainties for test MAEs are obtained by statistical bootstrapping with sample size 5000 and 100 iterations.


\subsection{MD17 and rMD17 datasets} 
\label{sec:md17}

The MD17 dataset \cite{md17} contains energy and force labels from molecular dynamics trajectories of small organic molecules, and is used to benchmark ML methods for modelling a single instance of a molecular potential energy surface. Recently the revised-MD17 (rMD17) dataset\cite{rmd17} was reported with improved label fidelity. For both the MD17 and the rMD17 dataset, We train OrbNet-Equi simultaneously on energies and forces of 1000 geometries of each molecule and test on another 1000 geometries of the same molecule, using previously reported dataset splits (Section \ref{si_training_md17}). All results are obtained without performing system-specific hyperparameter selections and without using additional regularization techniques such as model ensembling or stochastic weight averaging~\cite{izmailov2018averaging}.

As shown in Table \ref{table:rmd17} and Table \ref{table:md17}, OrbNet-Equi with direct learning achieves competitive accuracy when compared against the best results reported by kernel methods \cite{sGDML,rmd17} and graph neural networks (GNNs) \cite{DimeNet,nequip,schutt2021equivariant}. Additional performance gains are observed when the models are trained with the delta learning strategy, resulting in the lowest test MAEs for both energy and forces on most of the test molecules. 
We note that MD17 represents an ideal case scenario where abundant high-level reference calculations are available on the potential energy surface of a single chemical system and the configurations of interest only spans a thermally-accessible energy scale.
Despite the highly-interpolative nature of this learning task, OrbNet-Equi results still matches the accuracy that can be achieved with state-of-the-art neural network potentials for which system-dependent optimizations are often employed.

\subsection{Inference timings}

\label{si_infer}

Wall-clock timing results for evalulating the pretrained OrbNet-Equi/SDC21 model are reported on the Hutchison dataset~\cite{Hutch} which represents the distribution of realistic drug-like molecules,. All timing results are obtained using 16 cores of an Intel Xeon Gold 6130 CPU. We note that due to the use of a Lagrangian formalism we previous developed \cite{qiao2020multi} to efficiently compute the analytic nuclear gradients in which the operator derivatives with respect to atomic coordinates are not explicitly evaluated, for such medium-sized organic molecules the overhead for computing energy and forces in addition to only computing the energies is still dominated by neural network back propagation.

\subsection{Downstream benchmarks} 

\label{si_downstream}

Table~\ref{table:downstream}-\ref{table:gmtkn55} provide summary statistics of method performances on downstream main-group quantum chemistry benchmarks considered in this study.


Additional computational details for geometry optimization experiments are provided as follows. The symmetry-corrected root-mean-square-deviations (RMSDs) are computed between the test geometries and the reference DFT ($\omega$B97X-D3/def2-TZVP) geometries following a Hungarian algorithm~\cite{allen2014implementation} to account for equivalent atoms. OrbNet-Equi/SDC21 and GFN-xTB results are obtained using Entos Qcore version 1.1.0 with the L-BFGS algorithm with tight thresholds (energy change after iteration < 1E-6 a.u., gradient RMS < 4.5E-4 a.u, max element of gradient < 3E-4 a.u., optimization step RMS < 1.8E-3 a.u., and max element of optimization step < 1.2E-3 a.u.). GFN2-xTB results are obtained with the XTB~\cite{xtbcode} package with default settings. ANI-2x energy calculations are performed with the TorchANI~\cite{gao2020torchani} package and geometry optimizations are performed with the geomeTRIC optimizer~\cite{tric} with default settings. Using this software setting, ANI-2x optimizations are found unable to converge on 21 out of 52 conformers on MCONF and 1 out of 12 conformers on ROT34, and the average errors are reported as ``-''. 
On the subsets that ANI-2x geometry optimizations converged, the average RMSD for ANI-2x is 0.154$\pm$0.038 $\AA$ on ROT34 and  0.324$\pm$0.026 $\AA$ on MCONF versus the reference geometries.
The histograms and kernel density estimations displayed in Figure~\ref{fig:6}d are computed on the subsets where all methods successfully converged, that is, 11 conformers from ROT34 and 31 conformers from MCONF.

\clearpage

\section{Hyperparameters and training details}
\label{si_training}


\subsection{Model hyperparameters}

We use the same set of model hyperparameters to obtain all numerical experiment results on open benchmarks reported in this work. The hyperparameters are summarized in Table \ref{table:hyper} and Table \ref{table:feat}.  The hyperparameters for the OrbNet-Equi/SDC21 model is locally optimized based on a 4219-sample validation set from the training data distribution, with the additional difference of using LayerNorm~\cite{ba2016layer} for mean $\mu$ and variance $\sigma$ estimates, and $\epsilon=0.5$ in all EvNorm layers. This choice is made to improve model robustness when applied to extrapolative molecular geometries. 

\begin{table}[h!]
\setlength\extrarowheight{3pt}
\caption{The model hyperparameters for OrbNet-Equi used for benchmarking studies.}
\label{table:hyper}
\centering
\begin{tabular}{cccc}
\toprule
 Symbol & Meaning & Defined in & Value(s)  \\
 \midrule
 ${N}$ & Total number of feature channels in $\mathbf{h}^{t}$ & Methods.~\ref{sec:unite}, \eqref{eq:nto1} & 256 \\
 ${N}_{lp}$ & Number of feature channels for each $(l, p)$ in $\mathbf{h}^{t}$ & Methods.~\ref{sec:unite}, \eqref{eq:nto1} & See Table \ref{table:feat} \\
 $t_1$ & Number of convolution-message-\allowbreak passing update steps & Methods.~\ref{sec:unite} & 4 \\
 $t_2$ & Number of post-update \allowbreak point-wise interaction modules & Methods.~\ref{sec:unite} & 4 \\
 $I$ & Number of convolution channels $i$ & Methods.~\ref{sec:blockwise}, \eqref{eq:conv} & 8 \\
 $J$ & Number of attention heads $j$ & Methods.~\ref{sec:message_passing}, \eqref{eq:mp} & 8 \\
 $d^{\mathrm{MLP}}$ & Depth of MLPs & MLPs in \eqref{eq:pi01}-\eqref{eq:pi02} & 2 \\
  & Activation function & MLPs in \eqref{eq:pi01}-\eqref{eq:pi02} & Swish~\cite{ramachandran2017searching} \\
 ${N}^\upxi$ & Number of Radial basis functions $\xi$ & Methods.~\ref{sec:message_passing}, \eqref{eq:rbf} & 16 \\
  & Estimation scheme for ($\mu$, $\sigma$) for EvNorm, $t < t_1$ & Methods.~\ref{sec:evnorm}, \eqref{eq:EvNorm} & BatchNorm~\cite{ioffe2015batch} \\
  & Estimation scheme for ($\mu$, $\sigma$) for EvNorm, $t \geq t_1$ & Methods.~\ref{sec:evnorm}, \eqref{eq:EvNorm} & LayerNorm~\cite{ba2016layer} \\
  & Initialization of $\beta_{nlp}$ in EvNorm layers & Methods.~\ref{sec:evnorm}, \eqref{eq:EvNorm} & $\mathrm{Uniform}\big([0.5, 1.5)\big)$ \\
 $\epsilon$  & Stability factor $\epsilon$ in EvNorm layers & Methods.~\ref{sec:evnorm}, \eqref{eq:EvNorm} & 0.1 \\
 & Total number of parameters & - & 2.1M \\
\bottomrule
\end{tabular}
\end{table}

\begin{table}[h!]
\caption{The number of feature channels $N^{\mathrm{h}}_{lp}$ for each representation group $(l,p)$ of $\mathbf{h}^{t}$ used in this work across all values of $t$. Note that $l_\mathrm{max}=4$.}
\label{table:feat}
\centering
\begin{tabular}{c| ccccc}
\toprule
  ${N}_{lp}$ & $l=0$ & $l=1$ & $l=2$ & $l=3$ & $l=4$ \\
 \midrule
  $p=+1$ & 128 & 48 & 24 & 12 & 6 \\
  $p=-1$ & 24 & 8 & 4 & 2 & 0 \\
\bottomrule
\end{tabular}
\end{table}

\subsection{Training}
For all training setups we use the Adam optimizer~\cite{kingma2014adam} with maximal learning rate $5 \times 10^{-4}$ and parameters $\beta_1=0.9, \beta_2=0.999$ and $\epsilon=10^{-4}$. The loss function denoted as $\mathcal{L}$ below refers to a SmoothL1Loss function~\cite{girshick2015fast}. Batch sizes and the total number of epochs are adjusted for different benchmarks to account for their vastly different training set sizes, as detailed below. No additional regularization such as weight decay or early stopping is employed.

\subsubsection{QM9} For QM9 tasks we optimize the model using the loss $\mathcal{L}(y, y_{\theta})$ for each target $y$. We use a batch size of 64 when the training sample size > 1000, and a batch size of 16 for smaller training sizes. We employ a learning rate schedule of first performing linear warmup for 100 epochs to the maximal learning rate followed by a cosine learning rate annealing~\cite{loshchilov2016sgdr} for 200 epochs. Models are trained on a single Nvidia Tesla V100-SXM2-32GB GPU, taking around 36 hours for training runs with 110k training samples. 

\subsubsection{MD17} \label{si_training_md17}
For MD17, we optimize the model by simultaneously training on energies $E(\mathbf{R})$ and forces $\mathbf{F}(\mathbf{R})$, using the following loss function:
\begin{equation}
    \mathcal{L}_{\mathrm{E+F}}(E, \mathbf{F}; E_{\theta}, \mathbf{F}_{\theta}) \defeq c_{\mathrm{E}} \cdot \mathcal{L}(E(\mathbf{R}); E_{\theta}(\mathbf{R})) + c_{\mathrm{F}} \cdot \frac{1}{3\lvert A\lvert } \sum_{A}^{\lvert A\lvert } \sum_{m\in\{x,y,z\}}  \mathcal{L}(-\frac{\partial E_{\theta}(\mathbf{R}))}{\partial R_{A,m}} - F_{A,m}(\mathbf{R}))
\end{equation}
Following previous works~\cite{schutt2017schnet,DimeNet,nequip}, we set $c_{\mathrm{E}}=1$, $c_{\mathrm{F}}=1000 \AA$  for training on the rMD17 labels, and $c_{\mathrm{E}}=0$, $c_{\mathrm{F}}=1000 \AA$ for training on the original MD17 labels. For each molecular system, we use the 1000 geometries of the `train 01' subset given by~\cite{rmd17} for training and the 1000 geometries of the `test 01' subset for testing. We use a batch size of 8, and train the model on a single Nvidia Tesla V100-SXM2-32GB GPU for 1500 epochs using a step decay learning rate schedule, taking around 30 hours for each training run.

For proof-of-principle purposes, the gradients of tight-binding features with respect to atomic coordinates are obtained using finite difference with a 5-point stencil for each degree of freedom. The grid spacing between the stencil points are set to 0.01 Bohr. We note that in principle, this cost of evaluating and storing feature gradients can be avoided if the electronic structure method is implemented with back-propagation and the model can be trained end-to-end on both energy and force labels.

\subsubsection{Electron densities} \label{si_training_density}
For electron density, we train the models on the analytic $L^{2}$ density loss following~\cite{Grisafi2018}:
\begin{equation}
    \label{eq:dens_loss}
    \mathcal{L}_{\rho}(\rho, \hat{\rho}) \defeq \int \lVert \rho(\vec{r})-\hat{\rho}(\vec{r})\rVert ^{2} d{\vec{r}} = (\mathbf{d} - \hat{\mathbf{d}})^{T} \, \mathbf{S}^{\mathrm{\rho}} \, (\mathbf{d} - \hat{\mathbf{d}})
\end{equation}
where the density coefficients $\mathbf{d}\defeq \bigoplus_{A, n, l, m} d_{A}^{n l m}$ are defined in \eqref{eq:des_expand}, and $\mathbf{S}^{\mathrm{\rho}}$ is the overlap matrix of the density fitting basis $\{\chi\}$. A sparse COO format is used for $\mathbf{S}^{\mathrm{\rho}}$ to efficiently compute $\mathcal{L}_{\rho}$ during batched training.
We use a batch size of 64 and a cosine annealing learning schedule for training; the models are trained on a single Nvidia Tesla V100-SXM2-32GB GPU for 2000 epochs on the BFDb-SSI dataset taking 10 hours, and for 500 epochs on the QM9 dataset taking 120 hours. 

\subsubsection{The OrbNet-Equi/SDC21 model} 
The training dataset (see Methods~\ref{si_dataset}) contains different geometries $b_{\upeta}$ for each molecule $\eta$ in the dataset. We train on a loss function following~\cite{orbnet1}:
\begin{align}
    \mathcal{L}_{\mathrm{G}}(E(\eta, b_{\upeta}), E_{\theta}(\eta, b_{\upeta})) 
    &\defeq \mathcal{L}(E(\eta, b_{\upeta}), E_{\theta}(\eta, b_{\upeta}) \\
    &+ c_{\mathrm{G}} \cdot \mathcal{L}( E(\eta, b_{\upeta}) - E(\eta, \hat{b}_{\upeta}) , E_{\theta}(\eta, b_{\upeta}) - E_{\theta}(\eta, \hat{b}_{\upeta}) ) \nonumber
\end{align}
where $\hat{b}_{\upeta}$ is a geometry randomly sampled from all the geometries $\{ b_{\upeta} \}$ of each molecule $\eta$ within each mini-batch during training. We use $c_{\mathrm{G}}=10$ in this work. 
We train the model for 125 epochs on a Nvidia Tesla V100-SXM2-32GB GPU using a batch size of 64 and a cosine annealing learning rate schedule taking 64 hours.

\bibliographystyle{unsrtnat}
\bibliography{main}

\end{document}